\pdfoutput=1
\documentclass{article} 
\usepackage{fullpage}
\usepackage[colorlinks, citecolor=blue]{hyperref}
\usepackage{natbib}
\usepackage{bm}

\usepackage{amsmath,amsfonts,amssymb, amsthm,thm-restate, bm}
\usepackage{mathrsfs, mathtools}
\usepackage{xcolor}
\usepackage{booktabs}
\usepackage{multirow}
\usepackage[ruled,linesnumbered]{algorithm2e}

\SetCommentSty{mycommfont}
\usepackage{xspace}
\usepackage{natbib}
\usepackage{bm}
\usepackage{cleveref}
\usepackage{accents}
\newtheorem{theorem}{Theorem}
\newtheorem{assumption}{Assumption}
\newtheorem{definition}{Definition}
\newtheorem{proposition}{Proposition}
\newtheorem{corollary}{Corollary}
\newtheorem{lemma}{Lemma}
\theoremstyle{definition}

\newtheorem{remark}{Remark}

\newcommand{\cA}{{\mathcal A}}
\newcommand{\cX}{\ensuremath{\mathcal{X}}}

\newcommand{\cF}{\ensuremath{\mathcal{F}}}

\newcommand{\R}{\ensuremath{\mathbb{R}}}

\renewcommand{\P}{\ensuremath{\mathbb{P}}}

\newcommand{\cN}{\ensuremath{\mathcal{N}}}

\DeclareMathOperator*{\rE}{{\mathbb E}}
\newcommand{\rV}{{\mathbb V}}

\DeclareMathOperator*{\argmax}{\text{argmax}}

\usepackage{eqparbox}

\newcommand{\order}{O}
\newcommand{\otil}{\ensuremath{\widetilde{\order}}}

\newcommand{\cW}{\ensuremath{\mathcal{W}}}

\renewcommand{\dim}{\ensuremath{\mathrm{dim}}}

\newcommand{\E}{\mathbb{E}}
\newcommand{\calD}{\mathcal{D}}

\newcommand{\norm}[1]{\left\|{#1}\right\|} %
\newcommand{\calE}{\mathcal{E}}
\newcommand{\defeq}{:=}
\newcommand{\calF}{\mathcal{F}}

\newcommand{\D}{\mathsf{D}}
\newcommand{\bsigma}{\bar{\sigma}}

\newcommand{\frob}{\mathsf{F}}
\newcommand{\calI}{\mathcal{I}}
\newcommand{\calH}{\mathcal{H}}
\newcommand{\calG}{\mathcal{G}}

\newcommand{\1}{\mathbf{1}}
\newcommand{\calT}{\mathcal{T}}
\newcommand{\calN}{\mathcal{N}}
\newcommand{\eps}{\epsilon}
\newcommand{\tf}{\tilde{f}}
\newcommand{\calA}{\mathcal{A}}

\newcommand{\bbeta}{\bar{\beta}}
\newcommand{\calS}{\mathcal{X}}
\newcommand{\calZ}{\mathcal{Z}}
\newcommand{\Z}{\mathbb{Z}}

\newcommand{\calW}{\mathcal{W}}

\newcommand{\mmu}{\bm{\mu}}

\newcommand{\oracle}{\mathcal{B}}
\newcommand{\calTo}{\calT_{\mathrm{o}}}
\newcommand{\calToo}{\calT_{\mathrm{oo}}}
\newcommand{\xid}{\xi_{\mathrm{dif}}}
\newcommand{\lin}{\mathsf{lin}}

\newcommand{\calM}{\mathcal{M}}
\newcommand{\calP}{\mathcal{P}}
\newcommand{\calR}{\mathcal{R}}

\newcommand{\Par}[1]{\left(#1\right)}

\newcommand{\epsc}{\eps_\mathrm{b}}
\newcommand{\epscov}{\eps_\mathrm{c}}
\newcommand{\epsb}{\eps_\mathrm{b}}

\newcommand{\alg}{VO$Q$L\xspace}
\newcommand{\alglong}{Variance-weighted Optimistic $Q$-Learning\xspace}

\title{VO$Q$L: Towards Optimal Regret in Model-free RL with\\  Nonlinear Function Approximation} %

\author{%
    Alekh Agarwal\textsuperscript{$\ast$1}
	~~
	Yujia Jin\textsuperscript{$\ast$2} ~~
	Tong Zhang\thanks{alphabetical order}~\textsuperscript{1,3} \\
	\\
\textsuperscript{1}	Google Research\quad \textsuperscript{2} Stanford University\quad\textsuperscript{3} HKUST\\
	\texttt{\{alekhagarwal,tozhang\}@google.com}~~~\texttt{yujiajin@stanford.edu}
}

\date{}

\begin{document}

\maketitle

\begin{abstract}
We study time-inhomogeneous episodic reinforcement learning (RL) under general function approximation and sparse rewards. We design a new algorithm, \alglong (\alg), based on  $Q$-learning and bound its regret assuming completeness and bounded Eluder dimension for the regression function class. As a special case, \alg achieves $\widetilde{O}(d\sqrt{HT}+d^6H^{5})$ regret over $T$ episodes for a horizon $H$ MDP under ($d$-dimensional) linear function approximation, which is asymptotically optimal. Our algorithm incorporates weighted regression-based upper and lower bounds on the optimal value function to obtain this improved regret. The algorithm is computationally efficient given a regression oracle over the function class, making this the first computationally tractable and statistically optimal approach for linear MDPs.
\end{abstract}

\section{Introduction}\label{sec:intro}

Optimally trading off exploration and exploitation to achieve a low regret is a fundamental question in Reinforcement Learning (RL) research. The last few years have seen some significant advances on this front, with the development of algorithms that achieve optimal regret guarantees when the underlying state space is finite~\citep{azar2017minimax,zanette2019tighter,zhang2019regret,simchowitz2019non,zhang2020almost,he2020minimax}. Remarkably, a line of work has shown that when the overall reward of each trajectory is bounded independent of the horizon, then the regret has no explicit polynomial horizon dependence~\citep{zanette2019tighter,zhang2021reinforcement,zhang2020nearly,ren2021nearly,tarbouriech2021stochastic}. These developments have been enabled by a combination of careful algorithmic design as well as analysis. However, the resulting insights typically fall some way short of being applicable to real-world RL settings which typically involve large state spaces, and where we rely on the use of function approximation to generalize across related states in the learned policies, valued functions and/or models. Some recent works~\citep{zhang2021improved,kim2021improved,zhou2022computationally} do generalize ideas from the tabular setting for a special class of RL problems with linear function approximation, called linear mixture MDPs, using a model-based approach that is more amenable to ideas from the tabular setting. Motivated by this landscape our paper asks if we can develop model-free techniques that attain optimal regret guarantees with general function approximation?

For studying settings with function approximation, several frameworks~\citep[see e.g.][]{jiang2017contextual, jin2021bellman, du2021bilinear, agarwal2022non, foster2021statistical} have been developed to capture the structural properties of the underlying MDP and function classes. A primary distinction arises between the so-called $V$-type and $Q$-type settings. While settings with linear function approximation are admissible in both the frameworks, $V$-type settings also allow representation learning more generally, but typically require more complex algorithms in order to incorporate optimistic exploration. The $Q$-type settings, on the other hand, enable optimistic reasoning using pointwise bonuses, and admit problems with linear function approximation, as well as bounded Eluder dimension~\citep{russo2013eluder,wang2020reinforcement,jin2021bellman}. Such bonus schemes have also been adapted with some success in the empirical literature~\citep{feng2021provably,burda2018exploration,henaff2022exploration}. Motivated by the simpler algorithms, we focus on $Q$-type settings where the underlying problem satisfies admits a bound on a generalized Eluder dimension. 

\begin{table}[!t]
    \centering
    \renewcommand{\arraystretch}{1.75}
    \begin{tabular}{{c}{c}{c}{c}}
    \toprule
Setting & Method &  Regret\\
\midrule
\multirow{3}{*}{linear mixture MDPs} & UCRL-VTR~\citep{ayoub2020model} & $dH\sqrt{T}~^\dagger$ \\
 & UCRL-VTR+~\citep{zhou2021nearly}  &  $\Par{d\sqrt{H}+\sqrt{d}H}\sqrt{T}~^\dagger$\\
 & HF-UCRL-VTR+~\citep{zhou2022computationally} &  $d\sqrt{T}~^\star$\\
\midrule
\multirow{4}{*}{linear MDPs} 
& LSVI-UCB~\citep{jin2020provably} & $d^{\frac{3}{2}}H\sqrt{T}$ \\
& ELEANOR~\citep{zanette2020learning} & $dH\sqrt{T}$ \\
& VO$Q$L (\textbf{our work,~\Cref{thm:regret-general-linear}}) & $d\sqrt{HT}$ \\
&\textbf{Lower bound}~\citep{zhou2021nearly} & $d\sqrt{HT}$ \\
\midrule
\multirow{3}{*}{general function class}
& $\calF$-LSVI~\citep{wang2020reinforcement} & $\dim(\calF)\sqrt{\log\calN\log\cN(\cX\times\cA)}\cdot H\sqrt{T}$ \\
& GOLF~\citep{jin2021bellman} & $\sqrt{\dim(\calF)\log\calN}\cdot H\sqrt{T}$ \\

& VO$Q$L (\textbf{our work,~\Cref{thm:regret-general-online}}) & $\sqrt{\dim(\calF)\log\calN\cdot HT}$\\
\bottomrule
\end{tabular}
    \caption{\textbf{Comparison of regret:} In-homogeneous episodic RL with horizon length $H$, trajectory number $T$, sparse rewards $\sum_{h\in[H]}r^h\le 1$. We use $^\dagger$ to indicate methods which need an additional uniform reward assumption $r^h\in[0,1/H]$, $h\in[H]$ and use $^\star$ to indicate work that considers a homogeneous setting instead. We only state the leading $O(\sqrt{T})$ term and hide poly-logarithmic factors in $T$, $H$, $\epsilon$ and $\delta$. In general function approximation, $\dim(\calF)$ and $\log\calN$ refer to properties of the function class $\calF$, i.e. the generalized Eluder dimension (see~\Cref{def:general-eluder-RL}) and $\calN = |\calF|$.\protect\footnotemark~ In \citet{wang2020reinforcement}, $\calN(\cX\times\cA)$ refers to the covering number of the state-action space. We omit a comparison with $V$-type settings pioneered in~\citet{jiang2017contextual} and subsequently refined in many works, as the dimension and horizon dependencies are generally sub-optimal using these approaches, when translated to $Q$-type settings such as linear MDPs.}
    \label{tbl:results}
\end{table}
\footnotetext{Formally, we allow $\calF^h$ to vary as a function of $h$, in which case the precise result can be found in Theorem~\ref{thm:regret-genera.}. Here we discuss the setting of $\calF$ being shared across $h$ for simplicity.}

\paragraph{Model and Our Results.} We study time-inhomogenous finite horizon MDPs with a horizon $h$, meaning that the transition dynamics and rewards at each step $h=1,2,\ldots,H$ can be different. We focus on model-free and value-based approaches, where the goal is to learn the optimal value function $Q_\star$ by searching over some function class $\cF$. We assume that $\cF$ satisfies standard realizability and completeness assumptions up to an accuracy $\epsilon$ (\Cref{ass:eps-realizability-RL}), and analyze the exploration complexity of $\cF$ in terms of a generalization of the Eluder dimension (\Cref{def:general-eluder-RL}). These assumptions capture settings with a (generalized) linear function approximation, and have also been used empirically with neural network-based function approximation in prior work~\citep{feng2021provably}. Following prior works on sparse-reward setting, we assume that the overall reward of each trajectory is bounded in $[0,1]$ independent of the horizon.\footnote{Note that we can reduce the uniform rewards setting when $r^h\in[0,1]$ to this sparse reward setting by considering $r^h\gets r^h/H$, which leads to an additional $H$-factor in the final regret bound.} Note, however, that we do not expect a completely horizon-free result despite this assumption, due to the time inhomogeneous nature of the setting (as demonstrated by the lower bound for linear MDPs in Table~\ref{tbl:results}).

Under these assumptions, our main result shows that the regret of our algorithm \alglong(\alg)\footnote{pronounced vocal} scales as
\begin{equation}
    \otil\left(\sqrt{TH\dim(\cF)\log \cN}\right),
    \label{eq:regret-final-simple}
\end{equation}
for $T$ large enough, where $\cN$ is the cardinality of $\cF$ and $\dim(\cF)$ is the generalized Eluder dimension of $\cF$. The key difference in our definition of the generalized Eluder dimension from the standard one is a generalization to weighted regression settings, which is essential to our approach as we explain shortly. The resulting sample complexity has a linear dependence on the horizon $H$, arising from the time-inhomogenous setting. The bounds depend on the existence of a certain bonus oracle $\oracle$ (see~\Cref{def:bonus-conditions}), that captures the predictive range of the version space given some data, on a new query point. For intuition, we instantiate our bounds in two special cases with concrete specifications of the bonus oracle.

For linear MDPs, the bonus oracle uses the standard elliptical bonus using the covariance of the observed data in the MDP features. With this choice, the regret asymptotically scales as $\otil(d\sqrt{HT})$, where $d$ is the dimension of the underlying features. This matches the lower bound in~\citet{zhou2021nearly}, by rescaling their lower bound by a factor of $H$ to account for the normalization of rewards.\footnote{The quantity $T$ in~\citet{zhou2021nearly} refers to the total number of samples, which is $TH$ in our setup. The bounds match with this translation.} For linear MDPs, this also matches a recent result of~\citet{hu2022nearly}, but their analysis unfortunately suffers from a technical issue which we explain in \Cref{app:comp}.

For more general function approximation, the main comparison points are the prior results of~\citet{wang2020reinforcement}, ~\citet{jin2021bellman}, and \citet{DMZZ2021-neurips}. We instantiate the bonus oracle using the sensitivity sampling approach of~\citet{kong2021online}, and obtain the bound~\eqref{eq:regret-final-simple} with an additional lower order term that depends on a covering number of the state-action space. Notably, the dependence on the covering number of the state-action space occurs in a lower order term only, unlike in~\citet{wang2020reinforcement} who incur it in the leading order term of their regret bound. We also improve in the dimension and/or horizon dependence relative to~\citet{wang2020reinforcement}, ~\citet{jin2021bellman}, and \citet{DMZZ2021-neurips}. The improvement in pushing the state-action covering number to a lower order term arises due to a careful analysis that does not incur the complexity of the possible bonuses in the leading term of the regret. This is also critical to attaining optimal regret in linear MDPs, where the class of bonuses has a complexity of $\Omega(d^2)$~\citep{jin2020provably}.

We summarize our result in the context of other relevant prior works in Table~\ref{tbl:results}.

\subsection{Overview of Techniques}

\paragraph{Optimal $d$ dependence through refined bonus construction.} Our algorithm \alg is based on optimistic $Q$-learning in a finite horizon setting, where we maintain an optimistic approximation $f_t$ for $Q_\star$, and update $f_t^h(x,a)$ to approximate the Bellman backup $r(x,a) + \E[\max_{a'\in\cA} f_t^{h+1}(x',a') | x,a]$. The main design choices are how to perform this approximate Bellman backup, and how to design an optimistic $f_t$. In prior works for both linear~\cite{jin2020provably} and non-linear~\cite{wang2020reinforcement} settings, the backups are implemented using a squared regression objective, while optimism is incorporated by adding an explicit bonus term to the regression estimate, which captures the range of predictions that good regression solutions can make at a new state-action pair. Since the bonus is data dependent, we need a uniform convergence argument over this random bonus in analyzing the regression, which results in a sub-optimal $d$ scaling. Some prior works~\citep{zanette2020learning,jin2021bellman} circumvent this issue by not using bonus, and use a more complicated constrained optimization problem to enforce optimism only in the initial state of the MDP, as opposed to a pointwise bonus. While theoretically attractive, this results in a impractical algorithms, relative to the bonus based approaches. In our work, we tackle this difficulty in the analysis, where the leading term in regret only depends on the statistical complexity of $\cF$. A key insight is to capture the leading order term in the regression analysis not in terms of the joint complexity of the random functions $f_t^h$ and $f_t^{h+1}$, but instead using $f_t^h$ and a conditional expectation of $r + f_t^{h+1}$, which always lies in $\cF$ by completeness. Additionally, we analyze the variance of our changing regression target $f_t^{h+1}$, in terms of the variance of $Q_\star^{h+1}$ and that of $f_t^{h+1} - Q_\star^{h+1}$. Through a careful argument, we only incur the first variance in the leading order term in our bonus, and hence regret, which is smaller due to the absence of any optimistic bias.

\paragraph{Variance-weighted regression with monotone variance structure.} Crucial to the above simplification and also our improved horizon dependence relative to prior works is the use of variance-weighted regression. \citet{gheshlaghi2013minimax, azar2017minimax} pioneered the use of the Bellman property of variance in improving the horizon dependence in tabular RL algorithms, and recent works~\citep{zhang2021improved,kim2021improved,zhou2022computationally} have extended this to linear model-based RL settings, known as linear mixture MDPs. In particular, \citet{zhang2021improved} show that leveraging this Bellman property requires changing the unweighted regression as in~\citet{ayoub2020model} to a weighted regression, where the sample weights correspond to an appropriate notion of variance. In the model-based setting studied previously, this appropriate variance is the variance of the current policy's value function under the true model, but we would like to instead use variance of the optimistic regression target $r + f_t^{h+1}$ in our model-free setting. However, this definition of variance creates additional challenges, where we need the variance estimate associated with a prior sample from some round $s < t$ to remain valid, even when the target function $f_t^h \ne f_s^h$. Indeed, an error in ensuring the monotonicity of variance breaks an important attempt from~\citet{hu2022nearly} to obtain near-optimal bounds for linear MDPs. We carry out a careful construction of additional \emph{overly optimistic} and \emph{overly pessimistic} estimates on $Q_\star$, and design monotone variance estimates using these quantities. The use of unweighted regression in the overly optimistic and pessimistic function estimation, together with a carefully designed greedy and safely exploration policy, avoids the technical mistake in~\citet{hu2022nearly} (see Appendix~\ref{app:comp} for more details).%

A careful combination of these ingredients, along with refined concentration bounds from recent works to handle the regression objectives and standard arguments for bounding regret in terms of Eluder dimension like quantities gives our results. It is natural to ask if the results can be further improved to be horizon independent (in the leading order term), if we restrict to time-homogenous MDPs. This is not immediate from our current analysis, and an interesting direction for future work. 
\section{Related Work}\label{ssec:related}

\paragraph{Improving regret bounds in RL problems.} \citet{gheshlaghi2013minimax,azar2017minimax} first studied the Bellman property of variance in tabular RL problems and used it for improving the horizon dependence in the sample complexity and regret bounds. In the tabular setting, \citet{zanette2019tighter} later improved the result to further achieve problem-dependent regret bounds which match the classical regret~\citep{azar2017minimax} in the worst-case, but obtain horizon-free guarantees in the sparse reward case when the cumulative reward is at most 1 in any trajectory. ~\cite{foster2021efficient} also studied problem-dependent regret bounds for tabular MDPs. Additionally, there are works on obtaining fine-grained regret bounds in other RL settings, e.g. the data-dependent regret bounds for adversarial bandits and MDPs~\citep{lee2020bias} and problem-dependent regret bounds for RL under linear function approximation~\citep{wagenmaker2022first}.

Horizon-free bounds for the sparse reward setting (when total reward of each trajectory is independent of horizon length $H$) has received increased attention recently, starting from the COLT open problem posed in~\citet{jiang2018open}. In the tabular setting,  a line of work~\citep{zanette2019tighter,zhang2021reinforcement,zhang2020nearly,ren2021nearly,tarbouriech2021stochastic} designs algorithms that incur a poly-logarithmic in $H$ regret, using tighter concentration bounds in their analysis. Another line of work further studies how to obtain completely horizon-free methods, at the cost of paying extra exponential~\citep{li2022settling} or polynomial factors~\cite{zhang2022horizon} in other problem parameters like state or action size. These ideas have been further generalized to linear mixture MDPs~\citep{zhang2021improved,kim2021improved,zhou2022computationally}. However, the model-based approach they rely on is challenging to extend to model-free settings with function approximation. In particular, the challenge of designing a montonic variance upper bound in the face of a changing regression target does not arise in the model-based setting.

\paragraph{Linear function approximation.} Linear MDPs have become a popular simple model for understanding function approximation beyond the tabular setting. Many works, such as~\citet{jin2020provably, yang2020reinforcement,zanette2020learning} obtain $O(\sqrt{T})$ regret bounds, and \citet{zanette2020learning} obtain a $O(dH\sqrt{T})$ bound, which is optimal in the scaling with $d$, but sub-optimal in the scaling with $H$ by a $\sqrt{H}$ factor. \Cref{tbl:results} provides more detailed comparisons. In terms of techniques, our approach is closest to that of \citet{jin2020provably} in using a bonus based approach in a model-free setting, but incorporates weighted regression and other analysis improvements to obtain optimal guarantees. More closely related is the very recent result of~\citet{hu2022nearly}, who provide near-optimal regret bound for linear MDPs. Unfortunately, their analysis suffers from a technical issue that we explain in Appendix~\ref{app:comp}, but their algorithm nevertheless contains many important design elements that we also incorporate. A key tool is the use of an \emph{over-optimistic} and \emph{over-pessimistic} value function estimates, in addition to a standard optimistic estimate, to help with bounding the variance of our regression targets. Like them, we also use the greedy policy of the over-optimistic function on certain rounds, as opposed to that of the optimistic function. Overall, these over-optimistic and over-pessimistic values, which we learn using unweighted regression, provide a safety net for our estimates and allow us to trade-off the amount of exploration with some extra regret by acting greedily with respect to the over-optimistic function.

\paragraph{Nonlinear function approximation.} RL under nonlinear function approximation has gained increasing emphasis to model complex function spaces like neural networks, which are routinely used in empirical works. Several works have developed rank-based measures to capture the hardness of RL in such settings, in frameworks such as Bellman rank~\citep{jiang2017contextual}, Bellman-Eluder dimension~\citep{jin2021bellman}, Bilinear classes~\citep{du2021bilinear} and DEC~\citep{foster2021statistical}. Many of the upper bounds in these frameworks, however, yield sub-optimal guarantees when specialized to linear or tabular MDPs owing to their generality. An alternative approach builds on the Eluder dimension framework of~\citet{russo2013eluder}, which has been extended to model-free RL in~\citet{wang2020reinforcement}. A related class of problems is that of small $Q$-type Bellman-Eluder dimension studied in~\citet{jin2021bellman}. Among these, the work of~\citet{wang2020reinforcement} is closest to ours. Like their work, we also use model-free regression to estimate value functions, use an Eluder dimension style argument to control the exploration complexity, and use their sensitivity sampling argument to create a bonus oracle. However, we use weighted regression for function fitting, and correspondingly use a generalized Eluder dimension to handle such weighted objectives. Our generalization of the Eluder dimension is based on recent ideas from active learning in~\citet{GWZ22}, although their work does not consider weighted settings. A related analysis of RL with non-linear function approximation with a similar definition of the Eluder dimension is also carried out in~\citet{TZ23-lt}, but they do not consider weighted regression and the guarantees are sub-optimal in $d$ and $H$ factors.
Compared with~\citet{wang2020reinforcement}, we also avoid paying a state-action covering number in the leading order term in the regret in Table~\ref{tbl:results}.

\section{Preliminaries}\label{sec:prelims}

We consider the following time-inhomogeneous episodic Markov Decision Process (MDP) $\calM = (\calS, \calA, \calP \defeq \{P^h\}_{h\in[H]}, \calR\defeq \{r^h\}_{h\in[H]})$ with horizon length $H\in\Z_{>0}$, where $[H]$ is a shorthand for the set $\{1,2,\ldots,H\}$. Here we let $P^h:\calS\times\calA\rightarrow \Delta^\calS$ and $r^h:\calS\times\calA\times \calS\rightarrow [0,1]$ characterize the transition kernel and instantaneous reward at a given level $h\in[H]$ respectively. We consider a sparse reward setting where $\sum_{h\in[H]}r^h\in[0,1]$ under the realization of any policy. We use $(x^h,a^h)\in\calS\times\calA$ 
to denote an arbitrary state-action pair at level $h$ (omitting $h$ when clear from context), and write $z=(x,a)$ as shorthand. A policy $\pi:\calS\rightarrow \calA$ is a mapping from state space to action space.\footnote{Many works also consider randomized policies $\pi:\calS\rightarrow\Delta^\calA$ in reinforcement learning. In this paper it suffices to constrain to the class of deterministic policies.} Since the optimal policy is non-stationary in an episodic MDP, we use $\pi$ to refer to the $H$-tuple $\{\pi^h\}_{h\in[H]}$

Given an episodic MDP $\calM$ and some policy $\pi\defeq \{\pi^h\}_{h\in[H]}$, the $V$-value and $Q$-value functions are defined as the expected cumulative rewards, starting at level $h$ from state-action pair $z^h = (x^h,a^h)$ or state $x^h$, when following the policy $\pi$, i.e.\
\begin{equation}\label{def:value-function}
\begin{aligned}
    Q_\pi^h\Par{x^h, a^h} & = \E\bigg[\sum_{h'\ge h}r^{h'}\Par{x^{h'}, a^{h'}}~|~a^{h'} = \pi^{h'}(x^{h'})~\text{for}~h'\ge h+1\bigg], \quad
    V_\pi^h\Par{x^h} & = Q_\pi^h\Par{x^h, \pi^h\Par{x^h}}%
    .
\end{aligned}
\end{equation}
Given some initial distribution $q\in\Delta^\calS$, the optimal policy $\pi_\star \in \arg\max_{\pi^h, h\in[H]}\E_{x^1\sim q}V_\pi^1(x^1)$. For simplicity we write $Q_\star^h = Q_{\pi_\star}^h = \sup_{\pi} Q_\pi^h$ and $V_\star^h = V_{\pi_\star}^h = \sup_{\pi} V_{\pi}^h$ when clear from context.

We define the Bellman optimality operator $\calT$ on functions $f:\calS\rightarrow \R$ so that $\calT f (x^h,a^h) = \E_{r^h, x^{h+1}}[r^h+f(x^{h+1})|x^h,a^h]$. We frequently use the shorthand $f(x) = \max_{a\in\calA} f(x,a)$. The definition of value functions ensures that they satisfy the \emph{Bellman equation} such that $Q_\star^h(x^h,a^h) = \calT V_\star^{h+1} (x^h, a^h)$. We also define the Bellman operator for second moment as $\calT_2 f (x^h,a^h) = \E_{r^h, x^{h+1}}\left[\left(r^h+f(x^{h+1})\right)^2|x^h,a^h\right]$. 

We consider a class of episodic MDPs such that the value functions satisfy the (approximate) completeness assumption under a general function class $\calF \defeq \{\calF^h\}_{h\in[H]}$. More concretely, we introduce the following assumption:

\begin{assumption}[$\epsilon$-completeness under general function approximation]\label{ass:eps-realizability-RL}
Given $\{\calF^h\}_{h\in[H]}$ where each set $\calF^h$ is composed of functions $f^h:\calS\times\calA\rightarrow [0,L]$. We assume for each $h\in[H]$, and any $V:\calS\rightarrow [0,1]$ there exists $f^h\in\calF^h$ such that 
\begin{align*}
    \max_{x,a\in\calS\times\calA}\big|f^h(x,a) - \calT V(x,a)\big|\le \epsilon,
    ~~\text{and}~\max_{x,a\in\calS\times\calA}\big|f^h(x,a) - \calT_2 V(x,a)\big|\le \epsilon.
\end{align*}
Also we assume there exists some $f_\star^h\in\calF^h$ such that $\norm{f_\star^h-Q_\star^h}_\infty\le \epsilon$, for all $h\in[H]$. We assume $L=O(1)$ throughout the paper.
\end{assumption}

When $\epsilon = 0$, the assumption states that the function class is complete and well-specified for the MDP. 
Such a completeness assumption is standard for analyzing value-based methods relying on squared regression~\citep[see e.g.][]{chen2019information,wang2020reinforcement,jin2021bellman}. While this assumption can be avoided in an information-theoretic sense using ideas developed in~\citet{jiang2017contextual} and follow-ups, we make this assumption in the interest of obtaining sharper guarantees. The assumption naturally holds for tabular and linear MDPs. More generally, $\epsilon$ allows us to capture a bounded misspecification. When we instantiate the function class as a cover of a larger infinite class, the covering might also induce a non-zero $\epsilon$ in Assumption~\ref{ass:eps-realizability-RL}. 

Since we use linear MDPs as a running example to illustrate our key definitions and assumptions, we define them formally next.

\begin{definition}[Linear MDPs~\citep{yang2020reinforcement,jin2020provably}]\label{def:linear}
An MDP $\calM = (\calS,\calA,\calP,\calR)$ is a linear MDP if there exists a known feature mapping $\phi^h:\calS\times\calA\rightarrow\R^d$ for every $h\in[H]$, such that for any $h\in[H]$, and any $(x^h,a^h)\in\calS\times\calA$, we have
\[
P^h(\cdot|x^h, a^h) = \langle \phi^h(x^h,a^h), \mu^h(\cdot)\rangle~\text{and}~\E\left[ r^h|x^h, a^h\right] = \langle \phi^h(x^h, a^h),\theta^h\rangle,
\]
for some unknown measures $\mmu^h = \{\mu^h(x)\}_{x\in\calS}$ where each $\mu^h(x)\in\R^d$ and $\theta^h \in\R^d$.

Without loss of generality we make the assumptions that $\sup_{x,a\in\calS\times\calA}\norm{\phi(x,a)}_2\le 1$, $\norm{\sum_{x\in\calS}\mu^h(x)}_2+\norm{\theta^h}_2\le B^h$, and $\sum_{h\in[H]}r^h\in[0,1]$.
\end{definition}

\citet{jin2020provably} show that linear MDPs satisfy the completeness assumption under the linear function class $\calF_\lin^h$ defined as
\begin{equation}\label{def:linear-F-h}
\begin{aligned}
    \calF_\lin^h & \defeq\{\langle w^h, \phi^h(\cdot,\cdot) \rangle: w^h\in\R^d,\|w^h\|_2\le B^h\},~\text{for any}~h\in[H].\\
\end{aligned}
\end{equation}
We also define $\calF_\lin^h(\epscov)$ be an $\epscov$-cover of $\calF_\lin^h$ under the $\ell_\infty$ norm, so that $\log|\calF_\lin^h(\epscov)| = O(d\log(B^h/\epscov))$.

While the completeness assumption allows us to control the error of our regression solution to $Q_\star$ under the data distribution used in regression, it does not control the complexity of exploration in the MDP, when the learner uses the classes $\{\calF^h\}_{h\in[H]}$. To capture this complexity, we now define an additional quantity which we call a \emph{generalized Eluder dimension}, which extends the original definition of~\cite{russo2013eluder} to weighted regression settings, based on recent work of~\citet{GWZ22} (also see \cite{TZ23-lt}).

\begin{definition}[Generalized Eluder dimension]\label{def:general-eluder-RL}
Let $\lambda > 0$, a sequence of state-action pairs $Z = \{z_i\}_{i\in[T]}$ and $\bm{\sigma} = \{\sigma_i\}_{i\in[T]}$ be given. The \emph{generalized Eluder dimension} of a function class $\calF~:~\cX\times\cA\to [0,L]$ is given by $\dim_{\alpha,T}(\calF)  \defeq\sup_{Z,\bm{\sigma}:|Z|=T, \bm{\sigma}\ge \alpha}\dim(\calF,Z,\bm{\sigma})$, where
\begin{align*}
\dim(\calF,Z,\bm{\sigma}) & \defeq\sum_{i=1}^T \min\left(1,\frac{1}{\sigma_i^2}D^2_\calF(z_i;z_{[i-1]}, \sigma_{[i-1]})\right),\\
    ~~\text{where}~~D^2_\calF(z;z_{[t-1]}, \sigma_{[t-1]}) & \defeq  \sup_{f_1,f_2\in\calF}\frac{\left(f_1(z)-f_2(z)\right)^2}{\sum_{s\in[t-1]}\frac{1}{\sigma_s^2}\left(f_1(z_{s})-f_2(z_s)\right)^2+\lambda}.
\end{align*}
We also use the notation $d_\alpha \defeq \frac{1}{H}\sum_{h\in[H]}\dim_{\alpha,T}(\calF^h)$ when function class $\{\calF^h\}_{h\in[H]}$ is clear from context.
\end{definition}

For linear MDPs, the definition of generalized Eluder dimension can be simplified as follows:
\begin{restatable}[Eluder dimension of linear MDPs]{lemma}{remarkeluderlinear}\label{remark:eluder-linear}
For linear MDPs and the class $\calF_\lin^h$ defined in~\eqref{def:linear-F-h}, letting $\calF_\lin^h(\epscov)$ be the $\epscov$-cover of $\calF_\lin^h$ for some $\epscov>0$, we have $\dim_{\alpha, T}(\calF^h_\lin(\epscov))\le \dim_{\alpha, T}(\calF^h_\lin) =$  \mbox{$O\big(d\log\big(1+\frac{(B^h)^2T}{\alpha^2 d\lambda}\big)\big)$} $= \widetilde{O}(d)$.
\end{restatable}

\begin{remark}[Relation to standard Eluder dimension]\label{remark:comp-standard-Eluder}
When $\bm{\sigma} \equiv 1$, we have  $\max_{Z:|Z|=T}\dim(\calF,Z,\1) \leq \dim_E(\calF, \sqrt{\lambda/T})+1$, where $\dim_E(\calF,\varepsilon)$ is the standard Eluder dimension of $\calF$ as defined in~\cite{russo2013eluder}. The unweighted version of our definition has also appeared in~\cite{GWZ22}. We note the generalized definition we give also takes supremum over weight sequence $\bm{\sigma}$ for any $\bm{\sigma}\ge \alpha$, and thus is incomparable with the standard Eluder dimension even when $\alpha=1$.
\end{remark}

\begin{proof}[Proof of~\Cref{remark:comp-standard-Eluder}]
Suppose $\dim_{E}(\calF,\varepsilon) = n$. By definition of Eluder dimension, for any length $T$ sequence $Z$, there are at most $n$ distinct (sorted) indices $t_i\in[T]$, $i\in[n]$ such that for $z_{t_i}$, $D_{\calF}^2(z_{t_i},1; z_{[t_i-1]}, \1_{[t_i-1]})\ge \frac{\varepsilon^2}{\varepsilon^2+\lambda}$. We then bound 
\[
 \max_{Z:|Z|=T}\dim(\calF,Z,\1) \le \max_{Z:|Z|=T}\Par{n+\frac{\varepsilon^2}{\varepsilon^2+\lambda}T} = n+\frac{\varepsilon^2}{\varepsilon^2+\lambda}T\le \dim_{E}(\calF,\varepsilon)+\varepsilon^2\cdot\frac{T}{\lambda}.
\]
Here for the first inequality we use the definition of $\dim(\calF,Z,\1)$ and the fact that only $n$ terms in the  summation of $\min(1,D_\calF^2)$ are upper bounded by $1$ instead of  $\frac{\varepsilon^2}{\varepsilon^2+\lambda}$ as argued above. Plugging in the choice of $\varepsilon = \sqrt{\lambda/T}$ concludes the proof.
\end{proof}
We use the learning protocol in episodic reinforcement learning where at every episode $t\in[T]$ and horizon level $h\in[H]$, the learner explores the trajectory based on some exploration rule that only depends on the historical data. The learner then generates new data $\{x_t^h, a_t^h, r_t^h\}_{h\in[H]}$ based on her data exploration rule, chooses action $a_t^h$, transitions to the next state $x_{t}^{h+1}\sim P^h(x_t^h, a_t^h)$ and receives reward $r_t^h = r^h(x_t^h, a_t^h, x_{t}^{h+1})$. The goal of learner is to optimize her \emph{regret} while interacting with the environment in $T$ episodes, where the initial distribution $x^1\sim\mu$ generates from some given fixed initial distribution. Formally, we define the regret as:
\begin{align}\label{def:regret}
    \textup{Regret}(T) &~= \sum_{t\in[T]}\E_{x^1\sim \mu}\left[V^1_\star\Par{x^1}-V_t^1\right],~~\text{where}~V_t^1~\defeq \E\left[\sum_{h\in[H]}r_t^h|x^1, f_{t,1}^{[h]}, f_{t,2}^{[h]}\right].
\end{align}
Here $V_t^1$ denotes the expected cumulative reward in the $t_{th}$ trajectory, where the exploration policy may depend on some functions $f_{t,1}^{[H]}$, $f_{t,2}^{[H]}$ based on the history (see exploration rule~\eqref{eq:greedy-policy-informal}). 
\section{Algorithm}\label{sec:algo}

We now discuss our algorithm, \alglong (\alg) in detail. The pseudocode for the algorithm is presented in Algorithm~\ref{alg:fitted-Q-simpler}. At a high-level, the algorithm performs optimistic $Q$-learning style updates (\Cref{line:optimistic-start} to~\Cref{line:optimistic-end}), where we repeatedly apply the empirical Bellman optimality operator at each level $h$ to an optimistic value function for $h+1$, and add an additional bonus to the resulting function to account for the regression uncertainty in the empirical Bellman operator. We additionally maintain over-optimistic and over-pessimistic estimates (\Cref{line:overly-start} to~\Cref{line:overly-end}), which are combined to form both a variance estimate in reweighting our regression objective, as well as in defining the data collection policy in~\Cref{eq:greedy-policy-informal} (used in \Cref{line:greedy-policy}). We first describe some of the key elements of the algorithm in isolation, before discussing how they fit together in the pseudocode. 

For brevity, we only specify the parameters of the algorithm somewhat informally in the following discussion, focusing on the dependency on the key parmeters. Precise parameter settings can be found in Table~\ref{tbl:parameters} in the Appendix.

\begin{algorithm}[t!]
\caption{\alglong (\alg)}\label{alg:fitted-Q-simpler}
\DontPrintSemicolon
\textbf{Input:} function class $\{\calF^h\}_{h\in[H]}$, a consistent bonus oracle $\oracle$, $\epsilon>0$\;
\textbf{Parameters:} $\{u_t\}_{t\in[T]}$, $\lambda$, bonus error $\epsc$, $\alpha$, $\delta$, $\{\beta_{t,1}^h,\beta_{t,2}^h,\bbeta_t^h\}_{t\in[T]}^{h\in[H]}$\;
\textbf{Initialize} $\calD_{[0]}^h = \emptyset$ for all $h\in[H]$\;
\For{episode $t=1,2,\cdots, T$}{
Initialize last step $f^{H+1}_{t,j}(\cdot)\gets0$, for all $j = 1,2,-2$\;
\If{$t>1$}{
\For{$h=H,H-1,\cdots,1$}{
\tcp{constructing CI for optimistic $Q$-value functions}
When $t>1$, define $\bsigma_{t-1}^h$ as in~\Cref{eq:def-bsigma-informal}\;
Solve $\hat{f}^h_{t,1} = \arg\min_{f^h\in\calF^h}\sum_{s\in[t-1]}\frac{1}{(\bsigma_s^{h})^2}\left(f^h(x_s^{h}, a_s^{h})-r^{h}_s-f_{t,1}^{h+1}(x_s^{h+1})\right)^2$\label{line:optimistic-start}
Set $b_{t,1}^h \gets \oracle\Par{\{\bsigma_s^h\}_{s\in[t-1]},\calD_{[t-1]}^h, \calF^h,\hat{f}_{t,1}^h,\beta_{t,1}^h,\lambda, \epsc}$\tcp*{see~\Cref{def:bonus-conditions}}\label{line:bonus-one}
Update $f_{t,1}^h(\cdot) =\min\left(\hat{f}_{t,1}^h(\cdot)+b_{t,1}^h(\cdot)+\epsilon,1\right)$ \label{line:def-b}\; 
Update optimistic $V$-value $f^h_{t,1}(x) = \max_{a} f^h_{t,1}(x,a)$ for all $x\in\calS$\;\label{line:optimistic-end}
\tcp{constructing CI for overly optimistic, pessimistic $Q$-values}
Solve $\hat{f}^h_{t,j} = \arg\min_{f^h\in\calF^h}\sum_{s\in[t-1]}\left(f^h(x_s^{h}, a_s^{h})-r^{h}_s-f_{t,j}^{h+1}(x_s^{h+1})\right)^2$~for~$j = \pm 2$\label{line:overly-start}\;
Set $b_{t,2}^h \gets \oracle\Par{\1_{[t-1]}, \calD_{[t-1]}^h,\calF^h, \hat{f}_{t,2}^h,\beta_{t,2}^h,\lambda, \epsc}$ \tcp*{see~\Cref{def:bonus-conditions}}\label{line:bonus-two}
Update $f_{t,2}^h(\cdot) = \min\left(\hat{f}_{t,2}^h(\cdot)+2b_{t,1}^h(\cdot)+b_{t,2}^h(\cdot)+3\epsilon,2\right)$\label{line:def-b-overly-opti}\;
Update $f_{t,-2}^h(\cdot) = \max\left(\hat{f}_{t,-2}^h(\cdot)-b_{t,2}^h(\cdot)-\epsilon,0\right)$\label{line:def-b-overly-pess}\;
Update $f^h_{t,\pm2}(x) = \max_{a} f^h_{t,\pm2}(x,a)$ for all $x\in\calS$\label{line:overly-end}\; 

\label{line:greedy-policy}

\tcp{constructing CI for estimating variance uing overly optimistic pessimistic $Q-value$}
Solve $\hat{g}^h_t = \arg\min_{g^h\in\calF^h}\sum_{s\in[t-1]}\left(g^h(x^{h}_s,a^{h}_s)-\left(r^{h}_s+f^{h+1}_{t,2}(x_s^{h+1})\right)^2\right)^2$\;\label{line:second-start}

}
}
Receive initial state $x^1_t\sim\mu$\;
\For{$h=1,2,\cdots, H$}{
Generate $\calD_t^{[H]}$ from $x_t^1$ according to $u_t$ and data collection policy~(\Cref{eq:greedy-policy-informal})\;
\tcp{define variance using overly optmistic, pessimistic $Q$-values}
\eIf{$t=1$}{
$\left(\sigma^{h}_t\right)^2 = 4$\;}
{%
\setlength{\abovedisplayskip}{-20pt}
\setlength{\belowdisplayskip}{-10pt}
\begin{flalign}
\left(\sigma^h_t\right)^2 &=\min\left(4, \hat{g}_t^h(z_t^h)-\left(\hat{f}_{t,-2}^h(z_t^{h})\right)^2\right.\nonumber &\\
&\left.+D_{\calF^h}(z_t^h; z_{[t-1]}^h, \1_{[t-1]}^h)\cdot\left(\sqrt{\left(\bbeta_t^h\right)^2+\lambda}+2L\sqrt{\left(\beta_{t,2}^h\right)^2+\lambda}\right)+2(1+L)\epsilon\right) &
\label{eq:sigma-def-alg}
\end{flalign}\label{line:def-variance}
}
}
}
\end{algorithm}

\paragraph{Regression and weighted regression.}
In episodic reinforcement learning, many online algorithms iteratively solve the following (weighted) regression problem to fit the past dataset: At episode $t$, given a target function $f_t^{h+1}\approx V^{h+1}_\star$, we define
\begin{equation}\label{def:confidence-crude-H}
\begin{aligned}
\text{least square estimator}~\hat{f}_{t}^h & = \arg\min_{f^h\in\calF^h} \sum_{s\in[t-1]}\frac{1}{\Par{\bsigma_s^h}^2}\left(f^h\left(x^{h}_s,a^{h}_s\right)-r^{h}_s-f_t^{h+1}(x^{h+1}_s)\right)^2,\\
\text{and version space}~\calF^h_{t} & \defeq \left\{f^h\in\calF^h:\sum_{s\in[t-1]}\frac{1}{\left(\bsigma_s^h\right)^2}\left(f^h(x_s^{h}, a_s^{h})-\hat{f}^h_t(x_{s}^{h},a_s^{h})\right)^2\le \left(\beta_{t}^h\right)^2\right\}.
\end{aligned}
\end{equation}
Standard approaches in $Q$-learning solve least squares problems like~\Cref{def:confidence-crude-H} by solving an unweighted regression with $\bsigma^h \equiv 1$ for all $h \in [H]$, where $f_t^{h+1} \approx Q_\star^{h+1}$. Note that we take a backup of the iterate $f_t^{h+1}$ instead of $f_{t-1}^{h+1}$ as it is natural to do bottom-up approximate dynamic programming in a finite horizon setting. Optimistic variants of $Q$-learning~\citep{jin2020provably,wang2020reinforcement} use some optimistic estimate $f_t^{h+1}$ of $Q_\star^{h+1}$ instead. More recently, a line of work has studied the benefits of using weights informed by the variance to satisfy $\rV\left[r^h+f_t^{h+1}(x^{h+1})|z_s^h\right]\le \Par{\bsigma_s^h}^2\le \widetilde{O}\Par{1}$, for obtaining stronger guarantees in terms of their horizon dependence in linear bandits and linear mixture MDP settings~\citep{zhou2021nearly,zhou2022computationally}, as well as in linear MDPs~\citep{hu2022nearly}. Our formulation of weighted regression here is motivated by these works, extending such techniques to a non-linear and model-free setting.

The radius of confidence interval $\beta_t^h$ is properly chosen to ensure that $\calT f_t^{h+1}\in\calF^h_t$  with high probability (up to small additive error element-wise due to~\Cref{ass:eps-realizability-RL}). The solution of this regression problem also admits pointwise confidence bounds on the error to $Q_\star^h$ under the \emph{bounded Eluder dimension} condition (see~\Cref{def:general-eluder-RL}).

\paragraph{Optimistic value estimation and bonus oracle.}

We now concretely describe how we use weighted regression to construct an optimistic estimate of $Q_\star^h$. Since $Q_\star^h(x,a) = \E[r^h + V_\star^{h+1}(x') | x,a]$, there are two source of uncertainty which need to be upper bounded. First is from the error in our estimates of $V_\star^{h+1}$, that is addressed by using a regression target $f_t^{h+1}$ which is optimistic for $V_\star^{h+1}$. The second source is the estimation error in the conditional expectation using samples at step $h$. In simple settings where the space $\cX$ is discrete or when the class $\cF$ is linear, this error is quantified as an optimistic bonus using either the number of samples for $x,a$ or the standard elliptical bonus~\citep[see e.g.][]{abbasi2011improved,jin2020provably}. An optimistic function at time $h$ is then defined as $f_{t,1}^h = \hat{f}_t^h + b_t^h$, where $b_t^h$ is the optimistic bonus and $\hat{f}_t^h$ is as defined  in~\eqref{def:confidence-crude-H} with an optimistic target $f_{t,1}^{h+1}$. The reason for denoting the optimistic function as $f_{t,1}^h$ instead of $f_t^h$ will be shortly clarified when we define an additional overly optimistic sequence.

For a general function class, we can capture the regression uncertainty as $b_t^h(z^h) = \max_{f^h\in\calF^h_t} f^h(z^h) - \min_{f^h\in\calF^h_t} f^h(z^h)$ for all $z^h\in\calS\times\calA$~\citep{feng2021provably}. However, this uncertainty bonus has a high complexity in that the maximizing and minimizing functions can differ arbitrarily for each $z^h$. Consequently, the target function $f_{t,1}^{h+1}$ defined using $\hat{f}_t^{h+1} + b_t^{h+1}$ is very complex.  Since $f_t^{h+1}$ is random in the regression objective~\eqref{def:confidence-crude-H}, this high complexity induces a potentially poor confidence bound for the solution of regression. To circumvent this issue, we assume the existence of a low complexity bonus oracle which roughly dominates the value obtained by the pointwise maximization over $\calF_t^h$ for now. We subsequently provide concrete instantiations of this bonus oracle for linear MDPs and general settings with a low Eluder dimension in Section~\ref{sec:example}. 

\begin{definition}[Oracle $\oracle$ for bonus function $b_t^h$]\label{def:bonus-conditions}
Given index $h\in[H]$, $t\in[T]$, sequence of $\{\bsigma_{s}^h\}_{s\in[t-1]}$ and dataset $\calD_{[t-1]}^h = \{(x_s^h, a_s^h, r_s^h, x_s^{h+1})\}_{s\in[t-1]}$, function class $\calF^h$ with $\hat{f}^h\in\calF^h$, $\beta^h,\lambda\ge0$, error parameter $\epsc \ge0$, the \emph{bonus oracle} $\oracle(\{\bsigma_{s}^h\}_{s\in[t-1]}, \calD_{[t-1]}^h, \calF^h, \hat{f}^h,\beta^h,\lambda,\epsc)$ outputs a bonus function $b^h(\cdot)$ such that 
\begin{itemize}
    \item $b^h:\calS\times\calA\rightarrow \R_{\ge 0}$ belongs to function class $\calW$.
    \item $b^h(z^h)\ge \max\left\{|f^h(z^h)-\hat{f}^h(z^h)|,~f^h\in\calF^h:\sum_{s\in[t-1]}\frac{\Par{f^h(z_s^h)-\hat{f}_t^h(z_s^h)}^2}{\Par{\bsigma_s^h}^2}\le \Par{\beta^h}^2\right\}$~~for any~$z^h\in\calS\times\calA$.
    \item $b^h(z^h)\le C\cdot\bigl( D_{\calF^h}(z^h;z_{[t-1]}^{h},\bsigma_{[t-1]}^{h})\cdot\sqrt{\left(\beta^h\right)^2+\lambda}+\epsc\cdot\beta^h\bigr)$  for all $z^h\in\calS\times\calA$ with constant $0<C<\infty$.
\end{itemize}
Further we say the oracle $\oracle$ is \emph{consistent} if for any $t<t'$ with consistent $\{\bsigma_s^h\}_{s\in[t-1]}\subseteq \{\bsigma_s^h\}_{s\in[t'-1]}$, $\calD_{[t-1]}^h\subseteq\calD_{[t'-1]}^h$, $\beta_t^h$ non-decreasing in $t$ for each $h\in[H]$ and $\calF^h_t, \hat f_t^h$ as defined in~\eqref{def:confidence-crude-H}, it holds that $\oracle(\{\bsigma_{s}^h\}_{s\in[t-1]}, \calD_{[t-1]}^h, \calF_t^h, \hat{f}_t^h,\beta_t^h,\lambda,\epsc)\ge \oracle(\{\bsigma_{s}^h\}_{s\in[t'-1]}, \calD_{[t'-1]}^h, \calF_{t'}^h, \hat{f}_{t'}^h,\beta_{t'}^h,\lambda,\epsc)$ element-wise.
\end{definition}

With such an oracle $\oracle$, we define the optimistic value function sequence $f_{t,1}^{h} \approx \hat{f}_{t,1}^h+b_{t,1}^h$ (approximation due to additive $\eps$ term and truncation), where $b_{t,1}^h = \oracle(\{\bsigma_{s}^h\}_{s\in[t-1]}, \calD_{[t-1]}^h, \calF_t^h, \hat{f}_t^h,\beta_{t,1}^h,\lambda,\epsc)$, $\beta_{t,1}^h \approx \sqrt{\log \cN}$ and $\hat{f}_{t,1}^h$ is a solution to~\eqref{def:confidence-crude-H} with $f_{t,1}^{h+1}$ as the target function.

\paragraph{Overly optimistic and overly pessimistic value estimates.} 

A sharp analysis of the convergence of the optimistic estimates $f_{t,1}^h$ requires appropriately estimated variances for weighting the regression examples, which satisfy $\rV\left[r^h+f_{t,1}^{h+1}(x^{h+1})|z_s^h\right]\le \Par{\bsigma_s^h}^2$ as mentioned before. In order to produce such variance estimates, we first define an auxiliary set of value function estimates, which are then used in variance estimation. Specifically, we define an \emph{overly optimistic} estimate $f_{t,2}^h$, as well as an \emph{overly pessimistic} estimate $f_{t,-2}^h$. At a high level, these functions are designed to ensure that 

\begin{equation}\label{eq:mono-informal}
\underbrace{f_{s,-2}^h(z^h)}_{\text{overly-pessimistic}}\le \underbrace{f_\star^h(z^h)}_{\text{$\defeq Q_\star^h$, true optimal}} \le \underbrace{f_{t,1}^h(z^h)}_{\text{optimistic}} \le \underbrace{f_{s,2}^h(z^h)}_{\text{overly optimistic}}~\text{for any}~z^h\in\calS\times\calA~\text{and}~s\le t.
\end{equation}

Concretely, the functions $f_{t,\pm2}^h$ are defined by finding an unweighted regression solution $\hat{f}_{t,\pm2}^h$ to~\eqref{def:confidence-crude-H} with target function $f_{t,\pm 2}^{h+1}$, and defining $f_{t,\pm2}^h \approx \hat{f}_{t,\pm2}^h \pm 2b_{t,2}^h\pm b_{t,1}^h$ (approximation due to $\eps$ term and truncation), for some bonus $b_{t,2}^h$ defined using our bonus oracle. A key difference, however, is that the weight sequence $\bsigma \equiv 1$ in defining $\hat{f}_{t,\pm2}^h$, so that we use standard unweighted regression in~\eqref{def:confidence-crude-H}. We note that the idea of having an overly optimistic sequence for variance estimation is first introduced in~\citet{hu2022nearly}. However, they do not perform unweighted regression like us, which leads to some technical problems in their analysis as described in Appendix~\ref{app:comp}. The bonuses $b_{t,2}^h$ are given by $\oracle(\1_{[t-1]}, \calD_{[t-1]}^h, \calF_t^h, \hat{f}_t^h,\beta_{t,2}^h,\lambda,\epsc)$ and $\beta_{t,2}^h \approx \sqrt{\log \cN\cN_b}$. We note that the bonus multiplier $\beta_{t,2}$ also incorporates the complexity of the bonus oracle class $\cW$. For intuition, $\beta_{t,1} = \otil(\sqrt{d})$ in a linear MDP, while $\beta_{t,2} = \otil(d)$, and making this distinction is crucial to obtaining the asymptotically optimal $d$ dependence in our bounds.

\paragraph{Estimating variance.} Next we discuss how to construct an appropriate variance upper bound $(\bsigma_s^h)^2$. Unlike the convenient condition of $(\bsigma_s^h)^2 \geq \rV_{r^h, x^{h+1}}[r^h  + f_{t,1}^{h+1}(x^{h+1}) | x_s^h, a_s^h]$, which necessitates reasoning about a changing target, we perform a more careful analysis and leverage properties of the overly optimistic function $f_{s,2}^h$ to show that it suffices to ensure that $(\bsigma_s^h)^2 \geq \rV[r^h  + f_{\star}^{h+1} (x^{h+1})| x_s^h, a_s^h]$ at all rounds $s \leq t$. This change, which crucially relies on \eqref{eq:mono-informal}, fixes the target function to be $f_\star^h$, and enables the creation of a valid variance estimate.

Since variance involves second moment and squared expectation, we estimate the two separately. We estimate the second moment directly using an \emph{unweighted regression} described below.
\begin{equation}\label{eq:C-update-rule-RL-second-informal}
\begin{aligned}
\forall~t\in[T], h\in[H],~~~\hat{g}_t^h & = \arg\min_{g^h\in\calF^h} \sum_{s\in[t-1]}\left(g^h\left(x^{h}_s,a^{h}_s\right)-\left(r^{h}_s+f_{t,2}^{h+1}\left(x_s^{h+1}\right)\right)^2\right)^2,
\end{aligned}
\end{equation}
and choose $\bbeta_t^h \approx \sqrt{\log \calN\calN_b}$.

A natural variance upper bound at step $t$ can then be obtained as $\hat{g}_t^h - (f_{t,-2}^h)^2$, but we need additional terms to account for the estimation errors and obtain a valid upper bound. This is achieved through the sequence $\bsigma_t^h$ defined as (informally here, see~\Cref{eq:def-bsigma} for the precise setting)
\begin{align}
    \sigma_t^h~&~\text{is as defined in \Cref{eq:sigma-def-alg}},\nonumber\\
    \bsigma_{t}^h&=  \max\left\{\sigma_{t}^h, \alpha,\sqrt{\widetilde{\Theta}\Par{\log \Par{\calN\calN_b}}\cdot\Par{f_{t,2}^h(z_t^h)-f_{t,-2}^h(z_t^h)}},\sqrt{\widetilde{\Theta}\Par{\log \Par{\calN\calN_b}}\cdot D_{\calF^h}\Par{z_t^h; z_{[t-1]}^h,\bsigma_{[t-1]}^h}}\right\}  ,\label{eq:def-bsigma-informal}
\end{align} 
where $\widetilde{\Theta}(\cdot)$ contains logarithmic factors in $T,H,L,1/\alpha,1/\delta$ that are precisely defined in ~\Cref{eq:def-bsigma}.

\paragraph{Design of exploration policy.} A natural exploration policy, given the optimistic sequence $f_{t,1}^h$ is to be greedy with respect to it. However, it is beneficial to sometimes act greedily with respect to the overly optimistic function $f_{t,2}^h$ if the two sequences begin to differ by a lot. At iteration $t$ we choose actions as per the following rule:
\begin{equation}\label{eq:greedy-policy-informal}
a_t^h =\left\{\begin{array}{cc}
     \argmax_{a\in\cA} f_{t,1}^h(x_t^h,a) & \mbox{if $f_{t,1}^h(x_t^{h'}) \geq f_{t,2}^h(x_t^{h'}) - u_t$ for all $h' \leq h$},\\
    \argmax_{a\in\cA} f_{t,2}^h(x_t^h,a) & \mbox{otherwise},
\end{array}    \right.
\end{equation}
where $u_t$ is an appropriately chosen threshold. Note that the action sequence defined this way is not a Markovian policy since it depends on the entire prefix trajectory at each step. However, it does constitute a valid exploration scheme, which suffices for regret minimization. In comparison, \citet{hu2022nearly} also use overly optimistic function in exploration, but there the agent only acts greedily with respect to the overly optimistic function.

\section{Main Result and Applications}\label{sec:example}

We first provide the theoretical guarantee of~\Cref{alg:fitted-Q-simpler}, which is the main result of the paper. We state the in-expectation guarantee and high-probability guarantee respectively, in the next two theorems.

\begin{theorem}[Regret bound for \alg, complete version in~\Cref{thm:regret} and~\Cref{thm:regret-hp}]\label{thm:regret-genera.}
Suppose we are given function classes $\{\calF^h\}_{h\in[H]}$ satisfy~\Cref{ass:eps-realizability-RL} with $\epsilon\in[0,1]$ and~\Cref{def:general-eluder-RL} with $\lambda=1$, and access to a consistent bonus oracle $\oracle$ satisfying~\Cref{def:bonus-conditions} with $TH \epsc = O(1)$. Let $ d_\alpha = \frac{1}{H}\cdot \sum_{h\in[H]}\dim_{\alpha, T}(\calF^h)$ with $\alpha= \sqrt{1/TH}$, and set 
\begin{align*}
    u_t= \order\left(\frac{\sqrt{\log(\calN TH/\delta)+T^2H\epsilon}\cdot(\log(\calN\calN_bTH/\delta) \sqrt{H^5 d_\alpha})}{\sqrt{t}}+H^2\eps+H\delta\right).
\end{align*}
 For any $\delta<1/(T+H^2+11)$ the regret of \alg satisfies 
\begin{align*}
\E R_T = & O\bigg(\sqrt{\log(\calN TH/\delta)+T^2H\epsilon}\cdot\sqrt{THd_\alpha} +\bigl(\log(\calN TH/\delta)+T^2H\eps\bigr)\cdot\log^2(\calN\calN_b TH/\delta) H^{5} d_\alpha\bigg).
\end{align*}
The above regret bound also holds with high probability $1-\delta$.
\end{theorem}

We remark that we don't pay attention to optimize the low-order $\mathrm{poly}(H)$ terms, which may be easily improved through more careful analysis. When $\epsilon = 0$ in Assumption~\ref{ass:eps-realizability-RL} and $T =\widetilde\Omega\left(d_\alpha H^9 \log^4(\calN\calN_b)\log\calN\right) $, we see that the regret scales as $\otil(\sqrt{THd_\alpha\log\cN})$. The dependence on $T$, $\log\cN$ and the generalized Eluder dimension $d_\alpha$ is standard and unimprovable, as we will see in the special case of a linear MDP shortly.  Also, we note that the regret scales as $\sqrt{H}$. This might appear sub-optimal since we assume that the trajectory level rewards, and hence values, are normalized in $[0,1]$. This scaling, however, captures the model complexity of an inhomogeneous process and is unavoidable here due to a matching lower bound in linear MDPs as we discuss in the next section. 

It is natural to ask if the horizon dependence can be removed if the process is homogenous. We suspect that this is not feasible with the current approach as the regression functions across the levels do not share any parameters, and hence fail to leverage the homogenous structure of the problem. Exploring this further is an important direction for future research. 

We give a high-level proof sketch of the theorem in Section~\ref{sec:analysis}, with details deferred to~\Cref{app:proofs}. First we present some specific consequences of the general result using concrete instantiations of the bonus oracle $\oracle$ for both linear and nonlinear function approximation.

\subsection{Linear Function Approximation}\label{ssec:linear}

We apply the theorem to the specific setting of linear MDPs (see~\Cref{def:linear}). Recall that in this case, the finite function class $\cF^h$ is defined as an $\epscov$-cover $\calF_\lin^h(\epscov)$ at each level $h$ for the linear function class $\{\langle w^h, \phi^h(\cdot,\cdot) \rangle: w^h\in\R^d,\|w^h\|_2\le B^h\}$. We define $B = \max_{h\in[H]} B^h$ as a constant (same effect as $L$), so that $\log |\calF_\lin^h(\epscov)| = O\big(d\log(1+B/\epscov)\big)$. Furthermore, we know from Lemma~\ref{remark:eluder-linear} that in this case $d_\alpha = O\big(d\log\big(1+\frac{B^2 T}{\alpha^2d\lambda}\big)\big)$.

The bonus oracle for linear MDPs is naturally instantiated using the standard elliptical bonus, and satisfies all our properties as we show below. We refer readers to~\Cref{app:algo} for a complete proof.

\begin{restatable}[Bonus oracle $\oracle$ for linear MDPs]{lemma}{lemlinearbonus}\label{lem:linear-bonus}
Given $T, H\in\calZ_{+}$, suppose all $\beta_{t}^h\le \beta$ and $\beta_{t}^h$ is non-decreasing in $t\in[T]$ for each $h\in[H]$. For any $t \geq 1$, $h\in[H]$, variance sequence $\{\bsigma_{s}^h\}_{s\in[t-1]}$, dataset $\calD_{[t-1]}^h = \{(\phi^h(z_s^h), a_s^h, r_s^h, \phi(z_s^{h+1}))\}_{s\in[t-1]}$, function class $\calF_t^h$ and $\hat{f}_t^h\in\calF_t^h$ correspondingly defined via weighted regression~\eqref{def:confidence-crude-H}, and parameters $\lambda, \epscov > 0$, define $\oracle(\{\bsigma_{s}^h\}_{s\in[t-1]}, \calD_{[t-1]}^h, \calF_t^h, \hat{f}_t^h,\beta_t^h,\lambda,\epscov) = \|\phi^h(x,a)\|_{(\Sigma_t^h)^{-1}}\sqrt{(\beta^h_t)^2 + \lambda}$, where $\Sigma_t^h = \frac{\lambda}{4(B^h)^2} I + \sum_{s\in[t-1]} \phi^h(z_s^h)\phi^h(z_s^h)^\top$. For any choice of covering radius $\epscov\le \sqrt{\lambda/8T}$, the oracle satisfies all the properties of Definition~\ref{def:bonus-conditions} with
\[\log\calN_b = \log|\calW| = O\biggl(d^2\log\biggl(1+\frac{\sqrt{d}\beta^2}{\lambda\epscov^2}\biggr)\biggr).\]
\end{restatable}

Combined with~\Cref{thm:regret-genera.}, we obtain the following result when applying~\alg to linear MDPs with aforementioned function class $\calF_\lin^h(\epscov), h\in[H]$ and bonus oracle $\oracle$.

\begin{restatable}
[Regret of \alg for linear MDPs]{theorem}{thmexplin}\label{thm:regret-general-linear}
Under conditions of Theorem~\ref{thm:regret-genera.}, suppose that the underlying MDP is linear, so that the original function class $\calF^h_\lin$ satisfies Assumption~\ref{ass:eps-realizability-RL} with $\epsilon = 0$ and the $\epscov$-cover $\calF^h_\lin(\epscov)$ satisfies Assumption~\ref{ass:eps-realizability-RL} with $\eps = \epscov$. Choosing $\lambda=1$, $u_t = \widetilde{\Theta}(d^3H^{5/2}/\sqrt{t})$, $\alpha=\sqrt{1/HT}$, $\epsc = \epscov \le 1/8HT$ and $\delta<1/(T+H^2+11)$, \alg with the bonus oracle defined in~\Cref{lem:linear-bonus}, achieves a total regret of
\[
\E R_T\ = \widetilde{O}\left(d\sqrt{HT}+d^6H^5\right).
\]
The above regret bound also holds with high probability $1-\delta$.
\end{restatable}

The regret bound is \emph{asymptotically optimal} in the leading order term by adapting the construction of~\citet{zhou2021nearly} to our setting. Specifically, we take the construction for linear MDP mentioned in Remark 5.8 of their paper, and rescale the rewards to be $1/H$ in the absorbing state to satisfy our normalization assumption. Their proof is based on embedding $H$ independent linear bandit instances in a horizon $H$ MDP, where the learner needs to solve $\Omega(H)$ of the bandit instances. It can be checked that Lemma C.8 of their analysis, which specifies the regret incurred in each bandit instance, now simply scales to be $\Omega(1/H)$, and the rest of the argument remains unchanged, leading to an overall lower bound of $\Omega(
d\sqrt{HT})$ after $T$ episodes for a horizon $H$ time-inhomogenous linear MDP. 

We note that all prior bonus-based methods suffer from a sub-optimal horizon and dimension scaling for linear MDPs. Given that this comes out of a consequence of a more general result here shows that our algorithm and analysis indeed handle the uncertainty in our predictions in a sharp manner. 

We also note that the Theorem~\ref{thm:regret-general-linear} does not strictly require a linear MDP assumption, since we only require~\Cref{ass:eps-realizability-RL} to hold for $\calF_\lin^h$. This is closely related to the weaker condition on low \emph{inherent Bellman error} studied in~\citet{zanette2019tighter}, but we additionally require closure under the conditional second moment for our function class. The error term $\epsilon$ further allows handling a small model misspecification, like in prior linear MDP and inherent Bellman error results~\citep{jin2020provably,zanette2019tighter}. As a minor point, the instantiation of our general result from Theorem~\ref{thm:regret-genera.} requires us to instantiate $\cF^h_{\lin}$ to be a finite $\epscov$-cover of the linear function class, and perform regression over this cover. It is also possible to directly analyze both the general and this special case in terms of a covering argument, which allows us to run our regressions directly on the original function class, which is preferable in practice. Here we choose the discrete setting for ease of presentation.

\subsection{Nonlinear Function Approximation}\label{ssec:nonlinear}

In the general setting of non-linear $\calF$ with a bound on the generalized Eluder dimension, the ideas from the linear case can be extended by leveraging the techniques of~\citet{wang2020reinforcement}. They define a bonus function roughly as $b_t^h(z) = \max_{f \in \calF^h_t} f(z) - \hat{f}_t^h(z)$. To avoid the high complexity arising with this definition, as remarked in Section~\ref{sec:algo} they instead consider an approximation to the class $\calF_t^h$ in the maximization, defined using a subsampled set of the data. By using standard online subsampling arguments~\cite{kong2021online}, we can argue that the predictive differences between functions are preserved up to constant factors, while the amount of data required is significantly smaller. We state the main guarantees here and refer the readers to~\Cref{app:bonus-condition}  and Algorithm~\ref{alg:online-sensitivity-weighted} for additional details of its implementation.

\begin{restatable}[Implementing $\oracle$ using online-subsampling]{corollary}{onlinesubsample}\label{coro:bonus-oracle}
There exists an algorithm (see~\Cref{alg:online-sensitivity-weighted}) that with probability $1-\delta$ implements a consistent bonus oracle $\oracle$ with $\epsc=0$ for all iterations $t\in[T]$, $h\in[H]$ where 
\[\log|\calW|\le O\Par{\max_{h\in[H]}\dim_{\alpha,T}(\calF^h)\cdot\log\frac{T\calN}{\delta}\log\frac{T|\calS\times\calA|}{\delta}}.\]
\end{restatable}

Note that when applying this approach to linear MDPs, it provides an alternative method to build $\calW$ so that $\log\calN_b = \log|\calW| = \widetilde{O}(d^2\cdot\log|\calS\times\calA|)$.

\begin{restatable}[Regret of \alg using subsampling based bonus oracle $\oracle$]{theorem}{thmexp}\label{thm:regret-general-online}
Under conditions of Theorem~\ref{thm:regret-genera.}, suppose the original function class $\calF^h$ satisfies~\Cref{ass:eps-realizability-RL} with $T^2H\epsilon=O(1)$, choosing $\lambda = 1$, $\alpha= \sqrt{1/TH}$, $\delta<1/(T+H^2+11)$, $\epsc = 0$, , \alg with the bonus oracle $\oracle$ as in~\Cref{coro:bonus-oracle}, and
    $u_t= \widetilde{\Theta}\Par{\frac{\log^{1.5}\calN\cdot\log|\calS\times\calA|\sqrt{d_\alpha}\cdot\max_{h\in[H]}\dim_{\alpha,T}(\calF^h)H^{5/2}}{\sqrt{t}}}$,
 achieves a total regret of~\footnote{Similar to the linear case, high-probability bound has slightly worse low-order dependence on $H$ with different $u_t$ choices.}
\begin{align*}
\E R_T = & \widetilde{O}\bigg(\sqrt{\log\calN d_\alpha H T} +\log^3\calN\cdot\log^2|\calS\times\calA|d_\alpha\cdot\bigg(\max_{h\in[H]}\dim_{\alpha,T}(\calF^h)\bigg)^2H^{5}\bigg).
\end{align*}
The above regret bound also holds with high probability $1-\delta$.
\end{restatable}

Apart from improving upon the dependence in $H$ and $d_\alpha$ relative to the earlier result of \citet{wang2020reinforcement}, as highlighted in Table~\ref{tbl:results}, a key improvement in Theorem~\ref{thm:regret-general-online} is that the log-covering number of the state-action space only appears in the lower order term. This is due to the bonus parameter $\beta_{t,1}^h$ for the optimistic value function being independent of the size of the bonus class, and a key insight in our analysis.

\section{Proof Sketch}\label{sec:analysis}

In this section, we provide a proof sketch of~\Cref{thm:regret-genera.}. For simplicity we focus on proving an informal version of the bound on expected regret. The informality comes from two simplifications: First, we assume the underlying function class used in \alg satisfies~\Cref{ass:eps-realizability-RL} with $\eps\le (T^2H)^{-1}$, and the bonus oracle satisfies conditions in~\Cref{def:bonus-conditions} with $\epsc\le (T^2H)^{-1}$, and $\delta$ is sufficiently small, i.e. $\delta= (TH^2)^{-1}$. Second, we use~~$\widetilde{\cdot}$~~notation to omit logarithmic factors in $T,H,1/\alpha, 1/\delta$ without further specification. We refer readers to~\Cref{app:proofs} and~\ref{app:regret-hp} for the complete analysis of the in-expectation and high-probability bounds, respectively.

\paragraph{Notations.} Throughout the section, we use $z^h = (x^h,a^h)$, and $f_{\star}^h(z^h) = Q_\star^h(z^h)$ for all $z^h\in\calS\times\calA$ and $f_{\star}^h(x^h) = V_\star^h(x^h)$ for all $x^h\in\calS$ interchangeably. 

For each $f_{t,j}^h\in\calF^h$ with $j = 1,\pm2$ as defined in~\Cref{line:def-b},~\ref{line:def-b-overly-opti} and~\ref{line:def-b-overly-pess}, we let $\bar{f}_{t,j}^h(\cdot)\in\calF^h$ to approximate the conditional expectation with target $f_{t,j}^{h+1}$, i.e. $\max_{z^h}\left|\bar{f}_{t,j}^h(z^h) - \calT f_{t,j}^{h+1}(z^h)\right|\le \epsilon$. Similarly, we will let
$\psi_t
^h(\cdot)\in\calF^h$ to approximate the conditional second moment with target $f_{t,1}^{h+1}$, i.e. $\max_{z^h}\left|\psi_t^h(z^h)-\calT_2f_{t,1}^{h+1}(z^h)\right|\le \epsilon$. 
We note such $\bar{f}_{t,j}^h$ for $j=1,\pm2$ and $\psi_t^h$ all
exist due to~\Cref{ass:eps-realizability-RL}.

We also recall the confidence intervals at each step $t\in[T]$, $h\in[H]$ induced by $\calD_{[t-1]}^h$, $\{\bsigma_{s}^h\}_{s\in[t-1]}$ and $\hat{f}^h$ as follows:
\begin{align*}
    \calF_{t,1}^h & \defeq \bigg\{f^h\in\calF^h:\sum_{s\in[t-1]}
\frac{1}{\Par{\bsigma_s^h}^2}\Par{f^h(x_s^h, a_s^h)-\hat{f}_{t,1}^h(x_s^h, a_s^h)}^2\le \Par{\beta_{t,1}^h}^2\bigg\},\\
\calF_{t,j}^h  & \defeq \bigg\{f^h\in\calF^h:\sum_{s\in[t-1]}\Par{f^h(x_s^h, a_s^h)-\hat{f}_{t,j}^h(x_s^h, a_s^h)}^2\le \Par{\beta_{t,j}^h}^2\bigg\}~~~\text{for} ~~j=\pm2,\\
\text{and also}~~\calG_{t}^h  & \defeq \bigg\{g^h\in\calF^h:\sum_{s\in[t-1]}\Par{g^h(x_s^h, a_s^h)-\hat{g}_{t}^h(x_s^h, a_s^h)}^2\le \Par{\bbeta_{t}^h}^2\bigg\}.
\end{align*}
with parameters $\beta_{t,j}^h$ and $\bbeta_t^h$ defined as in~\Cref{tbl:parameters}. We use $\calE_{s}^h$ to denote the event $\{\bar{f}_{s,1}^h\in\calF_{s,1}^h,~\bar{f}_{s,\pm 2}^h\in\calF_{s,\pm 2}^h~\text{and}~\psi_s^h\in\calG_s^h\}$ and use $\calE_{\le t}$ to denote the joint event that $\cap_{s\in[t]}\cap_{h\in[H]}\calE_s^h$.

Finally, we divide the set of iterations into disjoint subsets $[T] = \calTo\cup\calToo$ so that $t\in \calTo$ \emph{if and only if} the agent only uses greedy policy of $f_{t,1}$ throughout the trajectory at episode $t$ (based on rule~\eqref{eq:greedy-policy-informal}).

\paragraph{Proof structure.} We divide the proof into the following three steps. 

\begin{itemize}
    \item The first step shows that with high probability, the good event $\calE_{\le T}$ happens in~\Cref{coro:good-event-fbar}. We use two crucial lemmas~\Cref{lem:mono} and~\ref{lem:RL-variance-bounds} together with concentration arguments to obtain the result.
    \item The second step shows the agent only uses $f_{t,2}^h$ in the exploration \emph{occasionally} by bounding $|\calToo|$ in~\Cref{lem:size-of-Too-informal}.
    \item The third step bounds the summation of bonus terms that govern the confidence bound of the optimistic and overly optimistic functions, i.e.  $\sum_{t\in[T]}\sum_{h\in[H]}\min(1+L, b_{t,1}^h(z_t^h))$ and $\sum_{t\in[T]}\sum_{h\in[H]}\min(1+L, b_{t,2}^h(z_t^h))$, in~\Cref{lem:regret-bound-I-rough} and~\Cref{lem:regret-bound-III-rough}.
\end{itemize}
Combining these pieces allow us to prove the regret bound in expectation. Below we provide the detail of each individual step.

\paragraph{Step 1. Confidence interval and the good event.}

For our chosen value of $\beta_{t,2}^h,\bbeta_t^h\approx\sqrt{\log\calN\calN_b}$ and $\beta_{t,1}^h \approx \sqrt{\log\calN}$, we can show that the confidence intervals are properly defined to contain the conditional expectations of the respective regression targets.

\begin{restatable}{proposition}{coroGoodEvent}\label{coro:good-event-fbar}
Suppose~\Cref{alg:fitted-Q-simpler} uses a consistent bonus oracle satisfying~\Cref{def:bonus-conditions}. With probability $1-5\delta$, the good event $\calE_{\le T}$ happens, that is, $ \bar{f}_{t,1}^h\in\calF_{t,1}^h, \bar{f}_{t,\pm2}\in\calF_{t,\pm2}^h$ and $\psi_t^h\in\calF_t^h$ for all $t\in[T]$ and $h\in[H]$.
\end{restatable}

On the high level, $\bar{f}_{t,\pm2}\in\calF_{t,\pm2}^h$ and $\psi_t^h\in\calF_t^h$ follow from standard concentration arguments for martingale difference sequences and are immediate applications of Freedman's inequality (see~\Cref{coro:Freedman-variant}).  We refer readers to Lemmas~\ref{lem:confidence-interval-overly},~\ref{lem:confidence-interval-RL-pessimistic} and~\ref{lem:confidence-interval-RL-second} for their detailed proofs respectively.

The inclusion $\bar{f}_{t,1}^h\in\calF_{t,1}^h$ requires a more careful argument, since a straightforward approach requires $\beta_{t,1}^h$ to also scale as $\sqrt{\log{\calN\calN_b}}$, which leads to a sub-optimal dependence on $d_\alpha$ in the leading $\sqrt{T}$-order term. Instead we break apart the error between $r^h + f_{t,1}^{h+1}$ and $\bar{f}_{t,1}^h$ into two martingale differences (see~\Cref{lem:confidence-interval-RL} in~\Cref{app:CI} and the quantities $\eta_s^h$, $\Delta_s^h$ defined in~\Cref{lem:confidence-interval-helper-I,lem:confidence-interval-helper-II} respectively). One depends on the error between $r+f_\star^{h+1}$ and its conditional expectation, while the other looks at $f_{t,1}^{h+1} - f_\star^{h+1}$ and its conditional expectation. The first term has a lower complexity as it deals with a fixed target function $f_\star^h$, while the second is lower order with a carefully chosen weighting in $\bsigma_t^h\ge\sqrt{ \widetilde{\Omega}(\log(\calN\calN_b))\cdot(f_{t,2}^h(z_t^h)-f_{t,-2}^h(z_t^h))}$ as in~\eqref{eq:def-bsigma-informal}.
In effect, this is like a model-free analog of the error decomposition first used in~\citet{azar2017minimax} for obtaining optimal regret bounds for finite horizon tabular MDPs, which in our case can be used to remove the dependency on $\log\calN_b$ in the leading order.

The martingale analysis above requires two conditions to hold for all rounds $s \in [t-1]$:
\begin{equation}
   \E_{x^{h+1}}\bigl[|f_{t,1}^{h+1}(x^{h+1})-f_{\star}^{h+1}(x^{h+1})|~|~z_s^h\bigr] \le f_{s,2}^h(z_s^{h})- f_{s,-2}^h(z_s^h)~~\text{and}~~(\sigma_s^h)^2 \geq \rV_{r^h, x^{h+1}}[r^h + f_\star(x^{h+1}) | z_s^h].
    \label{eq:conc-prop-cond}
\end{equation}

The proposition is shown by induction over $t$ and $h$, where we can show the event $\calE_{\leq t}\cap\Par{\cap_{h+1\le h'\le H}\calE_{t+1}^{h'}}$ to establish these conditions at round $t+1$ level $h$, which enables the concentration for the event $\calE_{t+1}^h$. Next two lemmas show that $\calE_{\leq t}\cap\Par{\cap_{h+1\le h'\le H}\calE_{t+1}^{h'}}$ implies~\eqref{eq:conc-prop-cond} for all $s\in[t]$.

\begin{restatable}[Pointwise monotonicity]{lemma}{lemmono}\label{lem:mono}
Suppose~\Cref{alg:fitted-Q-simpler} uses a consistent bonus oracle satisfying~\Cref{def:bonus-conditions}. For any fixed $t\in[T]$, and $h\in[H]$, conditioning on past events $\calE_{\le {t-1}}\cap\Par{\cap_{h'=h}^{H}\calE_{t}^{h'}}$, we have for all $z^h, z^{h-1}\in\calS\times\calA$, 
\begin{align*}
    & 1.~~f^h_\star(z^h)\le f^h_{t,1}(z^h);\\
    & 2.~~f^h_{t,-2}(z^h)\le f^h_{\star}(z^h);\\
    & 3.~~f^h_{s,2}(z^h)\ge \max\bigl(\calT f_{t,1}^{h+1}(z^{h}),f^h_{t,1}(z^h)\bigr)~~\text{for all}~s\in[t].
\end{align*}

Consequently by taking max over $a^h$ we know all inequalities also hold for any $x^h\in\calS$, i.e.  $f^h_{s,-2}(x^h)\le f^h_\star(x^h)\le f^h_{t,1}(x^h)\le f^h_{s,2}(x^h)$ for all $s\in[t]$.
\end{restatable}
This formalizes the intuition given in~\eqref{eq:mono-informal}, and is proven using backward induction together with definition of $\oracle$ and conditioning of $\calE_{\le T}$. We provide its detailed proof in the beginning of~\Cref{app:variance}. A crucial aspect of Lemma~\ref{lem:mono} is that the bound $\max\bigl(\calT f_{t,1}^{h+1}(z^{h}),f^h_{t,1}(z^h)\bigr)\leq f^h_{s,2}$ holds uniformly for all $s \leq t$. That is, the over-optimistic function from an earlier round $s$ bounds not only $f_\star$, but also all future optimistic value functions 
and the conditional expectations induced by future value functions. This property of the over-optimistic sequence allows us to obtain a valid variance estimate in the face of changing target functions, and is the reason we require the auxiliary over-optimistic sequence. It is not clear if a similar property  holds for the optimistic sequence itself.

The monotonicity lemma also leads to an important property of our constructed variance estimate $\sigma_t^h$ in~\eqref{eq:sigma-def-alg}, stated in the next lemma.

\begin{restatable}[Lower bound of variance estimator]{lemma}{lemvariancelower}\label{lem:RL-variance-bounds}
Suppose~\Cref{alg:fitted-Q-simpler} uses a consistent bonus oracle satisfying~\Cref{def:bonus-conditions}. At step $t\ge2$, conditioning on the good event $\calE_{\le t}$, the 
variance estimate $\sigma_t^h$ satisfies $\left(\sigma_t^h\right)^2 \ge  \rV_{r^h, x^{h+1}}\left[r^h+f_{\star}^{h+1}(x^{h+1})|z_t^h\right]$ for all $h\in[H]$.
\end{restatable}

We refer readers to~\Cref{app:variance} for its proof. Not only is $\sigma_t^h$ always lower bounded by the variance, we also show in~\Cref{lem:RL-variance-bounds-upper} in the appendix that it is upper bounded by the variance plus some additional terms, the sum of which can be appropriately controlled. 

Combining these two lemmas with the concentration analysis allows us to prove $\bar{f}_{t+1,1}^h\in\calF_{t+1,1}^h$ with high probability. Consequently,~\Cref{coro:good-event-fbar} follows by induction.

Conditioning on the good event $\calE_{\le T}$, and recalling the definition $V_{t}^1 = \E[\sum_{h\in[H]}r^h|x_t^1, f_{t,1}^h, f_{t,2}^h]$ under the exploration rule~\eqref{eq:greedy-policy-informal} at step $t$, we can write the expected regret as 
\begin{align}
    \E R_T & = \E\sum_{t\in[T]}\left(f^1_\star(x_t^1)-V_{t}^1\right)\nonumber\\
    & \stackrel{(i)}{\le} O(1+HT\delta)+\E\sum_{2\le t \le T}\left(f^1_{t,1}(x_t^1)-V_{t}^1\right)\label{eq:regret-rough}\\
    & \stackrel{(ii)}{\le} O\bigg(1+HT\delta+HT\epsilon\bigg)\nonumber\\
    & \hspace{2em} +O\bigg(\E\bigg[\sum_{ t\in [T]}\sum_{h\in[H]} \min\left(1+L,b_{t,1}^h(z_{t}^h)\right)~|~\calE_{\le T}\bigg]+\E\biggl[\sum_{t\in\calToo}\sum_{h\in[H]} \min\left(1+L,b_{t,2}^h(z_{t}^h)\biggr)|~\calE_{\le T}\right]\bigg). \nonumber
\end{align}
Here we use $(i)$ \Cref{lem:mono} conditioning on~$\calE_{\le T}$, and $(ii)$ by bounding $f_{t,1}-V_{t}$ conditioning on $\calE_{\le T}$. Intuitively, for $t \in \calTo$, the error between $f^1_{t,1}$ and the return of its greedy policy is primarily governed by the optimistic bias of $f_{t,1}$, which is bounded by the sum of the bonus terms. We refer the readers to~\Cref{app:approx-error} for detailed derivations on how we bound $f_{t,1}-V_{t}$ using a backward induction argument. This part of the analysis is relatively standard and resembles typical approximate dynamic programming arguments in RL, modulo some careful handling of the $\calToo$ rounds.

This expression of expected regret motivates us to bound the size of $\calToo$ and the summation of bonus terms in the next two steps.

\paragraph{Step 2. Bounding size of $|\calToo|$.} By properly choosing $u_t$, the threshold which guides the exploration policy in~\eqref{eq:greedy-policy-informal}, we obtain the following bound. 
\begin{proposition}[Bounding $|\calToo|$, informal~\Cref{lem:size-of-Too}]\label{lem:size-of-Too-informal}
Given $\alpha\le 1$, when 
\[u_t\ge \widetilde{\Omega}\Par{\frac{\sqrt{\log\calN}\cdot\log\calN\calN_b\cdot H^{5/2}\cdot\sqrt{d_\alpha}}{\sqrt{t}}},\] 
it holds that $\rE\left[|\calToo||\calE_{\le T}\right]\le \widetilde{O}\Par{T/(\log(\calN\calN_b)\cdot H^3)}$.
\end{proposition}

To see intuitively why such a lemma holds true, note when $u_t$ is large enough, $f_{t,2}^h(x_t^h)\ge f_{t,1}^h(x_t^h)+u_t\ge V_{t}^h+u_t$ cannot happen too often, given the upper bound between $f_{t,2}^h-V_{t}^h$ shown~in~\Cref{lem:RL-f-tilde-oo} in~\Cref{app:approx-error}. We refer readers to~\Cref{lem:size-of-Too} in~\Cref{app:regret} and the proof therein for more details. The particular threshold of $u_t$ we choose in~\Cref{lem:size-of-Too-informal} also serves the purpose that $\sum_{t\in[T], h\in[H]}u_t = \widetilde{O}\left(\mathrm{poly}(H,d_\alpha)\cdot\sqrt{T}\right)$, which we need to control when bounding the summation of bonus in the next step.

\paragraph{Step 3. Bounding the summation of bonus terms.} The last main term to control in the expected regret expression~\eqref{eq:regret-rough} are the summation of $b_{t,1}^h$ and $b_{t,2}^h$ following the particular trajectory that the agent explored. Below we give a tight (in terms of the leading $\sqrt{T}$-term) bound on the summation terms.

\begin{restatable}[Bound on sum of bonus terms $b_{t,1}^h$, informal~\Cref{lem:regret-bound-I}]{proposition}{lembonustighter}\label{lem:regret-bound-I-rough}
Given $\lambda = 1$, $\alpha = 1/\sqrt{TH}$ and a valid bonus oracle $\oracle$ as defined~\Cref{def:bonus-conditions}, conditioning on the event $\calE_{\le T}$, the bonus terms $b_{t,1}^h$ in~\Cref{line:bonus-one} satisfy the following:
\begin{align*}
 & \E\biggl[\sum_{t\in[T]}\sum_{h\in[H]}\min\left(1+L, b_{t,1}^h(z_t^h)\right)|\calE_{\le T}\biggr] = \widetilde{O}\biggl(\sqrt{\log\calN\cdot d_\alpha}\cdot\sqrt{HT}\biggr)\\
 & \hspace{3em} +\widetilde{O}\biggl(\sqrt{\log\calN\cdot d_\alpha}\sqrt{\log\calN\calN_b}\sqrt{H^4\E[|\calToo||\calE_{\le T}]+H^3\sum_{t\in[T]}u_t}+\log \calN\cdot \log^{1.5} \calN\calN_b  \cdot d_\alpha\cdot H^{7/2}\biggr).
\end{align*}
\end{restatable}

We provide the detailed proof in~\Cref{app:regret}.  On the high level,  we use the definition of $\beta_{t,1}^h$ and Cauchy-Schwartz inequality so that
\begin{align*}
    \sum_{t,h}\min\Par{1+L, b_{t,1}^h(z_t^h)} & = O\Par{\sum_{t,h}\beta_{t,1}^h\cdot\bsigma_t^h\cdot \frac{1}{\bsigma_t^h}D_{\calF^h}\Par{z_t^h;z_{[t-1]}^h,\bsigma_{[t-1]}^h}}\\
    & = \widetilde{O}\Par{\sqrt{\log\calN}\cdot\sqrt{\sum_{t,h}\Par{\bsigma_{t}^h}^2}\cdot\sqrt{\sum_{t,h}\frac{1}{\Par{\bsigma_t^h}^2}D_{\calF^h}^2\Par{z_t^h;z_{[t-1]}^h,\bsigma_{[t-1]}^h}}}.
\end{align*}
The leading-order terms of interest are iterations $t,h$ when $\Par{\bsigma_t^h}^{-1}\cdot D_{\calF^h}(z_t^h;z_{[t-1]}^h,\bsigma_{[t-1]}^h)\le 1$ and when $\bsigma_t^h = \sigma_t^h$ or $\widetilde{\Theta}(\sqrt{\log\Par{\calN\calN_b}})\cdot\sqrt{f_{t,2}^h(z_t^h)-f_{t,-2}^h(z_t^h)}$. When $\Par{\bsigma_t^h}^{-1}\cdot D_{\calF^h}(z_t^h;z_{[t-1]}^h,\bsigma_{[t-1]}^h)\le 1$ we can bound $\sqrt{\sum_{t,h}\Par{\bsigma_t^h}^{-2}\cdot D_{\calF^h}^2\Par{z_t^h ;z_{[t-1]}^h,\bsigma_{[t-1]}^h}}\le\sqrt{H\cdot d_\alpha}$. When $\bsigma_t^h = \sigma_t^h$ or $\bsigma_t^h =\widetilde{\Theta}(\sqrt{\log\Par{\calN\calN_b}})\cdot\sqrt{f_{t,2}^h(z_t^h)-f_{t,-2}^h(z_t^h)}$, we show that $\sum_{t,h}\Par{\bsigma_t^h}^2$ can be bounded by $O(T)$ plus low-order terms that only depend logarithmically in $T$ due to law of total variance (see~\Cref{coro:sum-variance}) and the converging property of $\Par{f_{t,2}^h(z_t^h)-f_{t,-2}^h(z_t^h)}$ (see~\Cref{lem:difference-f-bar}). This fine-grained analysis leads to the main improvement of an $\sqrt{H}$ term in the leading $\sqrt{T}$-order term.

Plugging the bounds we obtain in~\Cref{lem:size-of-Too-informal} together with the choice of $u_t$, \Cref{lem:regret-bound-I-rough} implies that 
\begin{equation}\label{eq:sum-I-informal}
\begin{aligned}
& \E\biggl[\sum_{t\in[T]}\sum_{h\in[H]}\min\left(1+L, b_{t,1}^h(z_t^h)\right)|\calE_{\le T}\biggr]= \widetilde{O}\biggl(\sqrt{\log\calN\cdot d_\alpha}\cdot\sqrt{HT}+\log \calN\cdot \log^{2} \calN\calN_b\cdot d_\alpha\cdot H^{5} \biggr).
 \end{aligned}
\end{equation}

For the second summation of bonus terms, a simpler argument directly bounding each variance $(\sigma_t^h)^2\le O(1)$ gives the following (see \Cref{app:regret} for detailed proof):

\begin{restatable}[Bound on sum of bonus terms $b_{t,2}^h$, informal~\Cref{lem:bound-E-III}]{proposition}{lembonustwotighter}\label{lem:regret-bound-III-rough}
Given $\lambda = 1$, $\alpha = 1/\sqrt{TH}$ and a valid bonus oracle $\oracle$ as defined~\Cref{def:bonus-conditions}, conditioning on the event $\calE_{\le T}$, the bonus terms  $b_{t,2}^h$ in~\Cref{line:bonus-two} satisfy the following:
\begin{align*}
 & \E\biggl[\sum_{t\in\calToo}\sum_{h\in[H]}\min\left(1+L, b_{t,2}^h(z_t^h)\right)|\calE_{\le T}\biggr] =  \widetilde{O}\biggl(\sqrt{\log(\calN\calN_b)\cdot d_\alpha}\cdot H\sqrt{ \E[|\calToo||\calE_{\le T}]}+ \sqrt{\log \calN\calN_b} \cdot d_\alpha\cdot H\biggr).
\end{align*}
\end{restatable}
Plugging the bounds we obtain in~\Cref{lem:size-of-Too-informal} under the stated choice of $u_t$, \Cref{lem:regret-bound-I-rough} implies that 
\begin{equation}\label{eq:sum-III-informal}
\begin{aligned}
& \E\biggl[\sum_{t\in[T]}\sum_{h\in[H]}\min\left(1+L, b_{t,2}^h(z_t^h)\right)|\calE_{\le T}\biggr]= \widetilde{O}\biggl(\sqrt{d_\alpha}\cdot\sqrt{T}+\sqrt{\log \calN\calN_b}\cdot d_\alpha\cdot H \biggr).
 \end{aligned}
\end{equation}

Plugging~\eqref{eq:sum-I-informal} and~\eqref{eq:sum-III-informal} back into the expected regret bound in~\eqref{eq:regret-rough}, we conclude the proof sketch of the expected regret as stated in~\Cref{thm:regret-genera.}.

\section*{Acknowledgments}
YJ was supported in part by the Danzig-Lieberman Graduate Fellowship. Part of the work was done while YJ was a research intern at Google Research, NY. 

\newpage
\bibliographystyle{plainnat}

\newpage
\appendix
\section{A Technical Issue in~\citet{hu2022nearly}}\label{app:comp}

\citet{hu2022nearly} designed an algorithm for linear MDPs and the main result in their paper states the method achieves a minimax-optimal regret of $\widetilde{O}\Par{d\sqrt{HT}}$. Unfortunately, their analysis suffers from a technical mistake which we explain in detail here. 

Their algorithm crucially relies on the assumption the over-optimistic values $\dot{\widehat{V}}_{i,h}(\cdot)$  upper bound the optimistic values $\widehat{V}_{i,h}(\cdot)$ point-wise with high probability. Specifically, Lemma D.2 and Equation (38) state that $\dot{\widehat{V}}_{i,h}(s)\ge \widehat{V}_{j,h}(s)~, \forall i \leq j$ and $s\in\calS$, where we use their notations with $s$ denoting state instead of $x$ as in this paper.
This is a critical condition needed in the proof. 

However, the proof of Lemma D.2 in~\cite{hu2022nearly} is incorrect. In particular, the authors used induction to prove Lemma D.2: Assuming the condition holds for $h+1$, they argue
\[
\dot{\widehat{Q}}_{i,h}(s,a)-\widehat{Q}_{j,h}(s,a) = \mathbb{P}_h\widehat{V}_{i,h+1}(s,a)-\mathbb{P}_h\widehat{V}_{j,h+1}(s,a)\ge 0.
\]
Yet the authors did not explicitly prove the last inequality: the proof states that the last inequality is directly from the induction hypothesis: $\dot{\widehat{V}}_{i,h+1}(\cdot)\ge \widehat{V}_{j,h+1}(\cdot)$, which is incorrect. 
In fact the inductive argument fails with their algorithm, which cannot be resolved by modifying the analysis. In this paper, we modify the overly optimistic value estimation procedure using unweighted regression (together with a different exploration policy) to ensure that the estimation made at any time $i$ remains valid for all time $j \geq i$. 
\section{Proofs for~Linear Function Approximation}\label{app:algo}

Here we provide the full proofs of several properties of linear function class as stated in the main paper.

The first property is about the Eluder dimension for linear MDPs~\cite{russo2013eluder}.

\remarkeluderlinear*

\begin{proof}[Proof of~\Cref{remark:eluder-linear}]
Fix $h\in[H]$. Recalling the definitions of linear function classes $\calF_\lin^h$, we can simplify the definition of generalized Eluder dimension to be the follows: 
\begin{align*}
    \Par{\bsigma^h}^{-2}D_{\calF^h_\lin}^2(z^h; z_{[t-1]}^h,\bsigma^h_{[t-1]}) & \le \Par{\bsigma^h}^{-2}D_{\calF^h_\lin}^2(z^h; z_{[t-1]}^h,\bsigma^h_{[t-1]}) = \norm{\frac{1}{\bsigma^h}\phi^h(z^h)}_{\Par{\Sigma_t^h}^{-1}}^2,\\
    \text{where}~~~\Sigma_t^h & := \frac{\lambda}{\Par{B^h}^2}I+\sum_{s\in[t-1]}\frac{1}{\Par{\bsigma_s^h}^2}\phi^h(z_s^h)\Par{\phi^h(z_s^h)}^\top.
\end{align*}
Consequently, in this case we can bound for any $\bm{\sigma}\ge\alpha$ that 
\begin{align*}\dim(\calF_\lin^h,\calZ,\sigma) & \stackrel{(i)}{\le} \sum_{t\in[T]}\frac{2\norm{\frac{1}{\bsigma^h}\phi^h(z_t^h)}^2_{\Par{\Sigma_t^h}^{-1}}}{1+\norm{\frac{1}{\bsigma^h}\phi^h(z_t^h)}^2_{\Par{\Sigma_t^h}^{-1}}}\stackrel{(ii)}{=}\sum_{t\in[T]} 2\norm{\frac{1}{\bsigma^h}\phi^h(z_t^h)}^2_{\Par{\Sigma_{t+1}^h}^{-1}}
\\
& \stackrel{(iii)}{=}2\sum_{t\in[T]}\Par{\log\det(\Sigma_{t+1}^h)-\log\det(\Sigma_{t}^h)}= O\Par{\log\left|\frac{\Par{B^h}^2}{\lambda}\Sigma_T^h\right|} = O\Par{d\log\Par{1+\frac{\Par{B^h}^2T}{\alpha^2 d\lambda}}}.
\end{align*}
Here we use $(i)$ the inequality that $\min(1,x)\le \frac{2x}{1+x}$ for any $x\ge0$,  $(ii)$ the Sherman-Morrison formula, and $(iii)$ writing $\norm{\frac{1}{\bsigma^h}\phi^h(z_t^h)}^2_{\Par{\Sigma_{t+1}^h}^{-1}} = \mathrm{trace}\Par{\Par{\Sigma_{t+1}^h}^{-1}(\Sigma_{t+1}^h-\Sigma_{t}^h)}=\log\det(\Sigma_{t+1}^h)-\log\det(\Sigma_{t}^h)$. This is the classical bound of summation of Elliptical bonuses, see e.g.\ Lemma 11 of~\cite{abbasi2011improved}. The above inequality shows that $\dim_{\alpha, T}(\calF_\lin^h) =  O\Par{d\log\Par{1+\frac{\Par{B^h}^2T}{\alpha^2 d\lambda}}} = \widetilde{O}(d)$.

Since by definition any $\epscov$-cover of $\calF_\lin^h(\epscov)$ is just a subset of $\calF_\lin^h$, we have $\dim_{\alpha,T}(\calF_\lin^h(\epscov)) = \widetilde{O}(d)$ by definition of generalized Eluder dimension~(\Cref{def:general-eluder-RL}).
\end{proof}

The next lemma shows that standard elliptical bonus functions satisfy the definition of bonus oracle $\oracle$.

\lemlinearbonus*

\begin{proof}[Proof of~\Cref{lem:linear-bonus}]
By definition of the class $\calF^h_\lin(\epscov)$, we note that for any $z^h$:
\begin{align*}
    |f^h(z^h)-\hat{f}_{t}^h(z^h)| \leq& \|w - \hat{w}_t\|_{\Sigma_t^h}\|\phi^h(z^h)\|_{(\Sigma_t^h)^{-1}},
\end{align*}
where $w$ and $\hat{w}_t$ are the weight parameters underlying $f^h$ and $\hat{f}_t^h$ respectively. By the definition~\eqref{def:confidence-crude-H} of the class $\calF_t^h$, we have that for any $f \in \calF_t^h$, defining the bonus as $\|\phi^h(x,a)\|_{(\Sigma_t^h)^{-1}}\sqrt{(\beta^h_t)^2 + \lambda}$ verifies the second property of Definition~\ref{def:bonus-conditions}. 

For the third property, we note there must exist $\Delta w_\star$ satisfying $\|\Delta w_\star\| = 2B^h$ and that 
\begin{align*}
\|\phi^h(z^h)\|^2_{\Par{\Sigma_t^h}^{-1}} = \frac{\Delta w_\star^\top \phi^h(z^h) \phi^h(z^h)^\top \Delta w_\star}{\sum_{s\in[t-1]}\Delta w_\star^\top \phi^h(z_s^h) \phi^h(z_s^h)^\top \Delta w_\star+\lambda}.
\end{align*}
Thus, by assumption of $\calF_\lin^h(\epscov)$ being an $\epscov$-cover we can find $\tilde{\Delta}w_\star = w-w'$ for $w,w'\in\calF_\lin^h(\epscov)$ such that $\langle \phi^h(z^h), \tilde{\Delta}w_\star-\Delta w_\star\rangle\le 2\epscov$ for all $z^h$ and thus 
\begin{align*}
\frac{\Delta w_\star^\top \phi^h(z^h) \phi^h(z^h)^\top \Delta w_\star}{\sum_{s\in[t-1]}\Delta w_\star^\top \phi^h(z_s^h) \phi^h(z_s^h)^\top \Delta w_\star+\lambda} & \le \frac{2\tilde{\Delta} w_\star^\top \phi^h(z^h) \phi^h(z^h)^\top \tilde{\Delta} w_\star+2\cdot(2\epscov)^2}{\frac{1}{2}\sum_{s\in[t-1]}\tilde{\Delta} w_\star^\top \phi^h(z_s^h) \phi^h(z_s^h)^\top \tilde{\Delta} w_\star-(2\epscov)^2(t-1)+\lambda}\\
& \le 4\frac{\tilde{\Delta} w_\star^\top \phi^h(z^h) \phi^h(z^h)^\top \tilde{\Delta} w_\star}{\sum_{s\in[t-1]}\tilde{\Delta} w_\star^\top \phi^h(z_s^h) \phi^h(z_s^h)^\top \tilde{\Delta} w_\star+\lambda}+16\epscov^2/\lambda\\
& \le 4\sup_{w,w'\in\calF_\lin^h(\epscov)}\frac{(w-w')^\top \phi^h(z^h) \phi^h(z^h)^\top (w-w')}{\sum_{s\in[t-1]}(w-w')^\top \phi^h(z_s^h) \phi^h(z_s^h)^\top(w-w')+\lambda}+16\epscov^2/\lambda,
\end{align*}
where for the last inequality we use the choice of $\epscov$ so that $\epscov^2\le \frac{\lambda}{8T}$. Thus by taking square root on both sides, we have
\[
\|\phi^h(z^h)\|_{\Par{\Sigma_t^h}^{-1}}\le 2\bsigma_t^h\cdot D_{\calF_\lin^h(\epscov)}(z^h, \bsigma_t^h; z_{[t-1]},\bsigma_{[t-1]}^h)+4\epscov/\sqrt{\lambda},
\]
which by multiplying over $\sqrt{\Par{\beta_t^h}^2+\lambda}$ on both sides proves the third property.

Since the matrix $\Sigma_t^h$ is data dependent, the standard way to specify the bonus class is to parameterize it by all possible choices of the matrix in the Mahalanobis norm~\citep{jin2020provably} , which means that the class $\calW$ consists of all bonus functions of the form 
\begin{align*}
    b^h(z^h) \in\left\{ \norm{\phi^h(z^h)}_{A}|\text{where}~A\in \mathcal{C}_A\right\},~~\text{for any}~h\in[H],\\
    \text{where}~~\mathcal{C}_A~\text{is an}~\epscov^2\text{-cover of } \left\{A\in\R^{d\times d},~\|A\|_\frob\le \frac{\sqrt{d}}{\lambda}\cdot\Par{ \beta^2+\lambda}\right\}.
\end{align*}

By standard argument we have the size of $\log |\mathcal{C}_A|$ is bounded by $O\Par{d^2\log\Par{1+\frac{\sqrt{d}\cdot\beta^2}{\lambda\epscov^2}}}$.  Thus, by definition of $\calW$ we also have $\log |\calW| = \widetilde{O}\Par{d^2\log\Par{1+\frac{\sqrt{d}\cdot\beta^2}{\lambda\epscov^2}}}$.

The consistency of the oracle follows from the fact that  $\beta_t^h$ is non-decreasing in $t$ element-wise for each $h\in[H]$, thus completing the proof.
\end{proof} %
\section{Implementing Bonus Oracle $\oracle$ using Online-subsampling}\label{app:bonus-condition}

The guarantees of~\Cref{alg:fitted-Q-simpler} hold assuming a consistent bonus oracle $\oracle$ satisfying~\Cref{def:bonus-conditions}. To implement such an oracle, we follow the online sensitivity sub-sampling approach described in~\cite{kong2021online}, which is a follow-up of the original sensitivity sub-sampling idea proposed in~\citet{wang2020reinforcement}. 

For completeness here we restate this sub-sampling procedure in~\Cref{alg:online-sensitivity-weighted} and its guarantees in~\Cref{prop:online-sensitivity-weighted}.~\footnote{Different from the original result, we don't consider a cover class since we are already working with a finite function class.}

We first define the weighted dataset $\calZ$ so that each element in it is $(z,\bsigma(z))$ and define the \emph{weighted sensitivity score} as 
\begin{align*}
    \mathrm{sensitivity}_{\calZ,\calF,\beta,\alpha}(z) = \min\left\{\sup_{f,f'\in\calF}\frac{\frac{1}{\bsigma^2(z)}\left(f(z)-f'(z)\right)^2}{\min\left\{\sum_{z'\in\calZ}\frac{1}{\bsigma^2(z')}\left(f(z')-f'(z')\right)^2,\frac{T(H+1)^2}{\alpha^2}\right\}+\beta^2},1\right\}
\end{align*}

Now consider the weighted dataset $\calZ_{[t-1]}^h = \{(x_s^h, a_s^h), \bsigma_s^h\}_{s\in[t-1]}$, we define $\|f\|_{\calZ}^2 = \sum_{z\in\calZ}\frac{1}{\bsigma^2(z)}f^2(z)$, i.e. weighted sum of $\ell_2$-norm square.  Now we introduce the sub-sampling procedure.

\begin{algorithm}[h!]
\caption{Online Sensitivity Sub-sampling with Weights}\label{alg:online-sensitivity-weighted}
\DontPrintSemicolon
\textbf{Input} function class $\calF$, current sub-sampled dataset $\hat{\calZ}\subseteq\calS\times\calA$, new state-action pair $z$, parameter $\beta$, threshold $\alpha>0$, failure probability $\delta$\;
\textbf{Parameter} $1\le C<\infty$\;
Let $p_z$ be the smallest real number such that 
\[1/p_z~\text{is an integer and }p_z\ge\min\left(1,C\cdot \mathrm{sensitivity}_{\calZ,\calF,\beta,\alpha}(z) \cdot\log(T\calN/\delta) \right)\]\;
Independently add $1/p_z$ copies of $(z,\bsigma(z))$ into $\hat{\calZ}_{+}$ with probability $p_z$\;
\textbf{Return} $\hat{\calZ}_{+}$.
\end{algorithm}

This algorithm can be called at every step $t,h$ with $\hat{\calZ}_{[t-1]}^h$ and the new data $z_t^h$ to obtain the next $\hat{\calZ}_{[t]}^h$. Below we will refer to the original dataset as $\calZ_{[t-1]}^h = \{(x_s^h,a_s^h),\bsigma_s^h\}_{s\in[t-1]}$, and $\hat{\calZ}_{[t-1]}^h$ as the dataset subsampled from $\calZ_{[t-1]}^h$.

\begin{proposition}[Guarantees in weighted case, generalizing Proposition 1 and 2 in~\cite{kong2021online}]\label{prop:online-sensitivity-weighted}
When $\bsigma(z)\ge \alpha$ always holds for any $z$, with probability $1-\delta$, it holds that
\begin{align*}
    \sup_{f_1, f_2: \|f_1-f_2\|_{\calZ_t^h}^2\le \beta^2}|f_1(z)-f_2(z)|\le  \sup_{f_1, f_2: \|f_1-f_2\|_{\hat{\calZ}_t^h}^2\le 10^2\beta^2}|f_1(z)-f_2(z)|\le  \sup_{f_1, f_2: \|f_1-f_2\|_{\calZ_t^h}^2\le 10^4\beta^2}|f_1(z)-f_2(z)|.
\end{align*}
Further, for each $h\in[H]$, the number of different elements in sub-sampled dataset $\hat{\calZ}_t^h$ $(t=1,2,\cdots, T)$ is always bounded by $S_{\max} = O\Par{\log \frac{T\calN}{\delta}\cdot\max_{h\in[H]}\dim_{\alpha, T}(\calF^h)}$ and the total size (counting repetitions) is bounded by $O(T^3/\delta)$.
\end{proposition}

\onlinesubsample*

We remark that here we state a slightly generalized version adapted to weighted regression, which includes unweighted regression stated in~\cite{kong2021online} as a special case when we set $\bsigma(z)\equiv1$. The result follows a straightforward generalization from~\cite{kong2021online} taking weights into consideration. %
\section{Full Analysis of~\Cref{thm:regret-genera.}}\label{app:proofs}

This section provides a full proof for the bound on the expected regret in~\Cref{alg:fitted-Q-simpler}. This is stated in~\Cref{thm:regret-genera.} in the main paper and here we first state a more formal version of that theorem.

\begin{restatable}[Bound on expected regret]{theorem}{theoremregret}\label{thm:regret}
Suppose function class $\{\calF^h\}_{h\in[H]}$ satisfies~\Cref{ass:eps-realizability-RL} with $\epsilon\in[0,1]$ and~\Cref{def:general-eluder-RL} with $\lambda=1$, and given consistent bonus oracle $\oracle$ (output function in class $\calW$) satisfying~\Cref{def:bonus-conditions}, \alg with $\alpha= \sqrt{1/TH}$, $\delta<1/(T+10)$, $\eps\le 1$ and 
\[u_t=  C\cdot\Par{\frac{\sqrt{\log\frac{\calN TH}{\alpha\delta}+\frac{T}{\alpha^2}\epsilon}\cdot\Par{\log\frac{\calN\calN_bTH}{\alpha\delta}\cdot H^{5/2}\sqrt{d_\alpha}+\sqrt{t}H\epsc}}{\sqrt{t}}+H^2\epsilon+H\delta}
\]
for sufficiently large constant $C<\infty$, achieves a total regret of
\begin{align*}
\E R_T & = O\left(\sqrt{\log\frac{\calN TH}{\delta}+T^2H\epsilon}\cdot\sqrt{TH d_\alpha} +\Par{\log\frac{\calN TH}{\delta}+T^2H\epsilon}\cdot \Par{\log^2 \frac{\calN\calN_bTH}{\delta}\cdot H^{5}d_\alpha+T^2\epsc^2}\right).
\end{align*}
\end{restatable}

The section is organized as follows: We first introduce some general notations, definitions, and helper lemmas that will be used throughout the proof in this section in~\Cref{app:notation}. In~\Cref{app:CI}, we prove the properties of constructed confidence intervals $\calF_{t,j}^h$, $j=1,\pm2$ and $\calG_t^h$. In~\Cref{app:variance} we prove the key point-wise monotonicity property and also some properties of our chosen variance estimator $\Par{\sigma_t^h}^2$ defined in~\Cref{eq:sigma-def-alg}. In~\Cref{app:approx-error}, we show the approximation error between our constructed values $f_{t,j}^h$, $j=1,\pm2$ with respect to the true expected reward $V_t^h$. In~\Cref{app:regret} we provide the formal proofs for bounding the regret. We refer readers to~\Cref{app:regret-hp} for the complete theorem statement and full proof for the \emph{high-probability} regret bound.

In~\Cref{tbl:notation} we summarize the main notations used in the paper. We also provide the concrete choices of parameters of~\Cref{alg:fitted-Q-simpler} for obtaining the claimed regret bounds in~\Cref{tbl:parameters}.

\subsection{Notations and Preliminaries}\label{app:notation}

Here we briefly give self-contained notations and definitions used throughout the proof. We summarize the main notations used  in~\Cref{tbl:notation} and the specific choice of parameters in~\Cref{tbl:parameters} for easy reference.

\begin{table}[!]
    \centering
    \renewcommand{\arraystretch}{1}
    \begin{tabular}{{c}{c}{c}}
    \toprule
Notation & Meaning &  Remark\\
\midrule
$\calS$, $\calA$ & state space, action space & \\
$t, h$ & $t\in[T]$ trajectory/step, $h\in[H+1]$ level & \\
$r_t^h,x_t^h, a_t^h$ & reward, state and action at step $t$ and level $h$ & \\
$r^h,x^h, a^h$ & random reward, state and action at step $t$ and level $h$ & \\
$z$ & shorthand for state-action pair $(x,a)$ & \\
$\calD_{[t-1]}^h$ & data set $\{(x_s^h,a_s^h,r_s^h, x_s^{h+1})\}_{s\in[t-1]}$ & \\
$\calF^h$ & function class for $h\in[H]$  & Ass.~\ref{ass:eps-realizability-RL}\\
$\calF_\lin^h$ & general linear function class & Eqn.~\eqref{def:linear-F-h}\\
$\calF_\lin^h(\epscov)$ & $\epscov$-cover of general linear function class $\calF_\lin^h$ & Eqn.~\eqref{def:linear-F-h}\\
$\calW$ & bonus function class defined for $\oracle$ & Def.~\ref{def:bonus-conditions}\\
$\epsc$ & error paremeter for bonus oracle &  Eqn.~\eqref{def:linear-F-h}\\
$\calN$ & maximal size of function class $\max_{h\in[H]}|\calF^h|$\\
$\calN_b$ & size of bonus function class $|\calW|$ & Def.~\ref{def:bonus-conditions}\\
$D_{\calF}(z; z_{[t-1]}, \sigma_{[t-1]})$ & $\defeq\sqrt{\sup_{f_1,f_2\in\calF}\frac{\left(f_1(z)-f_2(z)\right)^2}{\sum_{s\in[t-1]}\frac{1}{\sigma_s^2}\left(f_1(z_{s})-f_2(z_s)\right)^2+\lambda}}$ & $\lambda$ param. \\
$\dim_{\alpha,T}(\calF)$ & generalized Eluder dimension defined in~\Cref{def:general-eluder-RL} & $\alpha$ param.\\
$d_\alpha$ & shorthand for $\frac{1}{H}\sum_{h\in[H]}\dim_{\alpha,T}(\calF^h)$ (\Cref{def:general-eluder-RL}) & $\alpha$ param.\\
$f_{t,1}^h$ & optimistic value function at step $t,h$  &\\
$\hat{f}_{t,1}^h$ & solution of fitting weighted regression at step $t,h$  & Eqn.~\eqref{eq:C-update-rule-opti}\\
$\bar{f}_{t,1}^h$ & in $\calF^h$ and $\max_{z^h}\left|\bar{f}_{t,1}^h(z^h) - \E\left[r^h+f_{t,1}^{h+1}(x^{h+1})|z^h\right]\right|\le \epsilon$ & Ass.~\ref{ass:eps-realizability-RL}\\
$\calF_{t,1}^h$ & version space of optimistic value functions at step $t,h$ & Eqn.~\eqref{eq:C-update-rule-opti}\\
$f_{t,2}^h$ & overly optimistic  value function at step $t,h$ &\\
$f_{t,-2}^h$ & overly pessimistic  value function at step $t,h$ &\\
$\hat{f}_{t,\pm2}^h$ & solution of fitting unweighted regression at step $t,h$  &Eqn.~\eqref{eq:C-update-rule-overly}\\
$\bar{f}_{t,\pm2}^h$ & in $\calF^h$ and  $\max_{z^h}\left|\bar{f}_{t,j}^h(z^h) - \E\left[r^h+f_{t,j}^{h+1}(x^{h+1})|z^h\right]\right|\le \epsilon$ & Ass.~\ref{ass:eps-realizability-RL}\\
$\calF_{t,\pm2}^h$ & version space of overly optimistic(pessimistic) value functions at $t,h$ & Eqn.~\eqref{eq:C-update-rule-overly}\\
$\hat{g}_t^h$ & solution of fitting second-moment regression at step $t,h$  & Eqn.~\eqref{eq:C-update-rule-RL-second}\\
$\psi_{t}^h$ & in $\calF^h$ and $\max_{z^h}\left|\psi_{t}^h(z^h) - \E\left[\Par{r^h+f_{t,2}^{h+1}(x^{h+1})}^2|z^h\right]\right|\le \epsilon$ & Ass.~\ref{ass:eps-realizability-RL}\\
$\calG_{t}^h$ & version space of second-moment estimates at $t,h$ & Eqn.~\eqref{eq:C-update-rule-RL-second}\\
$\calE_{t}^h$ & event that $\{\bar{f}_{t,j}^h\in\calF_{t,j}^h~\text{for}~j=1,\pm2~\text{and}~\psi_t^h\in\calG_t^h\}$ & \\
$\calE_{t}, \calE_{\le t}$ & joint event that $\cap_{h\in[H]}\calE_{t}^h$ or $\cap_{s\in[t]}\cap_{h\in[H]}\calE_{s}^h$ & \\
$h_t\in[H+1]$ & random $h$ when starting to take greedy w.r.t $f_{t,2}^h$ at step $t$ & \\
$\calTo,\calToo$ & disjoint subsets of $[T]$ when $h_t= H+1$ or $h_t\in[H] $ & \\
$V_t^h$ &  expected reward during exploration at time $t$ from step $h$ onwards &~\\
$V_\star^h$, $Q_\star^h$  & optimal $V$-value or $Q$-value function at level $h$ & \\
$f_\star^h$ & equivalent to $Q_\star^h$ or $V_\star^h$ depending on the context \\
$\xi_{t,j}^h$ & $\defeq r^{h}_t+f_{t,j}^{h+1}(x^{h+1}_{t})-\E_{r^h,x^{h+1}}\left[r^h+f_{t,j}^{h+1}(x^{h+1})|z_t^h\right]~\text{for}~j=1,\pm2$ & \\
$\xi_{t}^h$ & $\defeq r^{h}_t+V_{t}^{h+1}-\E\left[r^h+V_{t}^{h+1}|z_t^{[h]}, f_{t,1}^{[H]}, f_{t,2}^{[H]}\right]$ & \\
$b_{t,j}^h$ &  bonus term obtained in~\Cref{line:bonus-one,line:bonus-two} using $\oracle$, for $j=1,2$ & Def.~\ref{def:bonus-conditions}\\
$\calT$ & Bellman operator $\calT f (z^h) = \E[r^h+f(x^{h+1})|z^h]$
& \\
$\calT_2$ & second-moment operator $\calT_2 f (z^h) = \E\left[\left(r^h+f(x^{h+1})\right)^2|z^h\right]$
& \\
\bottomrule
\end{tabular}
    \caption{\textbf{Summary of notations.} Here we use $\E[\cdot|z^h] = \E_{r^h, x^{h+1}}[\cdot|z^h]$ for brevity.}
    \label{tbl:notation}
\end{table}

\begin{table}[!t]
    \centering
    \renewcommand{\arraystretch}{1.75}
    \begin{tabular}{{c}{c}{c}}
    \toprule
Parameter & Choice &  Remark\\
\midrule
$\delta$ & $\delta<1/(T+10)$ in~\Cref{thm:regret}, $\delta<1/(H^2+11)$ in~\Cref{thm:regret-hp} & \\
$\delta_{t,h}$ & $\delta/(T+1)(H+1)$ & Eqn.~\eqref{eq:def-delta-and-upsilon-opti}\\
$\epsilon$ & $\epsilon\in[0,1]$, model class misspecification error & Ass.~\ref{ass:eps-realizability-RL}\\
$\epscov$ & error due to taking covering of function class & \\
$\alpha$ & $\sqrt{1/TH}$ & Def.~\ref{def:general-eluder-RL}\\
$\lambda$ & 1 & Def.~\ref{def:general-eluder-RL}\\
$\upsilon(\delta)$ & $\sqrt{\log\frac{\calN^2\left(2\log(4LT/\alpha)+2\right)\left(\log(8L/\alpha^2)+2\right)}{\delta}}$ & Eqn.~\eqref{eq:def-delta-and-upsilon-opti} \\
$\iota(\delta)$ & $3\sqrt{\log\frac{\calN\calN_b\left(2\log(4LT/\alpha)+2\right)\left(\log(8L/\alpha^2)+2\right)}{\delta}}$ & Eqn.~\eqref{eq:def-iota-opti} \\
$\beta_{t,1}^h$ & $\sqrt{\Par{6\sqrt{\lambda}+156}}\cdot\sqrt{\log \frac{ \calN^2 (T+1)(H+1)\left(2\log\frac{4LT}{\alpha}+2\right)\left(\log\frac{8L}{\alpha^2}+2\right)}{\delta}}+\sqrt{\frac{8tL}{\alpha^2}\cdot\epsilon}$ & Eqn.~\eqref{eq:def-beta-opti} \\
$\dot{\iota}(\delta)$ & $\sqrt{2\log\frac{\calN\calN_b\left(2\log(18LT)+2\right)\left(\log(18L)+2\right)}{\delta}}$ & Eqn.~\eqref{eq:def-delta-and-iota-overly}\\
$\beta_{t,2}^h$ & $\sqrt{2\left(24L+21\right)\dot{\iota}^2(\delta_{t,h})+20tL\epsilon}$ & Eqn.~\eqref{eq:def-beta-overly}  \\
$\iota'(\delta)$ & $\sqrt{2\log\frac{\calN\calN_b\left(2\log(32LT)+2\right)\left(\log(32L)+2\right)}{\delta}}$ & Eqn.~\eqref{eq:def-delta-and-iota-second}\\
$\bbeta_t^h$ &  $\sqrt{8(11+9L)\left(\iota'(\delta_{t,h})\right)^2+32tL\epsilon}$ & Eqn.~\eqref{eq:def-beta-second}\\
\multirow{2}{*}{$\left(\sigma^h_t\right)^2$} & $\min\left(4, D_{\calF^h}(z_t^h; z_{[t-1]}^h, \1_{[t-1]}^h)\cdot\left(\sqrt{\left(\bbeta_t^h\right)^2+\lambda}+2L\sqrt{\left(\beta_{t,2}^h\right)^2+\lambda}\right)\right.$ & \multirow{2}{*}{Eqn.~\eqref{eq:sigma-def-alg}} \\
& $\left.+\hat{g}_t^h(z_t^h)-\left(\hat{f}_{t,-2}^h(z_t^{h})\right)^2+2(1+L)\epsilon\right)$ for $t\ge2$ &\\
\multirow{2}{*}{$\bsigma^h_t$} & $\max\left\{\sigma_t^h,\alpha,\sqrt{2}\iota(\delta_{t,h})\sqrt{f_{t,2}^h(z_t^h)-f_{t,-2}^h(z_t^h)},\right.$ & \multirow{2}{*}{Eqn.~\eqref{eq:def-bsigma}}\\
& $\left.\quad\quad\quad\quad\quad\quad 2\Par{\sqrt{\upsilon(\delta_{t,h})}+\iota(\delta_{t,h})}\cdot\sqrt{ D_{\calF^h}(z^{h}_t; z^{h}_{[t-1]},\bsigma^{h}_{[t-1]})}\right\}$ & \\
\bottomrule
\end{tabular}
    \caption{\textbf{Summary of parameter choices.}}
    \label{tbl:parameters}
\end{table}

\paragraph{Iterates and functions.} In general, we use $z = (x,a)$, $z^h = (x^h, a^h)$ and $z_t^h = (z_t^h, a_t^h)$ interchangeably. We also use $f_\star^h\in\calF^h$ to denote either the optimal $Q$-value function $Q_\star^h$ or the optimal $V$-value function $V_\star^h$ when clear from context.

As in the main paper, we will use the notation $\calT f (z^h) = \E_{r^h,x^{h+1}}[r^h+f(x^{h+1})|z^h]$ and also $\calT_2 f (z^h) = \E_{r^h,x^{h+1}}\left[\left(r^h+f(x^{h+1})\right)^2|z^h\right]$  to be the conditional expectation of future values and their second moment under any function $f$ at level $h$ and state-action pair $z^h$. 

We recall the definition of $f_{t,j}$ for $j=1,\pm2$ in~\Cref{alg:fitted-Q-simpler} that 
\begin{equation}\label{eq:def-f-t-j}
\begin{aligned}
    f_{t,1}^h(\cdot) & \defeq \min\left(\hat{f}_{t,1}^h(\cdot)+b_{t,1}^h(\cdot)+\epsilon,1\right),\\
    f_{t,2}^h(\cdot) & \defeq \min\left(\hat{f}_{t,2}^h(\cdot)+2b_{t,1}^h(\cdot)+b_{t,2}^h(\cdot)+3\epsilon,2\right),\\
    f_{t,-2}^h(\cdot) & \defeq \max\left(\hat{f}_{t,-2}^h(\cdot)-b_{t,2}^h(\cdot)-\epsilon,0\right),
\end{aligned}
\end{equation}
where $\hat{f}_{t,j}^h$ is the center of the constructed confidence interval (see next paragraph for concrete definitions) respectively. For each $f_{t,j}^h$ with $j = 1,\pm2$, we will let $\bar{f}_{t,j}^h(\cdot)\in\calF^h$ to approximate the conditional expectation with target $f_{t,j}^{h+1}$, i.e. $\max_{z^h}\left|\bar{f}_{t,j}^h(z^h) -\calT f_{t,j}^{h+1}(z^h)\right|\le \epsilon$ (such $\bar{f}$ exists due to~\Cref{ass:eps-realizability-RL}). Similarly, we will let $\psi_t
^h(\cdot)\in\calF^h$ to approximate the conditional expectation of second moment with target $f_{t,2}^{h+1}$, i.e. $\max_{z^h}\left|\psi_t^h(z^h)-\calT_2 f_{t,1}^{h+1}(z^h)\right|\le \epsilon$ (existence is guaranteed similarly).

\paragraph{Regression and confidence intervals.} We define the following (weighted) regression problems and their induced confidence intervals at each step $t\in[T]$, $h\in[H]$ when   the dataset $\calD_{[t-1]}^h\defeq \{(x_s^h, a_s^h, r_s^h, x_s^{h+1})\}_{s\in[t-1]}$ is given. Throughout we use $\calN \defeq \max_{h\in[H]}|\calF^h|$ and $\calN_b \defeq |\calW|$ to denote the sizes of the function class $\calF^h, h\in[H]$  and the bonus function class $\calW$.

The weighted regression problem for fitting optimistic value functions $f_{t,1}^h$ is:
\begin{equation}\label{eq:C-update-rule-opti}
\begin{aligned}
\hat{f}_{t,1}^h & = \arg\min_{f^h\in\calF^h} \sum_{s\in[t-1]}\frac{\left(f^h\left(x^{h}_s,a^{h}_s\right)-r^{h}_s-f_{t,1}^{h+1}\left(x_s^{h+1}\right)\right)^2}{\left(\bsigma^{h}_s\right)^2},\\
\text{and let}~\calF^h_{t,1} & \defeq \left\{f^h\in\calF^h:\sum_{s\in[t-1]}\frac{1}{\left(\bsigma_s^{h}\right)^2}\left(f^h(x_s^{h}, a_s^{h})-\hat{f}^h_t(x_{s}^{h},a_s^{h})\right)^2\le \left(\beta_{t,1}^h\right)^2\right\}.
\end{aligned}
\end{equation}
The parameters are as follows (for $t\ge 2$): 
\begin{align}
\bsigma^h_t & \defeq \max\left\{\sigma_t^h,\alpha,\sqrt{2}\iota(\delta_{t,h})\sqrt{f_{t,2}^h(z_t^h)-f_{t,-2}^h(z_t^h)},\right.\label{eq:def-bsigma}\\ & \hspace{12em} \left.2\Par{\sqrt{\upsilon(\delta_{t,h})}+\iota(\delta_{t,h})}\cdot\sqrt{ D_{\calF^h}(z^{h}_t; z^{h}_{[t-1]},\bsigma^{h}_{[t-1]})}\right\},\nonumber\\
    \beta_{t,1}^h & \defeq \sqrt{\Par{6\sqrt{\lambda}+156}}\cdot\sqrt{\log \frac{ \calN^2 (T+1)(H+1)\left(2\log\frac{4LT}{\alpha}+2\right)\left(\log\frac{8L}{\alpha^2}+2\right)}{\delta}}+\sqrt{\frac{8tL}{\alpha^2}\cdot\epsilon},\label{eq:def-beta-opti}\\
    \delta_{t,h} & \defeq \frac{\delta}{(T+1)(H+1)}~~~\text{,}~~~\upsilon(\delta_{t,h}) \defeq \sqrt{\log\frac{\calN^2\left(2\log(4LT/\alpha)+2\right)\left(\log(8L/\alpha^2)+2\right)}{\delta_{t,h}}},\label{eq:def-delta-and-upsilon-opti}\\
    & \hspace{7em}~~~\text{and}~~~\iota(\delta_{t,h}) \defeq 3\sqrt{\log\frac{\calN\calN_b\left(2\log(4LT/\alpha)+2\right)\left(\log(8L/\alpha^2)+2\right)}{\delta_{t,h}}}.\label{eq:def-iota-opti}
\end{align}

The unweighted regression for fitting overly optimistic and overly pessimistic value functions is as follows:
\begin{equation}\label{eq:C-update-rule-overly}
\begin{aligned}
\forall~t\in[T], h\in[H],~~~\hat{f}_{t,\pm2}^h & = \arg\min_{f^h\in\calF^h} \sum_{s\in[t-1]}\left(f^h\left(x^{h}_s,a^{h}_s\right)-r^{h}_s-f_{t,\pm2}^{h+1}\left(x_s^{h+1}\right)\right)^2,\\
\text{and we let}~\calF^h_{t,\pm2} & \defeq \left\{f^h\in\calF^h:\sum_{s\in[t-1]}\left(f^h(x_s^{h}, a_s^{h})-\hat{f}^h_{t,\pm2}(x_{s}^{h},a_s^{h})\right)^2\le \left(\beta_{t,2}^h\right)^2\right\}.\\
\end{aligned}
\end{equation}
We choose the parameters as follows:
\begin{align}
\text{Note}~~&~\max_{h\in[H]}\left|\calF^h+2\calW+\calW\right| \le \calN\calN_b^2,\nonumber\\
& \beta_{t,2}^h \defeq \sqrt{2\left(24L+21\right)\dot{\iota}^2(\delta_{t,h})+20tL\epsilon},\label{eq:def-beta-overly}\\
& \delta_{t,h} \defeq \frac{\delta}{ (T+1)(H+1)}~\text{and}~\dot{\iota}(\delta) = \sqrt{2\log\frac{\calN\calN_b\left(2\log(18LT)+2\right)\left(\log(18L)+2\right)}{\delta}}.\label{eq:def-delta-and-iota-overly}
\end{align}

The unweighted regression for fitting second-moment function values is as follows:

\begin{equation}\label{eq:C-update-rule-RL-second}
\begin{aligned}
\forall~t\in[T], h\in[H],~~~\hat{g}_t^h & = \arg\min_{g^h\in\calF^h} \sum_{s\in[t-1]}\left(g^h\left(x^{h}_s,a^{h}_s\right)-\left(r^{h}_s+f_{t,1}^{h+1}\left(x_s^{h+1}\right)\right)^2\right)^2,\\
\text{and similarly let}~\calG^h_t & \defeq \left\{g^h\in\calF^h:\sum_{s\in[t-1]}\left(g^h(x_s^{h}, a_s^{h})-\hat{g}_t^h(x_{s}^{h},a_s^{h})\right)^2\le \left(\bbeta_t^h\right)^2\right\}.
\end{aligned}
\end{equation}
Here we choose parameters:
\begin{align}
    \bbeta_t^h & \defeq \sqrt{8(11L+9)\left(\iota'(\delta_{t,h})\right)^2+32tL\epsilon},\label{eq:def-beta-second}\\
    \delta_{t,h} & \defeq \frac{\delta}{(T+1)(H+1)}~~~\text{and}~~~\iota'(\delta) \defeq\sqrt{2\log\frac{\calN\calN_b\left(2\log(32LT)+2\right)\left(\log(32L)+2\right)}{\delta}}.\label{eq:def-delta-and-iota-second}
\end{align}

In all the definitions above, we note that $\delta_{t,h}$ is independent of $t$ and $h$. We also observe $\sum_{t\in[T],h\in[H]}\delta_{t,h}\le \delta$. We will define the following good (probabilistic) event: $\calE_{t,j}^h\defeq\{\bar{f}_{t,j}^h\in\calF^h_{t,j}\}$ with $j = 1,\pm2$. Further, we let $\bar{\calE}_{t}\defeq\{\psi_t^h\in\calG_t^h\}$. Further, we let $\calE_t^h = \calE_{t,1}^h\cap \calE_{t,2}^h\cap \calE_{t,-2}^h\cap \bar{\calE}_{t}^h$, and use $\calE_t = \cap_{h\in[H]}\calE_{t}^h$ and $\calE_{\le t} = \cap_{t'\le t}\calE_{t'}$ as shorthand for joint events.

\paragraph{Design of exploration policy.} We restate the exploration policy for generating new data trajectory as stated in the main paper, i.e.~\Cref{eq:greedy-policy-informal}. At each iteration $t$, the algorithm collects data using both optimistic sequence $f_{t,1}^h$ and overly optimistic $f_{t,2}^h$. Given a sequence of pre-specified $\{u_t\}_{t\in[T]}$, at each iteration $t$, with function $f_{t,1}^h(\cdot)$ and $f_{t,2}^h(\cdot)$ at hand, we choose actions based on the following rule:
\begin{equation}\label{eq:greedy-policy}
a_t^h =\left\{\begin{array}{cc}
     \argmax_{a\in\cA} f_{t,1}^h(x_t^h,a) & \mbox{if $f_{t,1}^h(x_t^{h'}) \geq f_{t,2}^h(x_t^{h'}) - u_t$ for all $h' \leq h$},\\
    \argmax_{a\in\cA} f_{t,2}^h(x_t^h,a) & \mbox{otherwise},
\end{array}    \right.
\end{equation}

We also use $h_t\in[H+1]$ to denote the (random) threshold at which we first start taking greedy action based on overly optimistic sequence $f_{t,2}$. We divide the iteration set into the disjoint subsets $[T] = \calTo\cup\calToo$ so that
\begin{align}\label{def:calT}
    \calTo \defeq\{t\in[T]: h_t=H+1\}~~\text{and}~~\calToo \defeq\{t\in[T]: h_t\in[H]\}.
\end{align}

Al step $t$, we note the policy induced by our defined exploration rule~\eqref{eq:greedy-policy} at step $t,h$ given state $x_t^h$ depends on $\calD_{[t-1]}^h$ and also the new generated trajectory $\{x_1^h, x_2^h, \cdots, x_{[t-1]}^h, x_t^h\}$. For each $h\in[H]$, we use $V_t$ to denote the expected $V$-value only at the trajectory under the exploration rule at step $t$, formally as follows:
\begin{align}\label{def:explore-V}
   V_t^h\defeq  \E\left[\sum_{h'\ge h}r^{h'}|x_{t}^{[h]}, f_{t,1}^{[H]}, f_{t,2}^{[H]}\right].
\end{align}
The expectation is taken with respect to both model transition and exploration policy.

\paragraph{Other notations.} 
We define the martingale difference sequence so that 
\begin{align*}
    \xi_{t,j}^h & \defeq r^{h}_t+f_{t,j}^{h+1}(x^{h+1}_{t})-\E_{r^h,x^{h+1}}\left[r^h+f_{t,j}^{h+1}(x^{h+1})|z_t^{h}\right],~~\text{for}~j=1,-2,2,\\
    \xi_{t}^h & \defeq r^{h}_t+V_t^{h+1}-\E\left[r^h+V_t^{h+1}|z_t^{[h]}, f_{t,1}^{[H]}, f_{t,2}^{[H]}\right].
\end{align*}

We use the following shorthand of summation of bonus terms.
\begin{align*}\label{def:I-II}
 I & \defeq \sum_{t\in[T]}\sum_{h\in[H]}\min\left(1+L,b_{t,1}^h(z_{t}^h)\right), \\
    II & \defeq \sum_{t\in[T]}\sum_{h\in[H]}\min\left(1+L,b_{t,2}^h(z_{t}^h)\right).
\end{align*}

In the high-probability regret proof in~\Cref{app:regret-hp}, we also define probabilistic event $\calE_{\xid}$, $\calE_{\rV}$, $\calE_{\xi}$, $\calE_{\xi_1}$, $\calE_{\xi_2}$, $\calE_{\xi_{-2}}$, we refer readers to~\Cref{lem:concentration-indicator,coro:sum-variance,lem:freedman-xi} for their concrete meanings.

\paragraph{Helper lemmas.} We now include a few helper lemmas that will be used in multiple parts of the analysis of~\Cref{alg:fitted-Q-simpler}. The first few lemmas characterize the concentration behavior of martingale difference sequence, and are used heavily in~\Cref{app:CI}, and for proving~\Cref{lem:concentration-indicator,lem:freedman-xi}.

\begin{lemma}[Freedman's inequality, cf. Theorem of~\cite{beygelzimer2011contextual}]\label{lem:Freedman}
Let $M,v>0$ be constants, and $\{x_i\}_{i\in[t]}$ be stochastic process adapted to a filtration $\{\calH_i\}_{i\in[t]}$. Suppose $\E[x_i|\calH_{i-1}]=0$,  $|x_i|\le M$ and $\sum_{i\in[t]}\E[x_i^2|\calH_{i-1}]\le V^2$. Then for any $\delta>0$, with probability at least $1-\delta$ we have
\[
\sum_{i\in[t]}x_i\le 2V\sqrt{\log(1/\delta)}+M\log(1/\delta).
\]
\end{lemma}

\begin{lemma}[Unbounded version of Freedman's inequality, cf.~\cite{dzhaparidze2001bernstein} and~\cite{fan2017martingale}]\label{lem:Freedman-variant}
Let $\{x_i\}_{i\in[t]}$ be adapted to filtration $\{\calH_i\}_{i\in[t]}$. Suppose $\E[x_i|\calH_{i-1}]=0$ and $\E[x_i^2|\calH_{i-1}]<\infty$ almosy surely. Then for any $a$, $v$, $y>0$ we have
\[
\P\left(\sum_{i\in[t]}x_i>a,~\sum_{i\in[t]}\left(\E[x_i^2|\calH_{i-1}]+x_i^2\cdot\1_{\{|x_i|>y\}}\right)<v^2\right)\le \exp\left(\frac{-a^2}{2(v^2+ay/3)}\right).
\]
\end{lemma}

\begin{corollary}\label{coro:Freedman-variant}
Let $M>0,V>v>0$ be constants, and $\{x_i\}_{i\in[t]}$ be stochastic process adapted to a filtration $\{\calH_i\}_{i\in[t]}$. Suppose $\E[x_i|\calH_{i-1}]=0$,  $|x_i|\le M$ and $\sum_{i\in[t]}\E[x_i^2|\calH_{i-1}]\le V^2$ almost surely. Then for any $\delta,\epsilon>0$, let $\iota = \sqrt{\log\frac{\left(2\log(V/v)+2\right)\cdot\left(\log(M/m)+2\right)}{\delta}}$ we have
\[
    \P\left(\sum_{i\in[t]}x_i>\iota\sqrt{2\left(2\sum_{i\in[t]}\E[x_i^2|\calH_{i-1}]+ v^2\right)}+\frac{2}{3}\iota^2\left(2\max_{i\in[t]}|x_i|+m\right) \right)\le \delta.
\]
\end{corollary}

Next, we state some helper lemmas on law of total variance, which are used in proving~\Cref{coro:sum-variance}.

\begin{proposition}[Law of total variance, LTV]\label{prop:LTV-1}
Suppose at step $t$, we use policy $\pi_t$ and have value function following such policy as $\{V_{\pi_t}^h\}_{h\in[H]}$, then by law of total variance we have 
\[
\rV\left[\sum_{h\in[H]}r^h_t\right] = \E\left[\sum_{h\in[H]}\rV\left[r^h_t+V_{\pi_t}^{h+1}|z_t^{[h]}, f_{t,1}^{[H]}, f_{t,2}^{[H]}\right]\right]\le 1.
\]
\end{proposition}

We now provide an adapted version of LTV applying to $f_{t,1}$, which generates our greedy policy only when 
$t\in\calTo$ based on the exploration rule as in~\Cref{eq:greedy-policy}. 

\begin{proposition}[Adapted version using LTV]\label{prop:LTV-2}
Suppose at step $t$, the agent explores based on rule~\eqref{eq:greedy-policy}. Then conditioning on the past $\calH_{t-1}^H =\sigma(x_1^{1},r_1^1, x_1^2,\cdots, r^{H}_1,x^{H+1}_1; x_2^{1},r_2^1, x_2^2,\cdots;\cdots;x_{t-1}^1,r_{t-1}^1,\cdots, r_{t-1}^{H},x_{t-1}^{H+1})$, we have 
\begin{align*}
& \rE \left[\sum_{h=1}^H \rV_{r^h, x^{h+1}} [r^h+f_{t,1}^{h+1}(x^{h+1})| z_t^{h}]~|~\calH_{t-1}^H\right]\leq 2 \rE\left[ \left(\sum_{h=1}^H r^h - f_{t,1}^1(x_t^1)\right)^2~|~\calH_{t-1}^H\right]\\
& \hspace{12em}+ 2 \rE \left[\left(\1(t \in \calTo)\sum_{h=1}^H \left(f_{t,1}^h(x_t^h,a_t^h)- \rE _{r^h, x^{h+1}}[r^h+f_{t,1}^{h+1}(x^{h+1})|z_t^{h}]\right)\right)^2~|~\calH_{t-1}^H\right]\\
& \hspace{12em} + 2 \rE\left[ \left(\1(t \in \calToo)\sum_{h=1}^H \left(f_{t,1}^h(x_t^h)- \rE_{r^h, x^{h+1}} [r^h+f_{t,1}^{h+1}(x^{h+1})|z_t^{h}]\right)\right)^2~|~\calH_{t-1}^H\right].
\end{align*}
\end{proposition}

\begin{proof}
We use $\rE[\cdot|z_t^h] = \rE_{r^h, x^{h+1}}[\cdot|z_t^h]$ and $\rV[\cdot|z_t^h] = \rV_{r^h, x^{h+1}}[\cdot|z_t^h]$ for simplicity. By conditional expectation and law of total variance that
\begin{align*}
&  \rE \left[\sum_{h=1}^H \rV [r^h+f_{t,1}^{h+1}(x^{h+1})| z_t^{h}]~|~\calH_{t-1}^H\right]= \rE \left[\left(\sum_{h=1}^H \left(r_t^h+f_{t,1}^{h+1}(x_t^{h+1}) - \rE [r^h+ f_{t,1}^{h+1}(x^{h+1})|z_t^{h}]\right)\right)^2~|~\calH_{t-1}^H\right] \\
& \hspace{3em} \stackrel{(\star)}{=} \rE\bigg[ \bigg(\sum_{h=1}^H r_t^h - f_{t,1}^1(x_t^1) + \sum_{h=1}^H \1_{\{t\in\calTo\}}\left(f_{t,1}^{h}(x_t^{h},a_t^h) - \rE [r^h + f_{t,1}^{h+1}(x^{h+1})|z_t^{h}]\right)\\
& \hspace{13em} +\sum_{h=1}^H \1_{\{t\in\calToo\}}\left(f_{t,1}^{h}(x_t^{h}) - \rE [r^h + f_{t,1}^{h+1}(x^{h+1})|z_t^{[h]}, f_{t,1}^{[H]}, f_{t,2}^{[H]}]\right)\bigg)^2~|~\calH_{t-1}^H\bigg] \\
& \hspace{3em} \le 2\rE \bigg[\bigg(\sum_{h=1}^H r_t^h - f_{t,1}^1(x_t^1)\bigg)^2~|~\calH_{t-1}^H\bigg]\\
& \hspace{10em} + 2\rE\left[\left(1_{\{t\in\calTo\}}\sum_{h=1}^H \left(f_{t,1}^{h}(x_t^{h},a_t^h) - \rE [r^h+f_{t,1}^{h+1}(x^{h+1})|z_t^{h}]\right)\right)^2~|~\calH_{t-1}^H\right]\\
& \hspace{10em} + 2\rE\left[\left(1_{\{t\in\calToo\}}\sum_{h=1}^H \left(f_{t,1}^{h}(x_t^{h}) - \rE [r^h+f_{t,1}^{h+1}(x^{h+1})|z_t^{h}]\right)\right)^2~|~\calH_{t-1}^H\right],
\end{align*}
here we use  $(\star)$ the fact that whenever $t\in\calTo$, we have $f_{t,1}^h(x_t^h, a_t^h) = f_{t,1}^h(x_t^h)$.
\end{proof}

\subsection{Confidence Intervals' Properties}\label{app:CI}

In this subsection, we justify the choices of parameters $\beta_{t,j}^h$ and $\bbeta_t^h$ in constructing our confidence intervals $\calF_{t,j}^h$ for $j=1,\pm2$ and $\calG_t^h$. We show with high probability under moderate assumptions, it holds that $\bar{f}_{t,j}\in\calF_{t,j}^h$ for $j=1,\pm2$ and $\psi_t^h\in\calG_t^h$ using martingale concentration. We write $\rV[\cdot|z_t^h]=\rV_{r^h, x^{h+1}}[\cdot|z_t^h]$ and $\E[\cdot|z_t^h] = \E_{r^h, x^{h+1}}[\cdot|z_t^h]$ where the randomness is taken with respect to only $r^h$ and $x^{h+1}$ due to model transition by abusing notation, when the meaning is clear from context.

\paragraph{Confidence interval of optimistic sequence.} This paragraph proves the property of the optimistic confidence interval we construct for $Q_\star^h$.

\begin{lemma}\label{lem:confidence-interval-RL}
At step $t\in[T]$ and horizon $h\in[H]$, suppose 
\begin{align}
   &~ (\sigma_s^{h})^2\ge \rV\left[r^{h}+f_\star^{h+1}(x^{h+1})|z_s^h\right]~&~\text{for all}~~s\in[t-1],\label{eq:CI-RL-conds-I}\\
    \text{and}~&~\E\bigl[|f_{t,1}^{h+1}(x^{h+1})-f_{\star}^{h+1}(x^{h+1})|~|~z_s^h\bigr] \le f_{s,2}^h(z_s^{h})- f_{s,-2}^h(z_s^h),~&~\text{for all}~~s\in[t-1],z_s^{h}\in\calS\times\calA.\label{eq:CI-RL-conds-II}
\end{align}
Recalling that $\bar{f}^h_{t,1}(x^{h}, a^{h})\in\calF^h$ as some function such that $|\bar{f}^h_{t,1}(z^h) - \calT f_{t,1}^{h+1}(z^h)|\le \epsilon$ for all $z^h = (x^h, a^h)$, we have with probability $1-2\delta_{t,h}$, it holds that $\bar{f}_{t,1}^h\in\calF^h_{t,1}$ for the constructed $\calF^h_{t,1}$ based on the definition of confidence interval in~\eqref{eq:C-update-rule-opti} and $\beta_{t,1}^h$ in~\eqref{eq:def-beta-opti}.
\end{lemma}

To prove this lemma, we first show the concentration properties of two martingale difference sequences (MDSs), building on~\Cref{coro:Freedman-variant}.

\begin{lemma}\label{lem:confidence-interval-helper-I}
Under the same setting and condition~\eqref{eq:CI-RL-conds-I} as in~\Cref{lem:confidence-interval-RL}, consider filtration defined as $\calH_s^{h}=\sigma(x_1^{1},r_1^1, x_1^2,\cdots, r^{H}_1,x^{H+1}_1; x_2^{1},r_2^1, x_2^2,\cdots, r^{H}_2,x^{H+1}_2;\cdots,x_s^1,r_s^1,\cdots, r_s^{h},x_s^{h+1})$, we consider for any fixed $f, \tilde{f}\in[0,L]$, define 
\begin{align*}
\eta^{h}_s & \defeq r^{h}_s+{f_\star^{h+1}}\left(x_s^{h+1}\right)-\E\left[r^{h}+f_\star^{h+1}\left(x^{h+1}\right)|z^{h}_s\right],\\
\text{and MDS}~~~\D_{s}^{h}[f,\tf] & \defeq 2\frac{\eta_s^{h}}{\left(\bsigma^{h}_s\right)^2}\cdot\left(f\left(z_s^{h}\right)-\tf\Par{z_s^h}\right),
\end{align*}
then we have with probability $1-\delta_{t,h}/\calN^2$,
\begin{align}
    \sum_{s\in[t-1]} \D_{s}^{h}[f,\tf] & \le  \frac{4}{3}\upsilon(\delta_{t,h})\sqrt{\lambda}+ \frac{2}{3}\upsilon^2(\delta_{t,h})+\sqrt{2}\upsilon(\delta_{t,h})+\frac{ 16^2\upsilon^2(\delta_{t,h})}{ 3^2}+\frac{\sum_{s\in[t-1]}\frac{1}{\left(\bsigma_s^{h}\right)^2}\left(f(z_s^{h})-\tf(z_s^h)\right)^2}{4},\label{eq:CI-sum-D}\\
~\text{where we recall}~&~\delta_{t,h} = \frac{\delta}{(T+1)(H+1)},~\upsilon(\delta_{t,h}) = \sqrt{\log\frac{\calN^2\left(2\log(4LT/\alpha)+2\right)\left(\log(8L/\alpha^2)+2\right)}{\delta_{t,h}}}~\text{as in}~\eqref{eq:def-delta-and-upsilon-opti}.\nonumber
\end{align}
\end{lemma}

\begin{proof}
By definition of $\D_{s}^{h}[f,\tf] = 2\frac{\eta_s^{h}}{\left(\bsigma^{h}_s\right)^2}\cdot\left(f\left(z_s^{h}\right)-\tf\Par{z_s^h}\right)$, it is a well-defined martingale difference sequence adapted to $\calH_{s}^{h}$. In order to apply~\Cref{coro:Freedman-variant}, we first give the almost-surely bounds on its maximum scale $M$ and sum of second moment $V$. It holds that
\begin{align}
      \left|\D_{s}^{h}\left[f,\tf\right]\right| & \le \frac{2|\eta_s^{h}|\cdot\max_{z^{h}_s}|f(z^{h}_s)-\tf(z^{h}_s)|}{\alpha^2}\le \frac{8L}{\alpha^2} =: M~~\text{w.p.}~1, \label{eq:CI-conditions-m},\\
    \sum_{s\in[t-1]}\E\left[\left(\D_{s}^{h}\left[f,\tf\right]\right)^2|z_s^{h}\right] & = \sum_{s\in[t-1]} 4\frac{\E\left[\left(\eta_s^{h}\right)^2|z_s^{h}\right]}{\left(\bsigma_s^{h}\right)^4}\left(f(z_s^{h})-\tf(z_s^{h})\right)^2\le \frac{16L^2}{\alpha^2}(t-1)\le \Par{\frac{4LT}{\alpha}}^2 =: V^2. \label{eq:CI-conditions-V}
\end{align}
Here we only use the size bound on $f, \tilde{f}\in[0,L]$ by~\Cref{ass:eps-realizability-RL}, $\bsigma_s^h\ge\alpha$ by definition~\eqref{eq:def-bsigma}, and $|\eta_s^h|\le 2$ since $r_s^h, f_\star^{h+1}\in[0,1]$ always hold true.

Additionally, we also have the following realization-dependent bounds on the maximum scale and sum of second moment for the MDS sequence.
\begin{align}
    \sum_{s\in[t-1]}\E\left[\left(\D_{s}^{h}\left[f,\tf\right]\right)^2|z_s^{h}\right] & = \sum_{s\in[t-1]} 4\frac{\E\left[\left(\eta_s^{h}\right)^2|z_s^{h}\right]}{\left(\bsigma_s^{h}\right)^4}\left(f(z_s^{h})-\tf(z_s^{h})\right)^2 \stackrel{(i)}{\le} 4\sum_{s\in[t-1]}\frac{1}{\left(\bsigma_s^{h}\right)^2}\left(f(z_s^{h})-\tf(z_s^{h})\right)^2 \label{eq:CI-conditions-second-moment}\\
    \max_{s\in[t-1]}\left|D_s^{h}\left[f,\tf\right]\right| & = \max_{s\in[t-1]}2\left|\frac{\eta_s^{h}}{\left(\bsigma_s^{h}\right)^2}\right|\cdot\left| f(z^{h}_{s})-\tf(z^{h}_s)\right|\nonumber\\
    & \hspace{1em} \stackrel{(ii)}{\le} \max_{s\in[t-1]}\frac{4}{\left(\bsigma_s^{h}\right)^2}\sqrt{D^2_{\calF^h}(z_s^{h};z_{[s-1]}^h, \bsigma_{[s-1]}^h)\left(\sum_{i\in[s-1]}\frac{1}{\left(\bsigma_i^h\right)^2}\left(f(z^h_i)- \tf(z^h_i)\right)^2+\lambda\right)}\nonumber\\ & \hspace{1em} \stackrel{(iii)}{\le}  \frac{1}{\upsilon(\delta_{t,h})}\sqrt{\sum_{s\in[t-2]}\frac{1}{\left(\bsigma_s^{h}\right)^2}\left(f(z_s^{h})-\tf(z_s^{h})\right)^2+\lambda} ,\label{eq:CI-conditions-first-moment}
\end{align}
Here we use $(i)$ the assumption in~\eqref{eq:CI-RL-conds-I} such that $\E\left[\left(\eta_s^{h}\right)^2|z_s^{h}\right] = \rV\left[r^{h}+f_\star^{h+1}(x^{h+1})|z_s^h\right]\le\left(\sigma_s^{h}\right)^2\le \left(\bsigma_s^{h}\right)^2$ and  $(ii)$ the size bound that $|\eta_s^{h}|\le 2$ since $r^h,f_\star^{h+1}\in[0,1]$ together with the definition of $D_{\calF^h}$, 
and $(iii)$ the choice of $\bsigma_s^{h}\ge 2\sqrt{D_{\calF^h}(z_s^{h};z_{[s-1]}^{h},\bsigma_{[s-1]}^{h})\cdot\upsilon(\delta_{t,h})}$ for all $s\in[t-1]$ and then taking the $\max_{s\in[t-1]}$ inside the summation of $i\in[s-1]$. 

Applying~\Cref{coro:Freedman-variant} with $M = 8L/\alpha^2$ \eqref{eq:CI-conditions-m}, $V = 4LT/\alpha$ \eqref{eq:CI-conditions-V}, $v = m =1$, using~\eqref{eq:CI-conditions-second-moment} and~\eqref{eq:CI-conditions-first-moment} we conclude that with probability at least $1-\delta_{t,h}/\calN^2$, for $\upsilon(\delta_{t,h}) = \sqrt{\log\frac{\calN^2\left(2\log(4LT/\alpha)+2\right)\left(\log(8L/\alpha^2)+2\right)}{\delta_{t,h}}}$,
\begin{align*}
& \sum_{s\in[t-1]} 2\frac{\eta_s^{h}}{\left(\bsigma_s^{h}\right)^2}\left(f(z_s^{h})-\tf(z_s^{h})\right) \le \upsilon(\delta_{t,h})\sqrt{16\left(\sum_{s\in[t-1]}\frac{1}{\left(\bsigma_s^{h}\right)^2}\left(f(z_s^{h})-\tf(z_s^{h})\right)^2\right)+2}\\
& \hspace{16em} +\frac{2}{3}\upsilon^2(\delta_{t,h})\cdot\Par{ \frac{2}{\upsilon(\delta_{t,h})}\sqrt{\sum_{s\in[t-1]}\frac{1}{\left(\bsigma_s^{h}\right)^2}\left(f(z_s^{h})-\tf(z_s^{h})\right)^2+\lambda}+1}\\
& \hspace{2em} \le \frac{4}{3}\upsilon(\delta_{t,h})\sqrt{\lambda}+ \frac{2}{3}\upsilon^2(\delta_{t,h})+\sqrt{2}\upsilon(\delta_{t,h})+\frac{16}{3}\upsilon(\delta_{t,h})\sqrt{\sum_{s\in[t-1]}\frac{1}{\left(\bsigma_s^{h}\right)^2}\left(f(z_s^{h})-\tf(z_s^{h})\right)^2}\\
& \hspace{2em}\le \frac{4}{3}\upsilon(\delta_{t,h})\sqrt{\lambda}+ \frac{2}{3}\upsilon^2(\delta_{t,h})+\sqrt{2}\upsilon(\delta_{t,h})+\frac{ 2\cdot 16^2\upsilon^2(\delta_{t,h})}{2\cdot 3^2}+\frac{\sum_{s\in[t-1]}\frac{1}{\left(\bsigma_s^{h}\right)^2}\left(f(z_s^{h})-\tf(z_s^h)\right)^2}{2\cdot 2}.
\end{align*}
The last inequality above uses AM-GM inequality and this concludes the proof.
\end{proof}

\begin{lemma}\label{lem:confidence-interval-helper-II}
Under the same setting and condition~\eqref{eq:CI-RL-conds-II} as in~\Cref{lem:confidence-interval-RL}, consider filtration defined as $\calH_s^{h}=\sigma(x_1^{1},r_1^1, x_1^2,\cdots, r^{H}_1,x^{H+1}_1; x_2^{1},r_2^1, x_2^2,\cdots, r^{H}_2,x^{H+1}_2;\cdots,x_s^1,r_s^1,\cdots, r_s^{h},x_s^{h+1})$, we consider for any fixed $f, \tilde{f}\in[0,L]$, and $f'= f_{t,1}^{h+1}$, define 
\begin{align*}
\xi^h_s[f'] & = f'\left(x_s^{h+1}\right)-{f_\star^{h+1}}\left(x_s^{h+1}\right)-\E\left[f'\left(x^{h+1}\right)-f_\star^{h+1}\left(x^{h+1}\right)|z^{h}_s\right],\\
\text{and MDS}~~~\Delta_{s}^{h} \left[f,\tf,f'\right] & = 2\frac{\xi_s^{h}[f']}{\left(\bsigma^{h}_s\right)^2}\cdot\left(f\left(z_s^{h}\right)-\tf\left(z_s^{h}\right)\right),
\end{align*}
then we have with probability $1-\delta_{t,h}/\calN^3\calN_b$,
\begin{align}
\sum_{s\in[t-1]} \Delta_{s}^{h} \left[f,\tf,f'\right] & \le \frac{4}{3}\sqrt{\lambda}+ \frac{2}{3}\cdot\frac{\iota^2(\delta_{t,h})}{\log\calN_b}+\sqrt{2}\cdot\frac{\iota(\delta_{t,h})}{\sqrt{\log\calN_b}}+\frac{ 16^2}{ 3^2}+\frac{\sum_{s\in[t-1]}\frac{1}{\left(\bsigma_s^{h}\right)^2}\left(f(z_s^{h})-\tf(z_s^h)\right)^2}{4},\label{eq:CI-sum-Delta}\\
~\text{where we recall}~&~\iota(\delta_{t,h}) = 3\sqrt{\log\frac{\calN\calN_b\left(2\log(4LT/\alpha)+2\right)\left(\log(8L/\alpha^2)+2\right)}{\delta_{t,h}}}~\text{as in}~\eqref{eq:def-iota-opti}.\nonumber
\end{align}
\end{lemma}

\begin{proof}
Note that the martingale difference $\Delta_s^h\left[f,\tf,f'\right]$ is adapted to $\calH_{s}^{h}$. In order to apply~\Cref{coro:Freedman-variant}, we again first give the almost-surely bounds on its maximum scale $M$ and sum of second moment $V$. It holds that
\begin{align}
    \left|\Delta_s^h\left[f,\tf,f'\right]\right| & \le \frac{2|\xi_s^{h}[f']|\cdot\max_{z^{h}_s}|f(z^{h}_s)-\tf(z^{h}_s)|}{\alpha^2}\le \frac{8L}{\alpha^2}=: M, \label{eq:CI-conditions-m-2} \\
    \sum_{s\in[t-1]}\E\left[\left(\Delta_s^h\left[f,\tf,f'\right]\right)^2|z_s^{h}\right] & = \sum_{s\in[t-1]} 4\frac{\E\left[\left(\xi_s^{h}[f']\right)^2|z_s^h\right]}{\left(\bsigma_s^{h}\right)^4}\left(f(z_s^{h})-\tf(z_s^{h})\right)^2\le \Par{\frac{4L}{\alpha}T}^2=:V^2. \label{eq:CI-conditions-v-2}
\end{align}
Here we again use $f,\tilde{f}\in[0,L]$ and $\bsigma_s^h\ge\alpha$ and $|\xi_s^h|\le 2$ since $f'=f_{t,1}^{h+1} \in[0,1]$.

Additionally, we also have the following realization-dependent bounds on the maximum magnitude and sum of second moment.
\begin{align}    
\sum_{s\in[t-1]}\E\left[\left(\Delta_s^h\left[f,\tf,f'\right]\right)^2|z_s^{h}\right] & = \sum_{s\in[t-1]} 4\frac{\E\left[\left(\xi_s^{h}[f']\right)^2|z_s^{h}\right]}{\left(\bsigma_s^{h}\right)^4}\left(f(z_s^{h})-\tf(z_s^{h})\right)^2\nonumber\\ & \hspace{1em} \stackrel{(i)}{\le} \frac{4}{\iota^2(\delta_{t,h})}\sum_{s\in[t-1]}\frac{1}{\left(\bsigma_s^{h}\right)^2}\left(f(z_s^{h})-\tf(z_s^{h})\right)^2,\label{eq:CI-conditions-second-moment-2}\\
\max_{s\in[t-1]}\left|\Delta_s^{h}\left[f,\tf,f'\right]\right| & = \max_{s\in[t-1]}2\left|\frac{\xi_s^{h}[f']}{\left(\bsigma_s^{h}\right)^2}\right|\cdot\left| f(z^{h}_{s})-\tf(z^{h}_s)\right|\nonumber\\
    & \stackrel{(ii)}{\le} \max_{s\in[t-1]}\frac{4}{\left(\bsigma_s^{h}\right)^2}\sqrt{D^2_{\calF^h}(z_s^{h};z_{[s-1]}^h, \bsigma_{[s-1]}^h)\left(\sum_{i\in[s-1]}\frac{1}{\left(\bsigma_i^h\right)^2}\left(f(z^h_i)-\tf(z^h_i)\right)^2+\lambda\right)}\nonumber\\
    &  \stackrel{(iii)}{\le}  \frac{1}{\iota^2(\delta_{t,h})}\sqrt{\sum_{s\in[t-2]}\frac{1}{\left(\bsigma_s^{h}\right)^2}\left(f(z_s^{h})-\tf(z_s^{h})\right)^2+\lambda}.\label{eq:CI-conditions-first-moment-2}
    \end{align}
Here we use $(i)$  $\E\left[\left(\xi_s^{h}[f']\right)^2|z_s^{h}\right] \le \E\left[ \Par{f_{t,1}^{h+1}(x^{h+1})-f_\star^{h+1}(x^{h+1})}^2|z_s^h\right] \le 2\E\left[|f_{t,1}^{h+1}(x^{h+1})-f_\star^{h+1}(x^{h+1})||z_s^h\right]$ $\le 2(f_{s,2}^h(z_s^h)-f_{s,-2}^h(z_s^h))\le \iota^{-2}(\delta_{t,h})\Par{\bsigma_s^h}^2$ given assumption~\eqref{eq:CI-RL-conds-II} and that $\delta_{t,h}$ doesn't depend on $t$ by definition. For  $(ii)$ we use the size bound that $|\xi_s^{h}[f']|\le 2$ since $r^h\in[0,1]$, and $f'\in[0,1]$ together with the definition of $D_{\calF^h}$
and $(iii)$ the choice of $\bsigma_s^{h}\ge 2\sqrt{D_{\calF^h}(z_s^{h};z_{[s-1]}^{h},\bsigma_{[s-1]}^{h})\cdot\iota^2(\delta_{t,h})}$ for all $s\in[t-1]$ and taking the $\max_{s\in[t-1]}$ inside the summation of $i\in[s-1]$. 

Applying~\Cref{coro:Freedman-variant} with $M = 8L/\alpha^2$ \eqref{eq:CI-conditions-m-2}, $V = 4LT/\alpha$ \eqref{eq:CI-conditions-v-2}, $v = 1/\sqrt{\log\calN_b}$ , $m = 1/\log\calN_b$, using~\eqref{eq:CI-conditions-second-moment-2} and~\eqref{eq:CI-conditions-first-moment-2} we conclude that with probability at least $1-\delta_{t,h}/(\calN^3\calN_b)$, 
it holds that 
\begin{align*}
& \sum_{s\in[t-1]} 2\frac{\xi_s^{h}[f']}{\left(\bsigma_s^{h}\right)^2}\left(f(z_s^{h})-\tf(z_s^{h})\right) \le \iota(\delta_{t,h})\sqrt{\frac{16}{\iota^2(\delta_{t,h})}\left(\sum_{s\in[t-1]}\frac{1}{\left(\bsigma_s^{h}\right)^2}\left(f(z_s^{h})-\tf(z_s^{h})\right)^2\right)+2\cdot \frac{1}{\log\calN_b}}\\
& \hspace{14em} +\frac{2}{3}\iota^2(\delta_{t,h})\cdot\Par{ \frac{2}{\iota^2(\delta_{t,h})}\sqrt{\sum_{s\in[t-1]}\frac{1}{\left(\bsigma_s^{h}\right)^2}\left(f(z_s^{h})-\tf(z_s^{h})\right)^2+\lambda}+\frac{1}{\log\calN_b}}\\
& \hspace{2em} \le \frac{4}{3}\sqrt{\lambda}+ \frac{2}{3}\cdot\frac{\iota^2(\delta_{t,h})}{\log\calN_b}+\sqrt{2}\cdot\frac{\iota(\delta_{t,h})}{\sqrt{\log \calN_b}}+\frac{16}{3}\sqrt{\sum_{s\in[t-1]}\frac{1}{\left(\bsigma_s^{h}\right)^2}\left(f(z_s^{h})-\tf(z_s^{h})\right)^2}\\
& \hspace{2em}\le \frac{4}{3}\sqrt{\lambda}+ \frac{2}{3}\cdot\frac{\iota^2(\delta_{t,h})}{\log\calN_b}+\sqrt{2}\cdot\frac{\iota(\delta_{t,h})}{\sqrt{\log\calN_b}}+\frac{ 2\cdot16^2}{2\cdot 3^2}+\frac{\sum_{s\in[t-1]}\frac{1}{\left(\bsigma_s^{h}\right)^2}\left(f(z_s^{h})-\tf(z_s^h)\right)^2}{2\cdot 2},
\end{align*}
for the choice of 
\begin{align*}\iota(\delta_{t,h}) & = 3\sqrt{\log\frac{\calN\calN_b\left(2\log(4LT/\alpha)+2\right)\left(\log(8L/\alpha^2)+2\right)}{\delta_{t,h}}}\\
& \ge \sqrt{\log\frac{\calN^3\calN_b\Par{2\log\Par{\frac{4LT\sqrt{\log\calN_b}}{\alpha}}+2}\Par{\log \frac{8L\log\calN_b}{\alpha^2}+2}}{\delta_{t,h}}}.
\end{align*}
Here for the last inequality we use $\log\log\calN_b\le \calN_b$.
\end{proof}

Making use of these two helper lemmas, we provide the complete proof for~\Cref{lem:confidence-interval-RL}.

\begin{proof}[Proof of~\Cref{lem:confidence-interval-RL}]
At step $t\in[T], h\in[H]$,  we locally denote the probability event $\calE_{t,h}$ as follows so that $\{\bar{f}_{t,1}^h\notin\calF^h_{t,1}\} \subseteq \calE_{t,h}$:
\begin{align*}
\calE_{t,h}
\defeq
\left\lbrace
\begin{array}{r@{}l}
   \sum\limits_{s\in[t-1]} \frac{1}{\left(\bsigma_s^{h}\right)^2}\left(\bar{f}^h_{t,1}\left(x_s^{h},a_s^{h}\right)-\hat{f}_{t,1}^h\left(x_s^{h},a_s^{h}\right)\right)^2 & > \left(\beta_{t,1}^h\right)^2
  \end{array}
  \right\rbrace.
\end{align*}

Note by definition of $\hat{f}_{t,1}^h$, we know that with probability $1$ it holds that \[ \sum\limits_{s\in[t-1]}\frac{1}{\left(\bsigma_s^{h}\right)^2}\left(\hat{f}_{t,1}^h\left(z_s^{h}\right)-\bar{f}^h_{t,1}\left(z_s^{h}\right)\right)^2 \le 2\sum\limits_{s\in[t-1]}\frac{\left(r^{h}_s+f^{h+1}_{t,1}\left(x_s^{h+1}\right)-\bar{f}^h_{t,1}(z^{h}_s)\right)}{\left(\bsigma_s^{h}\right)^2}\left(\hat{f}_{t,1}^h\left(z_s^{h}\right)-\bar{f}_{t,1}^h\left(z_s^{h}\right)\right).\] 
Thus this event can be equivalently expressed as 
\begin{align*}
    \left\lbrace
\begin{array}{r@{}l}
    \sum\limits_{s\in[t-1]}\frac{1}{\left(\bsigma_s^{h}\right)^2}\left(\hat{f}_{t,1}^h\left(z_s^{h}\right)-\bar{f}^h_{t,1}\left(z_s^{h}\right)\right)^2 & \le 2\sum\limits_{s\in[t-1]}\frac{\left(r^{h}_s+f^{h+1}_{t,1}\left(x_s^{h+1}\right)-\bar{f}^h_{t,1}(z^{h}_s)\right)}{\left(\bsigma_s^{h}\right)^2}\left(\hat{f}_{t,1}^h\left(z_s^{h}\right)-\bar{f}_{t,1}^h\left(z_s^{h}\right)\right)\\
   \sum\limits_{s\in[t-1]} \frac{1}{\left(\bsigma_s^{h}\right)^2}\left(\hat{f}_{t,1}^h\left(z_s^{h}\right)-\bar{f}^h_{t,1}\left(z_s^{h}\right)\right)^2 & > \left(\beta_{t,1}^h\right)^2
  \end{array}
  \right\rbrace.
\end{align*}

Now let $\hat{f}_{t,1}^h = f$, $\bar{f}_{t,1}^h = \tf$ and $f_{t,1}^{h+1} = f'$ so that $f_{t,1}^{h+1} = \min\left(f''+\epsilon,1\right)$ for some $f''\in\calF^{h+1}+\calW$. Now we apply Lemmas~\ref{lem:confidence-interval-helper-I} and~\ref{lem:confidence-interval-helper-II} with these choices, along with union bounds. Then, it holds that with probability at least $1-\delta_{t,h}$ that 
\begin{equation}\label{eq:CI-needed-1}
\begin{aligned}
    & 2\sum_{s\in[t-1]}\frac{\eta_s^{h}}{\left(\bsigma_s^{h}\right)^2}\left(\hat{f}_{t,1}^h(z_s^h)-\bar{f}_{t,1}^h(z_s^h)\right) \\
    & \hspace{2em} \le \frac{4}{3}\upsilon(\delta_{t,h})\sqrt{\lambda}+ \frac{2}{3}\upsilon^2(\delta_{t,h})+\sqrt{2}\upsilon(\delta_{t,h})+\frac{16^2\upsilon^2(\delta_{t,h})}{ 3^2}+ \frac{\sum_{s\in[t-1]}\frac{1}{\left(\bsigma_s^{h}\right)^2}\left(\hat{f}_{t,1}^h(z_s^{h})-\bar{f}_{t,1}^h(z_s^{h})\right)^2}{4},
\end{aligned}
\end{equation}

\begin{equation}\label{eq:CI-needed-2}
\begin{aligned}
    \mbox{and}~ & 2\sum_{s\in[t-1]}\frac{\xi_s^{h}[f_{t,1}^{h+1}]}{\left(\bsigma_s^{h}\right)^2}\left(\hat{f}_{t,1}^h(z_s^h)-\bar{f}_{t,1}^h(z_s^h)\right) \\
    & \hspace{2em} \le \frac{4}{3}\sqrt{\lambda}+ \frac{2}{3}\cdot\frac{\iota^2(\delta_{t,h})}{\log\calN_b}+\sqrt{2}\cdot\frac{\iota(\delta_{t,h})}{\sqrt{\log\calN_b}}+\frac{ 16^2}{ 3^2}+\frac{\sum_{s\in[t-1]}\frac{1}{\left(\bsigma_s^{h}\right)^2}\left(f(z_s^{h})-\tf(z_s^h)\right)^2}{4}\\
    & \hspace{2em} \le \frac{4}{3}\sqrt{\lambda}+ 6\Par{1+\upsilon^2(\delta_{t,h})}+3\sqrt{2}\sqrt{1+\upsilon^2(\delta_{t,h})}+\frac{16^2}{3^2}+ \frac{\sum_{s\in[t-1]}\frac{1}{\left(\bsigma_s^{h}\right)^2}\left(\hat{f}_{t,1}^h(z_s^{h})-\bar{f}_{t,1}^h(z_s^{h})\right)^2}{4}.
\end{aligned}
\end{equation}
Above for the last inequality we also use by definition of $\upsilon(\delta_{t,h})$ and $\iota(\delta_{t,h})$ that $\iota^2(\delta_{t,h})/\log\calN_b\le 9(1+\upsilon^2(\delta_{t,h}))$.

Combining~\Cref{eq:CI-needed-1,eq:CI-needed-2} and using $\upsilon^2(\delta_{t,h})\ge 1$ for upper bounding the coefficients, we have with probability $1-2\delta_{t,h}$, 
\begin{align*}
    & 2\sum\limits_{s\in[t-1]}\frac{\left(r^{h}_s+f^{h+1}_{t,1}\left(x_s^{h+1}\right)-\bar{f}^h_{t,1}(z^{h}_s)\right)}{\left(\bsigma_s^{h}\right)^2}\left(\hat{f}_{t,1}^h\left(z_s^{h}\right)-\bar{f}_{t,1}^h\left(z_s^{h}\right)\right)\\
    & \hspace{2em} \le 4t\frac{\epsilon}{\alpha^2}L+
    2\sum_{s\in[t-1]}\frac{\eta_s^h+\xi_s^{h}[f_{t,1}^{h+1}]}{\left(\bsigma_s^{h}\right)^2}\left(\hat{f}_{t,1}^h(z_s^h)-\bar{f}_{t,1}^h(z_s^h)\right) \tag{\Cref{ass:eps-realizability-RL}}\\
    & \hspace{2em} \le  \frac{1}{2}\cdot\sum_{s\in[t-1]}\frac{1}{\left(\bsigma_s^{h}\right)^2}\left(\hat{f}_{t,1}^h(z_s^{h})-\bar{f}_{t,1}^h(z_s^{h})\right)^2+(3\sqrt{\lambda}+78)\upsilon^2(\delta_{t,h})+\frac{4tL}{\alpha^2}\epsilon.
\end{align*}

This implies that
\begin{align*}
    \P(\calE_{t,h}) & \stackrel{(i)}{\le} \P\left(\sum_{s\in[t-1]} 2\frac{\left(r^{h}_s+f^{h+1}_{t,1}\left(x_s^{h+1}\right)-\bar{f}^h_{t,1}(z^{h}_s)\right)}{\left(\bsigma_s^{h}\right)^2}\left(\hat{f}_{t,1}^h(z_s^h)-\bar{f}_{t,1}^h(z_s^h)\right)>\frac{1}{2}\left(\beta_{t,1}^h\right)^2\right.\\ 
    & \qquad \qquad+
    \left.\frac{\sum_{s\in[t-1]}\frac{1}{(\bsigma_t^h)^2}\left(\hat{f}_{t,1}^h(z_s^h)-\bar{f}_{t,1}^h(z_s^h)\right)^2}{2}\right)\\
    & \stackrel{(ii)}{\le} \P\left(\sum_{s\in[t-1]} 2\frac{\left(r^{h}_s+f^{h+1}_{t,1}\left(x_s^{h+1}\right)-\bar{f}^h_{t,1}(z^{h}_s)\right)}{\left(\bsigma_s^{h}\right)^2}\left(\hat{f}_{t,1}^h(z_s^{h})-\bar{f}^h_{t,1}(z_s^{h})\right)>\Par{3\sqrt{\lambda}+8}\upsilon(\delta_{t,h})\right.\\ &\qquad\qquad\qquad+ \left.70\upsilon^2(\delta_{t,h})+\frac{4tL}{\alpha^2}\epsilon+\frac{\sum_{s\in[t-1]}\frac{1}{\left(\bsigma_t^{h}\right)^2}\left(f_n(z_s^{h})-\bar{f}^h_{t,1}(z_s^{h})\right)^2}{2}\right)\\
    & \le 2\delta_{t,h},
\end{align*} 
where we use $(i)$ the definition of $\calE_{t,h}$ and $(ii)$ the choice of $\beta_{t,1}^h$ as in~\Cref{eq:def-beta-opti}. Consequently, 
\[
\P\left(\bar{f}_{t,1}^h\notin\calF^h_{t,1} \right)\le \P\left(\calE_{t,h}\right)\le 2\delta_{t,h},
\]
which implies with probability $1-2\delta_{t,h}$, $\bar{f}^h_{t,1}\in\calF_{t,1}^h$ for any fixed given $t\in[T-1]$, $h\in[H]$ (the case $t=0$ holds with probability $1$ by definition).
\end{proof}

\paragraph{Confidence interval of overly optimistic and overly pessimistic sequence.} Here we prove properties of the overly optimistic and overly pessimistic confidence interval we construct for $Q_\star^h$.

\begin{lemma}\label{lem:confidence-interval-overly}
At step $t\in[T]$ and horizon $h\in[H]$, recall $\bar{f}_{t,2}^h(x^{h}, a^{h})\in\calF^h$ is some function such that $|\bar{f}_{t,2}^h(z^h)- \calT f_{t,2}^{h+1}(z^h)|\le \epsilon$ for all $z^h = (x^h, a^h)$, then we have with probability $1-\delta_{t,h}$, it holds that $\bar{f}_{t,2}^h\in\calF_{t,2}^h$ for the constructed $\calF_{t,2}^h$ based on the definition of confidence interval and $\beta_{t,2}$ in~\eqref{eq:def-beta-opti}.
\end{lemma}

Similar to proving the confidence interval of optimistic sequence, we first provide the following lemma.

\begin{lemma} Under the same setting as in~\Cref{lem:confidence-interval-overly}, consider filtration $\calH_s^{h}$ and any fixed pair functions $f\in[0,L]$ and $f'\in[0,2]$ we define random variables
\begin{align*}
\eta^{h}_s[f'] & \defeq  r^{h}_s+f'\left(x_s^{h+1}\right)-\E\left[r^{h}+f'\left(x^{h+1}\right)|z^{h}_s\right],\\
\text{and MDS}~~\D_{s}^{h}[f,f'] & \defeq 2\eta_s^{h}[f']\cdot\left(f\left(z_s^{h}\right)-\calT f'\left(z_s^{h}\right)\right),
\end{align*}
then we have with probability $1-\delta_{t,h}/(\calN^2\calN_b^2)$,
\begin{align}
\sum_{s\in[t-1]} \D_s^h[f,f'] &\le (24L+21)\dot{\iota}^2(\delta_{t,h})+\frac{\sum_{s\in[t-1]}\left(f(z_s^{h})-\calT f'(z^{h}_s)\right)^2}{2},\label{eq:CI-sum-D-overly}\\
~\text{where we recall}~&~\dot{\iota}(\delta_{t,h}) = \sqrt{2\log\frac{\calN\calN_b\left(2\log(18LT)+2\right)\left(\log(18L)+2\right)}{\delta_{t,h}}}~\text{as in}~\eqref{eq:def-delta-and-iota-overly}.\nonumber
\end{align}
\end{lemma}
\begin{proof}
Similar to the proof of~\Cref{lem:confidence-interval-helper-I}, we apply~\Cref{coro:Freedman-variant} on the defined MDS sequence $\D_s^h$. We first bound the quantities of interest:
\begin{align*}
    |\D_s^{h}[f,f']| & \le 2|\eta_s^{h}[f']|\max_{z^{h}_s}|f(z^{h}_s)-\calT f'(z^{h}_s)|\stackrel{(i)}{\le} 18L =:M,\\
    \sum_{s\in[t-1]}\E\left[\left(\D_s^{h}[f,f']\right)^2|z_s^{h}\right] & = \sum_{s\in[t-1]} 4\E\left(\eta_s^{h}[f']\right)^2\left(f(z_s^{h})-\calT f'(z_s^{h})\right)^2\stackrel{(i)}{\le} (18LT)^2 =:V^2;\\
    \sum_{s\in[t-1]}\E\left[\left(\D_s^{h}[f,f']\right)^2|z_s^{h}\right] & = \sum_{s\in[t-1]} 4\E\left[\left(\eta_s^{h}[f']\right)^2|z_s^{h}\right]\left(f(z_s^{h})-\calT f'(z_s^{h})\right)^2 \stackrel{(i)}{\le} 36\sum_{s\in[t-1]}\left(f(z_s^{h})-\calT f'(z_s^{h})\right)^2,
\end{align*}
where we use $(i)$ the size bound that $|\eta_s^{h}|\le 3$ and $\max_{z^{h}_s}|f(z^{h}_s)-\calT f'(z^{h}_s)|\le 3L$ (using $L\ge1$).

Thus, applying~\Cref{coro:Freedman-variant} with $M = 18L$, $V = 18LT$, $v = m =1$ to bound its summation we can conclude that with probability at least $1-\delta_{t,h}/(\calN^2\calN_b^2)$,
\begin{align*}
& \sum_{s\in[t-1]} 2\eta_s^{h}[f']\left(f(z_s^{h})-\calT f'(z_s^{h})\right)^2 \le \dot{\iota}(\delta_{t,h})\sqrt{36\left(\sum_{s\in[t-1]}\left(f(z_s^{h})-\calT f'(z_s^{h})\right)^2\right)+2v^2}\\
& \hspace{18em} +\frac{4}{3}\dot{\iota}^2(\delta_{t,h})\cdot 18L+\frac{2}{3}\dot{\iota}^2(\delta_{t,h})\\
& \hspace{2em} \le \sqrt{2}\dot{\iota}(\delta_{t,h})+\Par{24L+\frac{2}{3}}\dot{\iota}^2(\delta_{t,h})+6\dot{\iota}(\delta_{t,h})\sqrt{\sum_{s\in[t-1]}\left(f(z_s^{h})-\calT f'(z_s^{h})\right)^2}\\
& \hspace{2em}\le (24L+21)\dot{\iota}^2(\delta_{t,h})+\frac{\sum_{s\in[t-1]}\left(f(z_s^{h})-\calT f'(z^{h}_s)\right)^2}{2}.
\end{align*}
The last inequality again uses AM-GM inequality and the fact that $\dot{\iota}(\delta_{t,h})\ge 1$ by definition.
\end{proof}

This lemma helps us prove~\Cref{lem:confidence-interval-overly} as follows.

\begin{proof}[Proof of~\Cref{lem:confidence-interval-overly}]
At step $t\in[T], h\in[H]$, we locally define the probability event
\begin{align*}
\calE_{t,h}
\defeq
\left\lbrace
\begin{array}{r@{}l}
   \sum\limits_{s\in[t-1]} \left(\hat{f}_{t,2}^h\left(x_s^{h},a_s^{h}\right)-\bar{f}_{t,2}^{h}\left(x_s^{h},a_s^{h}\right)\right)^2 & > \left(\beta_{t,2}^h\right)^2
  \end{array}
  \right\rbrace
\end{align*}
so that $\{\bar{f}_{t,2}^h\notin\calF^h_{t,2}\} \subseteq \calE_{t,h}$.

Now by definition of $\hat{f}_{t,2}^h$, we know that with probability $1$ it holds that \[ \sum\limits_{s\in[t-1]}\left(\hat{f}_{t,2}^h\left(z_s^{h}\right)-\bar{f}_{t,2}^h\left(z_s^{h}\right)\right)^2  \le 2\sum\limits_{s\in[t-1]}\left(r^{h}_s+f_{t,2}^{h+1}\left(x_s^{h+1}\right)-\bar{f}_{t,2}^h(z^{h}_s)\right)\left(\hat{f}_{t,2}^h\left(z_s^{h}\right)-\bar{f}_{t,2}^h\left(z_s^{h}\right)\right).\] This event can be equivalently expressed as 
\begin{align*}
    \left\lbrace
\begin{array}{r@{}l}
    \sum\limits_{s\in[t-1]}\left(\hat{f}_{t,2}^h\left(z_s^{h}\right)-\bar{f}_{t,2}^h\left(z_s^{h}\right)\right)^2 & \le 2\sum\limits_{s\in[t-1]}\left(r^{h}_s+f_{t,2}^{h+1}\left(x_s^{h+1}\right)-\bar{f}_{t,2}^h(z^{h}_s)\right)\left(\hat{f}_{t,2}^h\left(z_s^{h}\right)-\bar{f}_{t,2}^h\left(z_s^{h}\right)\right)\\
   \sum\limits_{s\in[t-1]}\left(\hat{f}_{t,2}^h\left(z_s^{h}\right)-\bar{f}_{t,2}^h\left(z_s^{h}\right)\right)^2 & > \left(\beta_{t,2}^h\right)^2
  \end{array}
  \right\rbrace.
\end{align*}

Now for each particular pair of $f\in\calF^h$ where $\hat{f}_{t,2}^h = f$ and $f' = \min\left(1,f''+3\epsilon\right) = f_{t,2}^{h+1}$ where $f''\in\calF^{h+1}+2\calW+\calW$, 
we define the random variables
\[
\eta^{h}_s[f'] \defeq  r^{h}_s+f'\left(x_s^{h+1}\right)-\E\left[r^{h}+f'\left(x^{h+1}\right)|z^{h}_s\right],~~\text{and MDS}~~~\D_{s}^{h}[f,f'] = 2\eta_s^{h}[f']\cdot\left(f\left(z_s^{h}\right)-\calT f'\left(z_s^{h}\right)\right).
\]
Following~\eqref{eq:CI-sum-D-overly} we have with probability at least $1-\delta_{t,h}/\Par{\calN^2\calN_b^2}$,
\[
\sum_{s\in[t-1]}2\eta_s^{h}[f']\left(f(z_s^{h})-\calT f'(z_s^{h})\right)\le (16L+21)\dot{\iota}^2(\delta_{t,h})+\frac{\sum_{s\in[t-1]}\left(f(z_s^{h})-\calT f'(z^{h}_s)\right)^2}{2}.
\]
This implies that for any function $\bar{f}[f']$ satisfying  $\|\bar{f}[f']-\calT f'\|_\infty\le \epsilon$, it holds that with probability at least $1-\delta_{t,h}/\Par{\calN^2\calN_b^2}$,
\begin{align*}
    & 2\sum_{s\in[t-1]}\left(\eta_s^{h}[f']+\calT f'\left(z_s^h\right)-\bar{f}[f'](z^{h}_s)\right)\left(f(z_s^h)-\bar{f}[f'](z_s^h)\right)  \le \sum_{s\in[t-1]}2\eta_s^{h}[f']\left(f(z_s^{h})-\calT f'(z_s^{h})\right)+4tL\epsilon+4t\epsilon\\
    & \hspace{2em} \le  (24L+21)\dot{\iota}^2(\delta_{t,h})+ 8tL\epsilon +\frac{\sum_{s\in[t-1]}\left(f(z_s^{h})-\calT f'(z_s^{h})\right)^2}{2}\\
       & \hspace{2em} \le  (24L+21)\dot{\iota}^2(\delta_{t,h})+ 10tL\epsilon +\frac{\sum_{s\in[t-1]}\left(f(z_s^{h})-\bar{f} [f'](z_s^{h})\right)^2}{2}.
\end{align*}

Note the size of $\calF^h$ and $\calF^{h+1}+2\calW+\calW$ are bounded by $\calN$ and $\calN\calN_b^2$, we thus take a union bound over all choices of $f\in\calF^h$ and $f''\in\calF^{h+1}+2\calW+\calW$ so that 
\begin{align*}
     \P(\calE_{t,h}) & \stackrel{(i)}{\le} \P\left(\sum_{s\in[t-1]} 2\left(\eta_s^{h}+\E[r^h+f_{t,2}^{h+1}(x^{h+1})|z_s^h]-\bar{f}_{t,2}^h(z_s^{h})\right)\left(\hat{f}_{t,2}^h(z_s^{h})-\bar{f}_{t,2}^h(z_s^{h})\right)>\frac{1}{2}\left(\beta_{t,2}^h\right)^2\right.\\ 
    & \qquad \qquad+
    \left.\frac{\sum_{s\in[t-1]}\left(\hat{f}_{t,2}^h(z_s^h)-\bar{f}_{t,2}^h(z_s^h)\right)^2}{2}\right)\\
    & \stackrel{(ii)}{\le} \P\left(\sum_{s\in[t-1]} 2\left(\eta_s^{h}+\E[r^h+f_{t,2}^{h+1}(x^{h+1})|z_s^h]-\bar{f}_{t,2}^h(z_s^{h})\right)\left(\hat{f}_{t,2}^h(z_s^{h})-\bar{f}_{t,2}^h(z_s^{h})\right)>(24L+21)\dot{\iota}^2(\delta_{t,h})\right.\\
    & \qquad \qquad\left.+
  10tL\epsilon +\frac{\sum_{s\in[t-1]}\left(\hat{f}_{t,2}^h(z_s^{h})-\bar{f}_{t,2}^h(z_s^{h})\right)^2}{2}\right)\\
    & \le \delta_{t,h},
\end{align*}
where we use $(i)$ the definition of $\calE_{t,h}$ and $(ii)$ the choice of $\beta_{t,2}^h$. Consequently, 
\[
\P\left( \bar{f}_{t,2}^h\notin\calF^h_{t,2} \right)\le \delta_{t,h},
\]
which implies with probability $1-\delta_{t,h}$, $\bar{f}_{t,2}^h\in\calF_{t,2}^h$ for any fixed given $t\in[T-1]$, $h\in[H]$ (the case $t=0$ holds with probability $1$ by definition).
\end{proof}

Similarly, we have for overly pessimistic values $f_{t,-2}^h$ and $\bar{f}_{t,-2}^h$, the following lemma:

\begin{lemma}\label{lem:confidence-interval-RL-pessimistic}
At step $t\in[T]$ and horizon $h\in[H]$, recall $\bar{f}_{t,-2}^h(x^h, a^h)\in\calF^h$ is some function such that $|\bar{f}_{t,-2}^h(z^h) - \E\left[r^h+f_{t,-2}^{h+1}(x^{h+1})|z^h\right]\le \epsilon$ for all $z^h = (x^h, a^h)$, then we have with probability $1-\delta_{t,h}$, it holds that $\bar{f}_{t,-2}^h\in\calF_{t,-2}^h$ for the constructed $\calF_{t,-2}^h$ based on the definition of confidence interval in~\Cref{alg:fitted-Q-simpler} and $\beta_{t,2}$ in~\eqref{eq:def-beta-overly}.
\end{lemma}

\begin{proof}
The proof is symmetric as that of~\Cref{lem:confidence-interval-overly}.
\end{proof}

We also give the following consequence of~\Cref{lem:confidence-interval-RL-pessimistic} together with the definition of generalized Eluder dimension, which will be useful to justify our definition of $\sigma_t^h$ in~\Cref{eq:sigma-def-alg}.

\begin{lemma}\label{lem:RL-var-1}
Conditioning on the good event $\calE_{t,-2}^h$, we have \[\left|\left[\bar{f}_{t,-2}^h(z_{t}^h)\right]^2-\left[\hat{f}_{t,-2}^h(z_{t}^h)\right]^2\right|\le 2L\sqrt{\left(\beta_{t,2}^h\right)^2+\lambda} \cdot D_{\calF^h}(z_{t}^h; z_{[t-1]}^{h},\1_{[t-1]}^{h}).\]
\end{lemma}

\begin{proof}
    To see this, we note that conditioning on the good event $\bar{f}_{t,-2}^h(\cdot)\in\calF_{t,-2}^h$, we have for any $z$, \begin{align*}
    \left[\bar{f}_{t,-2}^h(z)\right]^2-\left[\hat{f}_{t,-2}^h(z)\right]^2 & \le 2L\left|\bar{f}_{t,-2}^h(z)-\hat{f}_{t,-2}^h(z)\right|\\
    & \le 2L\cdot D_{\calF^h}(z;z_{[t-1]}^h, \1_{[t-1]}^h)\cdot\sqrt{\sum_{s\in[t-1]}\left(\bar{f}_{t,-2}^h(z_s^h)-\hat{f}_{t,-2}^h(z_s^h)\right)^2+\lambda}\\
    & \le2L\cdot D_{\calF^h}(z;z_{[t-1]}^h, \1_{[t-1]}^h)\sqrt{(\beta_{t,2}^h)^2+\lambda}.
    \end{align*}
    Plugging the particular choice of $z = z_t^h$ concludes the proof.
\end{proof}

\paragraph{Confidence interval of second-moment sequence.} Here we prove the property of the optimistic confidence interval we construct for the second-moment sequence.

\begin{lemma}\label{lem:confidence-interval-RL-second}
At step $t\in[T]$ and horizon $h\in[H]$, recall $\psi^h_t(x^h, a^h)\in\calF^h$ satisfies $|\psi_t^h(z^h)- \calT_2 f^{h+1}_{t,1}(z^h)|\le \epsilon$ for any $z^h = (x^h, a^h)$, 
then we have with probability $1-\delta_{t,h}$, it holds that $\psi_t^h\in\calG^h_t$ for the constructed $\calG^h_t$ based on the definition of confidence interval in~\eqref{eq:C-update-rule-RL-second} and $\bbeta$ in~\eqref{eq:def-beta-second}.
\end{lemma}

Similar to proving the confidence intervals above we first provide the following lemma.

\begin{lemma}\label{lem:confidence-interval-helper-second} Under the same setting as in~\Cref{lem:confidence-interval-RL-second}, consider filtration $\calH_s^{h}$ and any fixed pair functions $f\in[0,L]$, $f'\in[0,1]$ we define random variables
\begin{align*}
\eta^{h}_s[f'] & \defeq \left(r^{h}_s+f'\left(x_s^{h+1}\right)\right)^2-\E\left[\left(r^{h}+f'\left(x^{h+1}\right)\right)^2|z^{h}_s\right],\\
\text{and MDS}~~~\D_{s}^{h}[f,f'] & \defeq  2\eta_s^{h}[f']\cdot\left(f\left(z_s^{h}\right)-\calT_2 f'\left(z_s^{h}\right)\right),
\end{align*}
then we have with probability $1-\delta_{t,h}/(\calN^2\calN_b)$,
\begin{align}
\sum_{s\in[t-1]} \D_s^h[f,f'] &\le 4(11L+9)\left(\iota'(\delta_{t,h})\right)^2+\frac{\sum_{s\in[t-1]}\left(f(z_s^{h})-\calT_2 f'(z^{h}_s)\right)^2}{2},\label{eq:CI-sum-D-second}\\
~\text{where we recall}~&~\iota'(\delta_{t,h}) = \sqrt{2\log\frac{\calN\calN_b\left(2\log(32LT)+2\right)\left(\log(32L)+2\right)}{\delta_{t,h}}}~\text{as in}~\eqref{eq:def-delta-and-iota-second}.\nonumber
\end{align}
\end{lemma}

\begin{proof}
Recall the definition of $\calT_2f (z^h)= \rE\left[\left(r^h+f(z^{h+1})\right)^2|z^h\right]$. We note the difference sequence $\D_s^h$ as defined is adapted to $\calH_{s}^{h}$ and satisfies
\begin{align*}
    |\D_s^{h}[f,f']| & \le 2|\eta_s^{h}|\max_{z^{h}_s}|f(z^{h}_s)-\calT_2 f'(z^{h}_s)|\stackrel{(i)}{\le} 32L =: M,\\
    \sum_{s\in[t-1]}\E\left[\left(\D_s^{h}[f,f']\right)^2|z_s^{h}\right] & = \sum_{s\in[t-1]} 4\E\left(\eta_s^{h}[f']\right)^2\left(f(z_s^{h})-\calT_2 f'(z_s^{h})\right)^2\stackrel{(i)}{\le} (32LT)^2 =: V^2\\
    \sum_{s\in[t-1]}\E\left[\left(\D_s^{h}[f,f']\right)^2|z_s^{h}\right] & = \sum_{s\in[t-1]} 4\E\left[\left(\eta_s^{h}[f']\right)^2|z_s^{h}\right]\left(f(z_s^{h})-\calT_2 f'(z_s^{h})\right)^2 \stackrel{(i)}{\le} 64\sum_{s\in[t-1]}\left(f(z_s^{h})-\calT_2 f'(z_s^{h})\right)^2,
\end{align*}
where we use $(i)$ the size bound that $|\eta_s^{h}[f']|\le 4$ and $\max_{z^{h}_s}|f(z^{h}_s)-\calT_2 f'(z^{h}_s)|\le 4L$.

Applying~\Cref{coro:Freedman-variant} with $M = 32L$, $V = 32LT$, $v = m =1$ to bound its summation we can conclude that with probability at least $1-\delta_{t,h}/\calN^2\calN_b^2$,
\begin{align*}
& \sum_{s\in[t-1]} 2\eta_s^{h}[f']\left(f(z_s^{h})-\calT_2 f'(z_s^{h})\right) \le \iota'(\delta_{t,h})\sqrt{64\left(\sum_{s\in[t-1]}\left(f(z_s^{h})-\calT_2 f'(z_s^{h})\right)^2\right)+2v^2}\\
& \hspace{18em} +\frac{4}{3}\left(\iota'(\delta_{t,h})\right)^2\cdot32L+\frac{2}{3}\left(\iota'(\delta_{t,h})\right)^2\\
& \hspace{2em} \le \Par{\frac{2+128L}{3}}\left(\iota'(\delta_{t,h})\right)^2+\sqrt{2}\iota'(\delta_{t,h})+8\iota'(\delta_{t,h})\sqrt{\sum_{s\in[t-1]}\left(f(z_s^{h})-\calT_2 f'(z_s^{h})\right)^2}\\
& \hspace{2em}\le 4(11L+9)\left(\iota'(\delta_{t,h})\right)^2+\frac{\sum_{s\in[t-1]}\left(f(z_s^{h})-\calT_2 f'(z^{h}_s)\right)^2}{2}.
\end{align*}
For the last inequality we use AM-GM inequality and by definition the fact that $\iota'(\delta_{t,h})\ge1$.
\end{proof}

This lemma again helps us prove~\Cref{lem:confidence-interval-RL-second} as follows.

\begin{proof}[Proof of~\Cref{lem:confidence-interval-RL-second}]
At step $t\in[T], h\in[H]$, we locally define the probability event 
\begin{align*}
\calE_{t,h}
\defeq
\left\lbrace
\begin{array}{r@{}l}
   \sum\limits_{s\in[t-1]} \left(\hat{g}_{t}^h\left(x_s^{h},a_s^{h}\right)-\psi_t^h\left(x_s^{h},a_s^{h}\right)\right)^2 & > \left(\bbeta_{t}^h\right)^2
  \end{array}
  \right\rbrace.
\end{align*}
so that  $\{\psi_{t}^h\notin\calG^h_{t}\} \subseteq \calE_{t,h}$

Now by definition of $\hat{g}_{t}^h$, we know that with probability $1$ it holds that $ \sum\limits_{s\in[t-1]}\left(\hat{g}_t^h\left(z_s^{h}\right)-\psi_t^h\left(z_s^{h}\right)\right)^2 \le 2\sum\limits_{s\in[t-1]}\left(\left(r^{h}_s+f_{t,1}^{h+1}\left(x_s^{h+1}\right)\right)^2-\psi_{t}^h(z^{h}_s)\right)\left(\hat{g}_t^h\left(z_s^{h}\right)-\psi_{t}^h\left(z_s^{h}\right)\right)$. This event can be equivalently expressed as This event can be equivalently expressed as 
\begin{align*}
    \left\lbrace
\begin{array}{r@{}l}
    \sum\limits_{s\in[t-1]}\left(\hat{g}_t^h\left(z_s^{h}\right)-\psi_t^h\left(z_s^{h}\right)\right)^2 & \le 2\sum\limits_{s\in[t-1]}\left(\left(r^{h}_s+f_{t,1}^{h+1}\left(x_s^{h+1}\right)\right)^2-\psi_{t}^h(z^{h}_s)\right)\left(\hat{g}_t^h\left(z_s^{h}\right)-\psi_{t}^h\left(z_s^{h}\right)\right)\\
   \sum\limits_{s\in[t-1]}\left(\hat{g}_t^h\left(z_s^{h}\right)-\psi_t^h\left(z_s^{h}\right)\right)^2 & > \left(\bbeta_{t}^h\right)^2
  \end{array}
  \right\rbrace.
\end{align*}

Now for any given pair of $f\in\calF^h$ and $f'$ such that $f' = \min\left(f''+\epsilon,1\right)$ for some $f''\in\calF^{h+1}+\calW$, when $\hat{g}_{t}^h = f$  and $f_{t,1}^{h+1} = f'$ we define random variable $\eta^{h}_s[f'] = \left(r^{h}_s+f'\left(x_s^{h+1}\right)\right)^2-\E\left[\left(r^{h}+f'\left(x^{h+1}\right)\right)^2|z^{h}_s\right]$ and the martingale difference sequence $\D_{s}^{h}[f,f'] = 2\eta_s^{h}[f']\cdot\left(f\left(z_s^{h}\right)-\calT_2 f'\left(z_s^{h}\right)\right)$. 

\Cref{eq:CI-sum-D-second} of~\Cref{lem:confidence-interval-helper-second} implies that for each particular pair of $(f,f')$ where $\hat{g}_{t}^h = f$ and $f_{t,1}^{h+1} = f'$, for any function $\bar{f}[f']\in\calF^h$ satisfying  $\|\bar{f}[f']-\calT_2 f'\|_\infty\le \epsilon$, it holds that with probability at least $1-\delta_{t,h}/\calN^2\calN_b$, we have 
\begin{align*}
    & 2\sum_{s\in[t-1]}\left(\eta_s^{h}[f']+\calT_2 f'\left(z_s^h\right)-\bar{f}[f'](z^{h}_s)\right)\left(f(z_s^h)-\bar{f}[f'](z_s^h)\right)  \le \sum_{s\in[t-1]}2\eta_s^{h}[f']\left(f(z_s^{h})-\calT_2f'(z_s^{h})\right)+4tL\epsilon+8L\epsilon\\
    & \hspace{2em} \le  4(11L+9)\left(\iota'(\delta_{t,h})\right)^2+ 12tL\epsilon +\frac{\sum_{s\in[t-1]}\left(f(z_s^{h})-\calT_2f'(z_s^{h})\right)^2}{2}\\
    & \hspace{2em} \le  4(11L+9)\left(\iota'(\delta_{t,h})\right)^2+ 16tL\epsilon +\frac{\sum_{s\in[t-1]}\left(f(z_s^{h})-\bar{f}[f'](z_s^{h})\right)^2}{2}.
\end{align*}

Consequently, by union bound over all choices of $f\in\calF^h$ and $f' = \min\left(1,f''+3\epsilon\right) = f_{t,2}^{h+1}$ where $f''\in\calF^{h+1}+\calW$, similar to~\Cref{lem:confidence-interval-overly} we have 
\begin{align*}
    \P(\calE_{t,h}) & \le \P\left(\sum_{s\in[t-1]} 2\left(\eta_s^{h}+\E[(r^h+f_{t,2}^{h+1}(x^{h+1}))^2|z_s^h]-\psi_{t}^h(z_s^{h})\right)\left(\hat{g}_t^h(z_s^{h})-\psi_{t}^h(z_s^{h})\right)>\right.\\
   & \quad\quad\quad \left. 4(11L+9)\left(\iota'(\delta_{t,h})\right)^2+ 16tL\epsilon +\frac{\sum_{s\in[t-1]}\left(\hat{g}_t^h(z_s^{h})-\psi_t^h(z_s^{h})\right)^2}{2}\right) \le \delta_{t,h},
\end{align*} and consequently, 
\[
\P\left(\psi_{t}^h\notin\calG^h_{t} \right)\le \delta_{t,h},
\]
which implies with probability $1-\delta_{t,h}$, $\psi_{t}^h\in\calG_{t}^h$ for any fixed given $t\in[T-1]$, $h\in[H]$ (the case $t=0$ holds with probability $1$ by definition).
\end{proof}

We also give the following consequence of~\Cref{lem:confidence-interval-RL-second} together with the definition of generalized Eluder dimension, which will be useful to justify our definition of $\sigma_t^h$ in~\Cref{eq:sigma-def-alg}.

\begin{lemma}\label{lem:RL-var-2} Conditioning on the good event $\bar{\calE}_t^h$, we have 
\[|\psi_t^h(z_t^h)-\hat{g}_t^h(z_t^h)|\le D_{\calF^h}(z_t^h; z_{[t-1]}^h, \1_{[t-1]}^h)\sqrt{\left(\bbeta_t^h\right)^2+\lambda}.\]
\end{lemma}

\begin{proof}
The proof is similar to that of~\Cref{lem:RL-var-1} so we omit here for brevity.
\end{proof}

\subsection{Validity of Variance Estimator}\label{app:variance}

In this section, we show that our variance over-estimate $\sigma_s^h$ at iteration $s$ is valid for all iterations afterwards $t\ge s$, and bound its difference with the true variance. We recall the definition of $f_{t,j}^h$ in~\Cref{alg:fitted-Q-simpler} such that $f_{t,1}^h(\cdot) \defeq \min\left(\hat{f}_{t,1}^h(\cdot)+b_{t,1}^h(\cdot)+\epsilon,1\right)$, $f_{t,2}^h(\cdot) \defeq \min\left(\hat{f}_{t,2}^h(\cdot)+2b_{t,1}^h(\cdot)+b_{t,2}^h(\cdot)+3\epsilon,2\right)$, $f_{t,-2}^h(\cdot) \defeq \max\left(\hat{f}_{t,-2}^h(\cdot)-b_{t,2}^h(\cdot)-\epsilon,0\right)$, and $\bar{f}_{t,j}^h$ as the conditional expectations. Abusing notation again, we use $\rV[\cdot|z_t^h]=\rV_{r^h, x^{h+1}}[\cdot|z_t^h]$ and $\E[\cdot|z_t^h]=\E_{r^h, x^{h+1}}[\cdot|z_t^h]$ where the randomness is taken with respect to $r^h$ and $x^{h+1}$ conditioning on $z_t^h$ when the meaning is clear from context. We also recall the definition of event $\calE_t^h = \{\bar{f}_{t,j}^h\in\calF_{t,j}^h~\text{for all}~j=1,\pm2~\text{and}~\psi_{t}^h\in\calG_t^h\}$ and $\calE_{\le t} = \cap_{s\le t, h\in[H]}\calE_s^h$.

First, we show that $f_{t,j}^h$ satisfies a pointwise monotonic relation conditioning on previous events. This is an important property that we need to satisfy for fulfilling the assumptions~\eqref{eq:CI-RL-conds-I} and~\eqref{eq:CI-RL-conds-II} required in~\Cref{lem:confidence-interval-RL}, and a reason that we design overly optimistic sequence using \emph{unweighted regression}.

We first restate \Cref{lem:mono} from the main text for convenience and then provide its full proof right after.

\lemmono*

\begin{proof}
We use induction to prove each inequality. Note that under the conditioning $\calE_{\le {t-1}}\cap\Par{\cap_{h'=h}^{H}\calE_{t}^{h'}}$, we have that $\bar{f}_{t,j}^{h'}\in \calF_{t,j}^{h'}$ for $j = 1,\pm2$ and all $h'\geq h$, by definition of the events.

For the first inequality, note that at step $t$ this holds trivially for $h'=H+1$. Now suppose this holds for some $h+1\le h'+1\le H+1$, i.e. we have $f_\star^{h'+1}(z^{h'+1})\le f_{t,1}^{h'+1}(z^{h'+1})$ for any $z^{h'+1}$ and thus $f_\star^{h'+1}(x^{h'+1})\le f_{t,1}^{h'+1}(x^{h'+1})$ for any $x^{h'+1}$. Then for level $h'$, we have conditioning on $\calE_t^{h'}$, for any $z^{h'}$, it holds that
\[
\hat{f}_{t,1}^{h'}(z^{h'})+b_{t,1}^{h'}(z^{h'})+\epsilon\ge \bar{f}_{t,1}^{h'}(z^{h'})+\epsilon\ge \E\left[r^{h'}+f_{t,1}^{h'+1}(x^{h'+1})|z^{h'}\right]\ge \E\left[r^{h'}+f_{\star}^{h'+1}(x^{h'+1})|z^{h'}\right] = f_\star^{h'}(z^{h'}),
\]
where the first inequality is due to the definition of bonus term $b_{t,1}^{h'}$ as in~\Cref{def:bonus-conditions} and conditioning event $\calE_t^{h'}$ so that $\bar{f}_{t,1}^{h'}(z^{h'}) \in \calF^{h'}_t$, the second inequality is due to definition of $\bar{f}_{t,1}^{h'}$ and the third inequality is due to induction.

Recall the definition of $f_{t,1}^{h'}(\cdot) \defeq \min\left(\hat{f}_{t,1}^{h'}(\cdot)+b_{t,1}^{h'}(\cdot)+\epsilon,1\right)$ in~\Cref{line:def-b} of~\Cref{alg:fitted-Q-simpler}, together with the upper bound of $1$ for $f_\star$ by the sparse reward assumption, we have consequently $f_{t,1}^{h'}(z^{h'})\ge f_\star^{h'}(z^{h'})$ for any $z^{h'}\in\calS\times\calA$ and any $h'\ge h$, which proves the inequality when $h'=h$.

For the second inequality, note it also holds trivially for $h'=H+1$. Now suppose this holds for some $h+1\le h'+1\le H+1$, i.e. we have $f_{t,-2}^{h'+1}(\cdot)\le f_{\star}^{h'+1}(\cdot)$. Then for level $h'$, conditioning on $\calE_{t}^{h'}$ we have for any $z^{h'}$, 
\begin{align*}
    \hat{f}^{h'}_{t,-2}(z^{h'})-b_{t,2}^{h'}(z^{h'})-2\epsilon
    & \stackrel{(i)}{\le} \bar{f}^{h'}_{t,-2}(z^{h'})-2\epsilon \stackrel{(ii)}{\le} \E[r^{h'}+f_{t,-2}^{h'+1}(x^{h'+1})|z^{h'}]-\epsilon\\
    & \stackrel{(iii)}{\le} \E[r^{h'}+f_\star^{h'+1}(x^{h'+1})|z^{h'}]-\epsilon \le f_{\star}^{h'}(z^{h'}).
\end{align*}
Here we use $(i)$ the definition of $b_{t,2}^{h'}$ and conditioning event of $\calE_{t}^{h'}$, $(ii)$ the definition of $\bar{f}_{t,-2}^{h'}$, and $(iii)$ the induction assumption. Recall the definition of
$f_{t,-2}^{h'}(\cdot) \defeq \max\left(\hat{f}_{t,-2}^{h'}(\cdot)-b_{t,2}^{h'}(\cdot)-\epsilon,0\right)$ in~\Cref{line:def-b-overly-pess} of~\Cref{alg:fitted-Q-simpler}, together with the fact that in above display $\text{RHS}\ge0$ always holds true by definition, we thus conclude by taking max with $0$ in above inequality that $f_{t,-2}^{h'}(z^{h'})\le f_{\star}^{h'}(z^{h'})$ for all $h'\ge h$ and specifically for $h'=h$.

For the third inequality given any fixed $s\le t$, note it also holds trivially for $h=H+1$. Now suppose this holds for some $h+1\le h'+1\le H+1$, i.e. we have $f_{s,2}^{h'+1}(z^{h'+1})\ge f_{t,1}^{h'+1}(z^{h'+1})$ for all $z^{h'+1}$, and consequently $f_{s,2}^{h'+1}(x^{h'+1})\ge f_{t,1}^{h'+1}(x^{h'+1})$ for all $x^{h'+1}$. Then for level $h'$, conditioning on $\calE_{s}^{h'}$ and $\calE_t^{h'}$ we have for any $z^{h'}$, it holds that 
\begin{equation}\label{eq:mono-helper}
\begin{aligned}
    \hat{f}^{h'}_{s,2}(z^h)+2b_{s,1}^{h'}(z^{h'})+b_{s,2}^{h'}(z^{h'})+3\epsilon
    & \stackrel{(i)}{\ge} \E[r^{h'}+f_{s,2}^{h'+1}(x^{h'+1})|z^{h'}]+2b_{s,1}^{h'}(z^{h'})+2\epsilon\\
    & \stackrel{(ii)}{\ge} \E[r^{h'}+f_{t,1}^{h'+1}(x^{h'+1})|z^{h'}]+2b_{t,1}^{h'}(z^{h'})+2\epsilon.
\end{aligned}
\end{equation}
Here we use $(i)$ the definition of $b_{s,2}$ so that $\hat{f}_{s,2}^{h'}(z^{h'})+b_{s,2}^{h'}(z^{h'})\ge \bar{f}_{s,2}^{h'}(z^{h'})$ conditioning on $\calE_{s}^{h'}$ and definition of $\bar{f}_{s,2}^{h'}(z^{h'})+\epsilon \ge \E[r^{h'}+f_{s,2}^{h'+1}(z^{h'+1})|z^{h'}]$, $(ii)$ the induction assumption together with the consistency condition on bonus thus that $b_{s,1}^{h'}(z^{h'})\ge b_{t,1}^{h'}(z^{h'})$. Recall definition $f_{s,2}^{h'}(\cdot) \defeq \min\left(\hat{f}_{s,2}^{h'}(\cdot)+2b_{s,1}^{h'}(\cdot)+b_{s,2}^{h'}(\cdot)+3\epsilon,2\right)$ in~\Cref{line:def-b-overly-opti} of~\Cref{alg:fitted-Q-simpler}, by taking min with $2$ on both sides of~\eqref{eq:mono-helper} and use non-negativity of $b_{t,1}$, we have $f_{s,2}^{h'}(z^{h'})\ge \calT f_{t,1}^{h'+1}(z^{h'})$. 

Additionally, we also have 
\begin{align*}
\E[r^{h'}+f_{t,1}^{h'+1}(x^{h'+1})|z^{h'}]+2b_{t,1}^{h'}(z^{h'})+2\epsilon	\ge  \hat{f}_{t,1}^{h'}(z^{h'})+b_{t,1}^{h'}(z^{h'})+\epsilon
\end{align*}
using $\hat{f}_{t,1}^{h'}(z^{h'})\le \E[r^{h'}+f_{t,1}^{h'+1}(x^{h'+1})|z^{h'}]+b_{t,1}^{h'}(z^{h'})+\epsilon$ due to definition of $b_{t,1}$ conditioning on $\calE_t^{h'}$. Now taking min with $2$ on both sides we  also obtain $f_{s,2}^{h'}(z^{h'})\ge f_{t,1}^{h'}(z^{h'})$  for all $h'\ge h$ to make the inductive argument. And thus the third inequality also holds when $h'=h$.

The inequalities for all $x^h$ is an immediate consequence of taking maximum over $a^h\in\calA$ for each inequality.
\end{proof}

Such point-wise monotonicity also allows us to prove the upper bound on constructed variance estimator $\sigma_t^h$ for each iteration $t\in[T]$ and level $h\in[H]$. We first restate the lemma from the main text below for convenience.

\lemvariancelower*

\begin{proof}
   Fix any $h\in[H]$, we first consider proving the stated inequality for 
   \[\left(\tilde{\sigma}_t^h\right)^2 \defeq \hat{g}_t^h(z_t^h)-\left(\hat{f}_{t,-2}^h(z_t^{h})\right)^2+D_{\calF^h}(z_t^h; z_{[t-1]}^h, \1_{[t-1]}^h)\cdot\left(\sqrt{\left(\bbeta_t^h\right)^2+\lambda}+2L\sqrt{\left(\beta_{t,2}^h\right)^2+\lambda}\right)+2(1+L)\epsilon .\] 
      
Conditioning on the good event $\calE_t$, by~\Cref{lem:RL-var-1} and~\Cref{lem:RL-var-2}, we know that 
    \begin{align} (\tilde{\sigma}_t^h)^2  & \ge \psi_t^h(z_t^h)-\left(\bar{f}_{t,-2}^h(z_t^h)\right)^2+(1+2L)\epsilon.\label{eq:variance-ineq-1}
    \end{align}
    Plugging $\left|\psi_t^h(z_t^h) - \E\left[(r^h+f_{t,1}^{h+1}(x^{h+1}))^2|z_t^h\right]\right|\le \epsilon$, $\left|\bar{f}_{t,-2}^h(z_t^h) - \E \left[r^h+f_{t,-2}^{h+1}(x^{h+1})|z_t^h\right]\right|\le \epsilon$ and $\bigl(\bar{f}_{t,-2}^h(z_t^h) + \E \left[r^h+f_{t,-2}^{h+1}(x^{h+1})|z_t^h\right]\bigr)\le (1+2L)$ into~\Cref{eq:variance-ineq-1}, we further have 
    \begin{align*} (\tilde{\sigma}_t^h)^2  \ge \E\left[(r^h+f_{t,1}^{h+1}(x^{h+1}))^2|z_t^h\right]-\biggl(\E \left[r^h+f_{t,-2}^{h+1}(x^{h+1})\right]\biggr)^2.
    \end{align*}
Now using the monotonic property  $f_{t,1}^{h+1}(\cdot)\ge f_{\star}^{h+1}(\cdot)\ge f_{t,-2}^{h+1}(\cdot) \geq 0$ conditioning on $\calE_{\le t}$, we have $\left(\tilde{\sigma}_t^h\right)^2\ge  \rV\left[r^h+f_{\star}^{h+1}(x^{h+1})|x_t^h, a_t^h\right]$.
  
 So far we have proven the stated inequality holds for $\left(\tilde{\sigma}_t^h\right)^2$. To show the inequality also holds true for $\left(\sigma_{t}^h\right)^2 = \min\left(4,\left(\tilde{\sigma}_{t}^h\right)^2\right)$, we take minimum with $4$ and note $\rV\left[r^h+f_{\star}^{h+1}(x^{h+1})|x_t^h, a_t^h\right]\le 4$ always holds true.
\end{proof}

The previous two lemmas on point-wise monotonicity and variance lower bound of $\Par{\sigma_t^h}^2$ immediately imply that the good event $\calE_{\le T}$ happens with high probability, following from an inductive argument.

\coroGoodEvent*

\begin{proof}
For any $t\ge 1$, conditioning on $\calE_{\le {t}}\cap\Par{\cap_{h'=h+1}^{H}\calE_{t+1}^{h'}}$, we first show the assumptions needed for step $t+1,h$ in~\Cref{lem:confidence-interval-RL} holds with probability $1$. 

The first assumption $\left(\sigma_s^h\right)^2 \ge\rV\left[r^h+f_{\star}^{h+1}(x^{h+1})|z_s^h\right]$  for all $s\in[t]$ in~\eqref{eq:CI-RL-conds-I} holds due to~\Cref{lem:RL-variance-bounds}. 

For second assumption in~\eqref{eq:CI-RL-conds-II}, it holds naively when $h=H$ since $f_{s,-2}^H\le f_{s,2}^H$ point-wise for all $s\in[t]$ using~\Cref{lem:mono}. When $h<H$ we have  $f_{\star}^{h+1}(x^{h+1})\le f_{t+1,1}^{h+1}(x^{h+1})$ for all $x^{h+1}$ conditioning on the event $\calE_{\le {t}}\cap\Par{\cap_{h'=h+1}^{H}\calE_{t+1}^{h'}}$ due to the first inequality of~\Cref{lem:mono}. Consequently, 
\[\rE\left[|f_{t+1,1}^{h+1}(x^{h+1})-f_{\star}^{h+1}(x^{h+1})|~|~z_s^h\right] = \rE\left[f_{t+1,1}^{h+1}(x^{h+1})-f_{\star}^{h+1}(x^{h+1})~|~z_s^h\right]=\calT f_{t+1,1}^{h+1}(z_s^h)-f_{\star}^{h}(z_s^h).\]
Now since we condition on $\calE_{\le {t}}\cap\Par{\cap_{h'=h+1}^{H}\calE_{t+1}^{h'}}$, the second inequality of~\Cref{lem:mono} implies that $f_{s,-2}^h(z_s^h)\le f_\star^h(z_s^h)$ and the third inequality of~\Cref{lem:mono} implies that $\calT f_{t+1,1}^{h+1}(z_s^h)\le f_{s,2}^h$ for all $s\in[t]$. Plugging these inequalities back we have  $\rE\left[|f_{t+1,1}^{h+1}(x^{h+1})-f_{\star}^{h+1}(x^{h+1})|~|~z_s^h\right]\le f_{s,2}^h(z_s^{h})-f_{s,-2}^h(z_s^{h})$ for all $s\in[t]$ and $z_s^h\in\calS\times\calA$. This also shows that the second assumption required in~\eqref{eq:CI-RL-conds-II} holds.

Thus we have shown conditioning on $\calE_{\le {t}}\cap\Par{\cap_{h'=h+1}^{H}\calE_{t+1}^{h'}}$, the event $\calE_{t+1}^h$ happens with probability $1-5\delta_{t+1,h}$ due to~\Cref{lem:confidence-interval-RL,lem:confidence-interval-overly,lem:confidence-interval-RL-pessimistic,lem:confidence-interval-RL-second}. Taking a union bound and note $\delta_{t,h}= \delta/(T+1)(H+1)$ we thus conclude that with probability $1-5\delta$ the good event $\calE_{\le T}$ happens.
\end{proof}

Next, we also provide an upper bound on the variance estimator $\Par{\sigma_t^h}^2$. It shows the estimator is not much bigger than the variance when taking greedy policy induced by optimistic function $f_{t,1}$.

\begin{restatable}[Upper bound of variance estimator]{lemma}{lemvariance}\label{lem:RL-variance-bounds-upper}
Suppose~\Cref{alg:fitted-Q-simpler} uses a consistent bonus oracle satisfying~\Cref{def:bonus-conditions}. For any step $t\ge2$ conditioning on the good event $\calE_{\le t}$, the variance we estimate $\sigma_t^h$ satisfies
\begin{align*}
\left(\sigma_t^h\right)^2 & \le \rV\left[r^h+f_{t,1}^{h+1}(x^{h+1})|x_t^h, a_t^h\right]+4\left(f_{t,2}^h(z_t^h)-f_{t,-2}^h(z_t^h)\right)\\
   & +4\min\left(1,D_{\calF^h}(z_t^h; z_{[t-1]}^h, \bsigma_{[t-1]}^h)\cdot\left(2\sqrt{\left(\bbeta_t^h\right)^2+\lambda}+4L\sqrt{\left(\beta_{t,2}^h\right)^2+\lambda}\right)\right)+4(2+L)\epsilon.
\end{align*}
\end{restatable}

\begin{proof}

Condioning on $\calE_{\le t}$, we have $\bar{f}_{t,-2}^h\in\calF_{t,2}^h$ and $\psi_t^h\in\calG_t^h$ due to~\Cref{lem:RL-var-1} and~\Cref{lem:RL-var-2}. Thus, by definition of bonus oracle and definition of $\sigma_t^h$ we have
\begin{align}    (\sigma_t^h)^2 & \le \psi_t^h(z_t^h)-\left(\bar{f}_{t,-2}^h(z_t^h)\right)^2 +D_{\calF^h}(z_t^h; z_{[t-1]}^h, \1_{[t-1]}^h)\cdot\left(2\sqrt{\left(\bbeta_t^h\right)^2+\lambda}+4L\sqrt{\left(\beta_{t,2}^h\right)^2+\lambda}\right)+2(1+L)\epsilon.\label{eq:variance-ineq-2}	
\end{align}

Recall $\left|\psi_t^h(z_t^h) - \E\left[(r^h+f_{t,1}^{h+1}(x^{h+1}))^2|z_t^h\right]\right|\le \epsilon$ and $\left|\bar{f}_{t,-2}^h(z_t^h) - \E \left[r^h+f_{t,-2}^{h+1}(x^{h+1})|z_t^h\right]\right|\le \epsilon$, \Cref{eq:variance-ineq-2} implies
   \begin{equation}\label{eq:var-bound-main}
    \begin{aligned}
    \left(\sigma_t^h\right)^2 \le & \E \left[(r^h+f_{t,1}^{h+1}(x^{h+1}))^2|z_t^h\right]- \left(\E \left[r^h+f_{t,-2}^{h+1}(x^{h+1})|z_t^h\right]\right)^2\\
    & +D_{\calF^h}(z_t^h; z_{[t-1]}^h, \1_{[t-1]}^h)\cdot\left(2\sqrt{\left(\bbeta_t^h\right)^2+\lambda}+4L\sqrt{\left(\beta_{t,2}^h\right)^2+\lambda}\right) + 4(1+L)\epsilon.
    \end{aligned}
    \end{equation}
    
Note we have 
    \begin{equation}\label{eq:var-bound-helper-2}
    \begin{aligned}
    & \left(\E\left[r^h+f_{t,1}^{h+1}(x^{h+1})|z_t^h\right]\right)^2-\left(\E\left[r^h+f_{t,-2}^{h+1}(x^{h+1})|z_t^h\right]\right)^2\\
    & \hspace{10em} \stackrel{(i)}{\le} 4\E\left[r^h+f_{t,1}^{h+1}(x^{h+1})-\left(r^h+f_{t,-2}^{h+1}(x^{h+1})\right)|z_t^h\right] \\
    &  \hspace{10em} \stackrel{(ii)}{\le} 4\left(\E\left[r^h+f_{t,1}^{h+1}(x^{h+1})|z_t^h\right]-f_{t,-2}^h(z_t^h)+\epsilon\right)\\
    & \hspace{10em} \stackrel{(iii)}{\le} 4\left(f_{t,2}^h(z_t^h)-f_{t,-2}^h(z_t^h)+\epsilon\right).
    \end{aligned}
    \end{equation}
Here we use $(i)$ the size bounds that $r^h, f_{t,1}^{h+1}, f_{t,-2}^{h+1}\in[0,1]$, $(ii)$ $f_{t,-2}^{h}(z_t^h)\le \bar{f}_{t,-2}^{h}(z_t^h)\le \calT f_{t,-2}^{h+1}+\epsilon$ due to the definition of $b_{t,2}$ and $\bar{f}_{t,-2}^{h} \in \calF_{t,-2}^h$ conditioning on $\calE_{\le t}$, and $(iii)$ the inequality that $f_{t,2}^h(z_t^h)\ge \calT f_{t,1}^{h+1}(z_t^h)$ conditioning on $\calE_{\le t}$ due to the third inequality in~\Cref{lem:mono}.
    
    Plugging~\eqref{eq:var-bound-helper-2} back to~\eqref{eq:var-bound-main}, we have  
    \begin{align*}
    \left(\sigma_{t}^h\right)^2 
    \le & \rV\left[r^h+f_{t,1}^{h+1}(x^{h+1})|z_t^h\right]+4\left(f_{t,2}^h(z_t^h)-f_{t,-2}^h(z_t^h)+\epsilon\right)\\
    & +D_{\calF^h}(z_t^h; z_{[t-1]}^h, \1_{[t-1]}^h)\cdot\left(2\sqrt{\left(\bbeta_t^h\right)^2+\lambda}+4L\sqrt{\left(\beta_{t,2}^h\right)^2+\lambda}\right)+4(1+L)\epsilon\\
    \le & \rV\left[r^h+f_{t,1}^{h+1}(x^{h+1})|z_t^h\right]+4\left(f_{t,2}^h(z_t^h)-f_{t,-2}^h(z_t^h)\right)\\
    & +D_{\calF^h}(z_t^h; z_{[t-1]}^h, \1_{[t-1]}^h)\cdot\left(2\sqrt{\left(\bbeta_t^h\right)^2+\lambda}+4L\sqrt{\left(\beta_{t,2}^h\right)^2+\lambda}\right)+4(2+L)\epsilon.
    \end{align*}
\end{proof}

\subsection{Approximation Error of Optimistic, Overly Optimistic(Pessimistic) Values}\label{app:approx-error}

In this section, we will provide a few inequalities for bounding the optimistic values, overly optimistic values, and overly pessimistic values sequence. We hope to show they will not deviate much from the expected value $V_t$ under exploration rule as defined in~\Cref{def:explore-V}, and thus not deviate much from $V_\star$ as well.

We again recall the definition of $f_{t,j}^h$ for $j=1,\pm2$ from $\eqref{eq:def-f-t-j}$ and also the use of $\calTo$ and $\calToo$ for disjoint set of iterations such that $[T] = \calTo\cup\calToo$ depending on whether $h_t\in[H]$ or not, as in~\eqref{def:calT}.

We also define the martingale difference sequence so that $\xi_{t,j}^h \defeq r^{h}_t+f_{t,j}^{h+1}(x^{h+1}_{t})-\E_{r^h, x^{h+1}}\left[r^h+f_{t,j}^{h+1}(x^{h+1})|z_t^h\right]$ for $j=1,-2,2$, and also $\xi_{t}^h \defeq r^{h}_t+V_t^{h+1}-\E\left[r^h+V_t^{h+1}|z_t^{[h]}, f_{t,1}^{[H]}, f_{t,2}^{[H]}\right]$. Abusing notation again, we use $\rV[\cdot|z_t^h]=\rV_{r^h, x^{h+1}}[\cdot|z_t^h]$ and $\E[\cdot|z_t^h]=\E_{r^h, x^{h+1}}[\cdot|z_t^h]$ where the randomness is taken with respect to $r^h$ and $x^{h+1}$ conditioning on $z_t^h$ when the meaning is clear from context.

For the overly pessimistic sequence $f_{t,-2}^h$, we first have the following guarantee on its lower bound.

\begin{lemma}\label{lem:RL-f-tilde-op}
Suppose~\Cref{alg:fitted-Q-simpler} uses a consistent bonus oracle satisfying~\Cref{def:bonus-conditions}. Conditioning on good event $\calE_{\le t}$, recall $\xi_{t,-2}^h = r^{h}_t+f_{t,-2}^{h+1}(x^{h+1}_{t})-\E_{r^h, x^{h+1}}\left[r^h+f_{t,-2}^{h+1}(x^{h+1})|z_t^h\right]$, and $\xi^h_t = r^{h}_t+V_t^{h+1}-\E\left[r^h+V_t^{h+1}|z_t^{[h]}, f_{t,1}^{[H]}, f_{t,2}^{[H]}\right]$. Then we have for any $t\in\calT$ and any $h\in[H]$, it holds that,
\begin{align*}
    f_{t,-2}^h(z_{t}^h)-V_t^h & \ge -2\sum_{h\le h'\le H}b^{h'}_{t,2}(z_{t}^{h'})+\sum_{h\le h'\le H}\left(\xi_t^{h'}-\xi_{t,-2}^{h'}\right)-2(H-h+1)\epsilon.
\end{align*}
\end{lemma}
\begin{proof}
We recall the definition of $\xi_t^h$ and $\xi_{t,-2}^h$. Similar to the previous lemma, for the base case we have  $f_{t,-2}^H(z_t^H)\ge \bar{f}_{t,-2}^H(z_t^H)-2b^H_{t,2}(z_t^H)-\epsilon \geq  V_t^H-2b^H_{t,2}(z_t^H)-2\epsilon$ using definition of $f_{t,-2}^H$ as in~\eqref{eq:def-f-t-j}. Now suppose the condition holds true for step $h+1$ where $2\le h+1\le H$. That is, it holds  $f_{t,-2}^{h+1}(z_t^{h+1})\ge V_t^{h+1}-2\sum_{h+1\le h'\le H}b^{h'}_{t,2}(z_{t}^{h'})+\sum_{h+1\le h'\le H}\left(\xi_{t}^{h'}-\xi_{t,-2}^{h'}\right)-2(H-h)\epsilon$. This implies that $f^{h+1}_{t,-2}(x_t^{h+1})\ge V_t^{h+1}-2\sum_{h+1\le h'\le H}b^{h'}_{t,2}(z_{t}^{h'})+\sum_{h+1\le h'\le H}\left(\xi_{t}^{h'}-\xi_{t,-2}^{h'}\right)-2(H-h)\epsilon$. Then for $z=z_t^h$ at level $h$, we have
\begin{align*}
 & f_{t,-2}^H(z_t^h)-V_t^h=  f^h_{t,-2}(z_t^h)-\bar{f}_{t,-2}^h(z_t^h)+\bar{f}_{t,-2}^h(z_t^h)-V_t^h\\
    & \hspace{2em} \stackrel{(i)}{\ge} -2b_{t,2}^{h}(z_t^{h})-2\epsilon+\E\left[r^h+f_{t,-2}^{h+1}(x^{h+1})-\left(r^h+V_t^{h+1}\right)|z_t^{[h]}, f_{t,1}^{[H]}, f_{t,2}^{[H]}\right]\\
    & \hspace{2em} \stackrel{(ii)}{=}-2b_{t,2}^{h}(z_{t}^{h})-2\epsilon+\xi_{t}^{h}-\xi_{t,-2}^h+f_{t,-2}^{h+1}(x^{h+1}_{t})-V_t^{h+1}\\
    & \hspace{2em}\stackrel{(iii)}{\ge}-2b_{t,2}^{h}(z_{t}^{h})+\xi_{t}^{h}-\xi_{t,-2}^h+\left(-2\sum_{h+1\le h'\le H}b^{h'}_{t,2}(z_{t}^{h'})+\sum_{h+1\le h'\le H}\xi_{t}^{h'}-\sum_{h+1\le h'\le H}\xi_{t,-2}^{h'}-2(H-h)\epsilon\right)-2\epsilon\\
    & \hspace{2em} = -2\sum_{h\le h'\le H}b^{h'}_{t,2}(z_{t}^{h'})+\sum_{h\le h'\le H}\left(\xi_{t}^{h'}-\xi_{t,-2}^{h'}\right)-2(H-h+1)\epsilon.
\end{align*}
Here we use $(i)$ the fact that $\bar{f}^h_{t,-2}\in\calF_{t,-2}^h$ by assumption and definition of $b_{t,2}^h$, $(ii)$ definition of $\xi_t^{h}$, $\xi_{t,-2}^{h}$, and $(iii)$ the recursion together with the definition that $f^{h+1}_{t,-2}(x^{h+1}_{t}) \ge f^{h+1}_{t,-2}(z_{t}^{h+1})$. 
\end{proof}

Next, we bound the optimistic sequence $f_{t,1}^h$ and the overly optimistic sequence $f_{t,2}^h$, depending on whether $t\in\calTo, h_t=H+1$ or $t\in\calToo, h_t\in[H]$.

\begin{lemma}\label{lem:RL-f-tilde-oo}
Suppose~\Cref{alg:fitted-Q-simpler} uses a consistent bonus oracle satisfying~\Cref{def:bonus-conditions}. Conditioning on good event $\calE_{\le t}$, recall $\xi_{t,2}^h = r^{h}_t+f_{t,2}^{h+1}(x^{h+1}_{t})-\E_{r^h, x^{h+1}}\left[r^h+f_{t,2}^{h+1}(x^{h+1})|z_t^h\right]$, $\xi^h_t = r^{h}_t+V_t^{h+1}-\E\left[r^h+V_t^{h+1}|z_t^{[h]},f_{t,1}^{[H]}, f_{t,2}^{[H]}\right]$. Then we have for any $h\ge h_t$,
\begin{align*}
    f_{t,2}^h(x_t^h)-V_t^h & \le  2\sum_{h\le h'\le H}b^{h'}_{t,1}(z_{t}^{h'})+2\sum_{h\le h'\le H}b^{h'}_{t,2}(z_{t}^{h'})+\sum_{h\le h'\le H}\left(\xi_{t}^{h'}-\xi_{t,2}^{h'}\right)+4(H-h+1)\epsilon.
\end{align*}
Further, recall $\xi_{t,1}^h = r^{h}_t+f_{t,1}^{h+1}(x^{h+1}_{t})-\E_{r^h, x^{h+1}}\left[r^h+f_{t,1}^{h+1}(x^{h+1})|z_t^h\right]$, for any $h\le h_t$ we have
\begin{align*}
    f_{t,1}(x_t^h)-V_t^h\le &  2\sum_{h\le h'\le H}b^{h'}_{t,1}(z_{t}^{h'})+2\sum_{h_t\le h'\le H}b^{h'}_{t,2}(z_{t}^{h'})\\
    & +\sum_{h_t\le h'\le H}\left(\xi_{t}^{h'}-\xi_{t,2}^{h'}\right)+\sum_{h\le h'<h_t}\left(\xi_{t}^{h'}-\xi_{t,1}^{h'}\right)+4(H-h+1)\epsilon.
\end{align*}
\end{lemma}

\begin{proof}
Conditioning on good event $\calE_t$, we prove this by math induction on $s=1,2,..,t$ and $h=H,\cdots,1$. 

We first recall the definition of $\xi_t^h$ and $\xi_{t,2}^h$. It is obvious that for $h=H$, we have $\hat{f}_{t,2}^H(z)+2b_{t,1}^H(z)+b_{t,2}^H(z)\le \bar{f}^H_{t,2}(z)+2b^H_{t,1}(z)+2b_{t,2}^H(z)\le \E[r^H|z]+2b^H_{t,1}(z)+2b_{t,2}^H(z)+\epsilon$ for any $z\in\calS\times\calA$. Using definition of $f_{t,2}^H$ as in~\eqref{eq:def-f-t-j}, this means in particular we have $f^H_{t,2}(x_t^H) = f^H_{t,2}(z_t^H) \le  V_t^H+2b^H_{t,1}(z_t^H)+2b_{t,2}^H(z_t^H)+4\epsilon$. Now suppose the condition holds true for step $h+1$ where $h_t+1\le h+1\le H$. That is, it holds that  $f_{t,2}^{h+1}(x_t^{h+1}) = f_{t,2}^{h+1}(z_t^{h+1})\le V_t^{h+1}+2\sum_{h+1\le h'\le H}b^{h'}_{t,1}(z_{t}^{h'})+2\sum_{h+1\le h'\le H}b^{h'}_{t,2}(z_{t}^{h'})+\sum_{h+1\le h'\le H}\left(\xi_{t}^{h'}-\xi_{t,2}^{h'}\right)$. 
Then for $z=z_t^h$ at level $h\ge h_t$, we have
\begin{align*}
 & f_{t,2}^h(z_t^h)-V_t^h=  f_{t,2}^h(z_t^h)-\bar{f}_{t,2}^{h}(z_t^h)+\bar{f}_{t,2}^{h}(z_t^h)-V_t^h\\
    & \hspace{2em} \stackrel{(i)}{\le} 2b_{t,1}^{h}(z_t^{h})+2b_{t,2}^h(z_t^h)+4\epsilon+\E\left[r^h+f_{t,2}^{h+1}(x^{h+1})-\left(r^h+V_t^{h+1}\right)|z_t^{[h]}, f_{t,1}^{[H]}, f_{t,2}^{[H]}\right]\\
        & \hspace{2em} \stackrel{(ii)}{=}2b_{t,1}^{h}(z_t^{h})+2b_{t,2}^h(z_t^h)+4\epsilon+\xi_{t}^{h}-\xi_{t,2}^h+f_{t,2}^{h+1}(x^{h+1}_{t})-V_t^{h+1}\\
    & \hspace{2em}\stackrel{(iii)}{\le}2b_{t,1}^{h}(z_{t}^{h})+2b_{t,2}^h(z_t^h)+\xi_{t}^{h}-\xi_{t,2}^h+\left(2\sum_{h+1\le h'\le H}b^{h'}_{t,1}(z_{t}^{h'})+2\sum_{h+1\le h'\le H}b^{h'}_{t,2}(z_{t}^{h'})\right.\\&\qquad\qquad\qquad\qquad\left.+\sum_{h+1\le h'\le H}\xi_{t}^{h'}-\sum_{h+1\le h'\le H}\xi_{t,2}^{h'}\right)+4(H-h+1)\epsilon\\
    & \hspace{2em} = 2\sum_{h\le h'\le H}b^{h'}_{t,1}(z_{t}^{h'})+2\sum_{h\le h'\le H}b^{h'}_{t,2}(z_{t}^{h'})+\sum_{h\le h'\le H}\left(\xi_{t}^{h'}-\xi_{t,2}^{h'}\right)+4(H-h+1)\epsilon.
\end{align*}
Here we use $(i)$ the fact that $\bar{f}_{t,2}^{h+1}\in\calF_{t,2}^h$ by assumption and definition of $b_t^h$, $(ii)$ definition of $\xi_t^{h}$, $\xi_{t,2}^h$, and $(iii)$ the recursion. By noting due to choice of greedy policy $f_{t,2}^h(z_t^h) = \max_{a^h}f_{t,2}^h(x_t^h,a^h) = f_{t,2}^h(x_t^h)$ for all $h\ge h_t$ concludes the final inequality.

For the second inequality, we note that by~\Cref{lem:mono} it holds that $f_{t,1}^h(\cdot)\le f_{t,2}^h(\cdot)$ point-wise and consequently $f_{t,1}^{h}(x_t^{h})\le f_{t,2}^{h}(x_t^{h})$ for $h=h_t\in[H]$, which implies when $h=h_t\in[H]$,
\begin{align*}
  f_{t,1}^h(x_t^h)-V_t^h\le  2\sum_{h\le h'\le H}b^{h'}_{t,1}(z_{t}^{h'})+2\sum_{h\le h'\le H}b^{h'}_{t,2}(z_{t}^{h'})+\sum_{h\le h'\le H}\left(\xi_{t}^{h'}-\xi_{t,2}^{h'}\right)+4(H-h+1)\epsilon.
  \end{align*}
  The above case also holds true when $h_t = H+1$.
 Thus, this shows the base case holds true when $h=h_t\in[H+1]$. Now suppose the inequality we want to show for $f^h_{t,1}$ holds for $h+1$ where $1\le h+1\le h_t$, then for $z=z_t^h$ at level $h$, we have
\begin{align*}
 & f_{t,1}^h(z_t^h)-V_t^h=  f_{t,1}^h(z_t^h)-\bar{f}^h_{t,1}(z_t^h)+\bar{f}^h_{t,1}(z_t^h)-V_t^h\\
    & \hspace{2em} \stackrel{(i)}{\le} 2b_{t,1}^{h}(z_t^{h})+2\epsilon+\E\left[r^h+f_{t,1}^{h+1}(x^{h+1})-\left(r^h+V_t^{h+1}\right)|z_t^{[h]}, f_{t,1}^{[H]}, f_{t,2}^{[H]}\right]\\
    & \hspace{2em} \stackrel{(ii)}{=}2b_{t,1}^{h}(z_{t}^{h})+2\epsilon+\xi_{t}^{h}-\xi_{t,1}^{h}+f_{t,1}^{h+1}(x^{h+1}_{t})-V_t^{h+1}\\
    & \hspace{2em}\stackrel{(iii)}{\le}2b_{t,1}^{h}(z_{t}^{h})+\xi_{t}^{h}-\xi_{t,1}^{h}+2\sum_{h+1\le h'\le H}b^{h'}_{t,1}(z_{t}^{h'})+2\sum_{h_t\le h'\le H}b^{h'}_{t,2}(z_{t}^{h'}) \\
    & \hspace{5em} +\sum_{h_t\le h'\le H}\left(\xi_{t}^{h'}-\xi_{t,2}^{h'}\right)+\sum_{h+1\le h'<h_t}\left(\xi_{t}^{h'}-\xi_{t,1}^{h'}\right) +4(H-h+1)\epsilon\\
    & \hspace{2em} \le  2\sum_{h\le h'\le H}b^{h'}_{t,1}(z_{t}^{h'})+2\sum_{h_t\le h'\le H}b^{h'}_{t,2}(z_{t}^{h'}) +\sum_{h_t\le h'\le H}\left(\xi_{t}^{h'}-\xi_{t,2}^{h'}\right)+\sum_{h\le h'<h_t}\left(\xi_{t}^{h'}-\xi_{t,1}^{h'}\right)+4(H-h+1)\epsilon.
\end{align*}
Here we use $(i)$ the fact that $\bar{f}^h_{t,1}\in\calF_{t,1}^h$ by assumption and definition of $b_t^h$, $(ii)$ definition of $\xi_t^{h}$, $\xi_{t,1}^{h}$, and $(iii)$ the recursion. By noting due to choice of greedy policy $f_{t,1}^h(z_t^h) = \max_{a^h}f_{t,1}^h(x_t^h,a^h) = f_{t,1}^h(x_t^h)$ concludes the final inequality.
\end{proof}

\subsection{Bounding the Regret in Expectation}\label{app:regret}

From now on we will denote the good event $\calE_{\le T}$, which by~\Cref{coro:good-event-fbar} happens with probability $1-5\delta$.

When $\calE_{\le T}$ happens, the regret can be expressed as 
\begin{equation}\label{eq:regret-prelim}
\begin{aligned}
      R_T & = \sum_{t\in[T]}\left(f^1_\star(x_t^1)-V_t^1\right)  \le O(1)+\sum_{2\le t \le T}\left(f^1_{t,1}(x_t^1)-V_t^1\right)\\
    & \le O(1+TH\epsilon)+ 2\sum_{ t\in \calTo} \min\left(1+L,\sum_{h\in[H]}b_{t,1}^h(z_{t}^h)\right)+\sum_{ t\in\calTo}\sum_{h\in[H]}\left(\xi_{t}^{h}-\xi_{t,1}^h\right)\\
    & \hspace{3em} +2\sum_{ t\in \calToo} \min\left(1+L,\sum_{h\in[H]}b_{t,1}^h(z_{t}^h)\right)+2\sum_{ t\in \calToo} \min\left(1+L,\sum_{h_t\le h\le H}b_{t,2}^h(z_{t}^h)\right)\\
    & \hspace{3em} +\sum_{ t\in\calToo}\left(\sum_{1\le h<h_t}\left(\xi_{t}^{h}-\xi_{t,1}^h\right)+\sum_{h_t\le h\le H}\left(\xi_{t}^{h}-\xi_{t,2}^h\right)\right)\\
    & \le O(1+TH\epsilon)+ 2\underbrace{\sum_{ t\in [T]}\sum_{h\in[H]} \min\left(1+L, b_{t,1}^h(z_{t}^h)\right)}_{I}+2\underbrace{\sum_{ t\in \calToo}\sum_{h\in[H]} \min\left(1+L, b_{t,2}^h(z_{t}^h)\right)}_{III}\\
    & \hspace{3em} +\left[\sum_{t\in[T],h\in[H]}\xi_t^h-\sum_{ t\in[T]}\sum_{h\in[H]}\xi_{t,1}+\sum_{ t\in\calToo}\left(\sum_{h_t\le h\le H}\xi_{t,1}^h-\sum_{h_t\le h\le H}\xi_{t,2}^h\right)\right].
\end{aligned}
\end{equation}

We again recall a few notations that we heavily use throughout this section.
\begin{align*}
    I & \defeq \sum_{t\in[T]}\sum_{h\in[H]}\min\left(1+L,b_{t,1}^h(z_{t}^h)\right), \\
    II & \defeq \sum_{t\in[T]}\sum_{h\in[H]}\min\left(1+L,b_{t,2}^h(z_{t}^h)\right),\\
    III & \defeq \sum_{t\in\calToo}\sum_{h\in[H]}\min\left(1+L,b_{t,2}^h(z_{t}^h)\right).
\end{align*}

We also recall the following notations
\begin{align*}
    \calToo & \defeq\{t\in[T]:\text{there exists some}~h\in[H]~\text{that exploration is guided by}~f_{t,2}^h\},\\
    \sum_{t\in[T], h\in[H]}\xi_{t}^{h} & \defeq \left(\sum_{t\in[T], h\in[H]}r^{h}_t+V_t^{h+1}-\E\left[r^h+V_t^{h+1}|z_t^{[h]}, f_{t,1}^{[H]}, f_{t,2}^{[H]}\right]\right),\\
     \sum_{t\in[T], h\in[H]}\xi_{t,1}^{h} & \defeq \sum_{t\in[T], h\in[H]}\left(r^h_t+f_{t,1}^{h+1}(x^{h+1}_t) -\rE_{r^h, x^{h+1}}\left[r^h+f_{t,1}^{h+1}(x^{h+1}) |z_t^h\right]\right),\\
      \sum_{t\in[T], h\in[H]}\xi_{t,\pm2}^{h} & \defeq \sum_{t\in[T], h\in[H]}\left(r^h_t+f_{t,\pm2}^{h+1}(x^{h+1}_t) -\rE_{r^h, x^{h+1}}\left[r^h+f_{t,\pm2}^{h+1}(x^{h+1}) |z_t^h\right]\right),\\
      d_\alpha & \defeq H^{-1}\Par{\sum_{h\in[H]}\dim_{\alpha,T}(\calF^h)}.
\end{align*}
The last equation defining $d_\alpha$ will be used mainly for notational simplicity.

The next is a simple fact about the defined $\xi_{t}^h$ and $\xi_{t,j}^h$, which will be useful multiple times in the later analysis. The fact builds on the observation that $a_t^h$ is determined solely by the filtration of $\calH_t^{h-1}=\sigma(x_1^{1},r_1^1, x_1^2,\cdots, r^{H}_1,x^{H+1}_1; x_2^{1},r_2^1, x_2^2,\cdots, r^{H}_2,x^{H+1}_2;\cdots,x_t^1,r_t^1,\cdots, r_t^{h-1},x_t^{h})$, due to policy exploration rule~\eqref{eq:greedy-policy}.

\begin{lemma}\label{fact-xi}
By definition $\xi_{t}^h$ and $\xi_{t,j}^h$ for $j=1,\pm2$ are all adapted to filtration $\calH_t^{h-1}$. They are martingale difference sequence (MDS) satisfying $\E[\xi_t^h|\calH_t^{h-1}]=0$ and $\E[\xi_{t,j}^h|\calH_t^{h-1}]=0$. As an immediate consequence, we have for all $j=1,\pm2$,
\begin{align*}
    \E[\sum_{t\in[T]}\sum_{h\in[H]}\xi_{t}^h] & =\E[\sum_{t\in[T]}\sum_{h\in[H]}\xi_{t,j}^h]=0,\\
    \text{and}~\E[\sum_{t\in\calToo}\sum_{h_t\le h\le H}\xi_{t}^h] & =\E[\sum_{t\in\calToo}\sum_{h_t\le h\le H}\xi_{t,
j}^h]=0.
\end{align*}
\end{lemma}
\begin{proof}
We only prove the last equation for $\xi_t^h$. Note we can write 
\[
\E\left[\sum_{t\in\calToo}\sum_{h_t\le h\le H}\xi_{t}^h\right]= \E\left[\sum_{t\in[T]}\sum_{1\le h\le H}\E[\1_{\{h_t\ge h\}}\xi_{t}^h|\calH_t^{h-1}]\right] \stackrel{(\star)}= \E\left[\sum_{t\in[T]}\sum_{1\le h\le H}\1_{\{h_t\ge h\}}\E[\xi_{t}^h|\calH_t^{h-1}]\right] = 0.
\]
Here for equation $(\star)$ we use the fact that random variable $\1_{\{h_t\ge h\}}$ is adapted to $\calH_t^{h-1}$. Similar for the proof of $\xi_{t,j}^h$.
\end{proof}

Now first of all, in the analysis we bound term $II$. To do so we rely on the definition of bonus oracle and the assumption of bounded generalized Eluder dimension.

\begin{lemma}[Crude bound on $II$]\label{lem:bound-E-III}
Given $b_{t,2}(\cdot)\le C\cdot \biggl(D_{\calF^h}(\cdot; z_{[t-1]}^h, \1_{[t-1]}^h)\sqrt{\left(\beta_{t,2}^h\right)^2+\lambda}+\epsc \cdot \beta_{t,2}^h\biggr)$, when $\lambda = \Theta(1)$, $\alpha\le 1$, we have for a subset $\calT\subseteq [T]$ the following inequality holds true:
\begin{align*}
    \sum_{t\in\calT}\sum_{h\in[H]}\min\left(1+L, b_{t,2}^h(z_t^h)\right) = O\left(\sqrt{\log\frac{\calN \calN_bTH}{\delta}+T\epsilon}\cdot\left(H\cdot d_\alpha+H\sqrt{|\calT|\cdot d_\alpha}+|\calT|H\epsc\right)\right),
\end{align*}
This immediately implies that
\begin{align*}
II & \defeq \sum_{t\in[T]}\sum_{h\in[H]}\min\left(1+L, b_{t,2}^h(z_t^h)\right) = O\left(\sqrt{\log\frac{\calN \calN_bTH}{\delta}+T\epsilon}\cdot\left(H\cdot d_\alpha+H\sqrt{T\cdot d_\alpha}+TH\epsc\right)\right).
\end{align*}
\end{lemma}

\begin{proof}
We first note that by assumption of $b_{t,2}$, 
\begin{align*}
     & \sum_{t\in\calT}\sum_{h\in[H]}\min\left(1+L,b_{t,2}^h(z_t^h) \right) \\
     & \hspace{6em} \stackrel{(i)}{=} O\left(\sum_{t\in\calT}\sum_{h\in[H]}\min\left(1,D_{\calF^h}(z_t^h; z_{[t-1]}^h, \1_{[t-1]})\cdot\sqrt{\left(\beta_{t,2}^h\right)^2+\lambda} \right)+|\calT|H\epsc\cdot\max_{t\in[T],h\in[H]}\beta_{t,2}^h\right) \\
    & \hspace{6em} \stackrel{(ii)}{=} O\left(\sqrt{\log\frac{\calN \calN_b TH}{\delta}+T\epsilon}\cdot \left(\sum_{t\in\calT}\sum_{h\in[H]}\min\left(1,D_{\calF^h}(z_t^h; z_{[t-1]}^h, \1_{[t-1]})\right)+|\calT|H\epsc\right)\right).
\end{align*}
Here we use $(i)$ the assumption on $b_{t,2}^h$ and $(ii)$ the definition of $\beta_{t,2}^h$ as in~\Cref{eq:def-beta-overly}.

Now we divide the indices of $(t,h)\in\calT\times[H]$ in two cases:
\begin{align*}
\calI_1 = \{(t,h)\in\calT\times[H]~|~D_{\calF^h}(z_t^h; z_{[t-1]}^h, \1_{[t-1]})\ge 1\},\\
\calI_2 = \{(t,h)\in\calT\times[H]~|~D_{\calF^h}(z_t^h; z_{[t-1]}^h, \1_{[t-1]})< 1\}.
\end{align*}
We then consider the summation of terms respectively, note
\[
\sum_{(t,h)\in\calI_1}\min\left(1,D_{\calF^h}(z_t^h; z_{[t-1]}^h, \1_{[t-1]})\right)\le \sum_{(t,h)\in\calI_1}D_{\calF^h}^2(z_t^h; z_{[t-1]}^h, \1_{[t-1]}) \le H\cdot d_\alpha,
\]
where the last inequality holds for any $\alpha\le 1$.

Also using  Cauchy-Schwarz inequality,
\begin{align*}
\sum_{(t,h)\in\calI_2}\min\left(1,D_{\calF^h}(z_t^h; z_{[t-1]}^h, \1_{[t-1]})\right) & \le \sqrt{\sum_{(t,h)\in\calI_2}1^2}\cdot\sqrt{\sum_{(t,h)\in\calI_2}D^2_{\calF^h}(z_t^h; z_{[t-1]}^h, \1_{[t-1]})} \le H\sqrt{|\calT|\cdot d_\alpha},
\end{align*} 
where the last inequality holds again for any $\alpha\le 1$.

Combining two terms together,
\begin{align}
\sum_{t\in\calT}\sum_{h\in[H]}\min(1,D_{\calF^h}(z_t^h; z_{[t-1]}^h, \1_{[t-1]}))\le O\left(H\cdot d_\alpha+H\sqrt{|\calT|\cdot d_\alpha}\right).\label{eq:bound-b-2-sum}
\end{align}
\end{proof}

Using the same idea we could get a similar crude bound for term $I$, formally as follows:

\begin{lemma}[Crude bound on $I$]\label{lem:bound-E-III-o}
Given $b_{t,1}(\cdot)\le C\bigl(\cdot D_{\calF^h}(\cdot; z_{[t-1]}^h, \bsigma_{[t-1]}^h)\sqrt{\left(\beta_{t,1}^h\right)^2+\lambda}+\epsc\cdot \beta_{t,1}^h\bigr)$, when $\lambda = \Theta(1)$,  $\alpha\le 1$, we have for a subset $\calT\subseteq [T]$ the following inequality holds true:
\begin{align*}
 & \sum_{t\in\calT}\sum_{h\in[H]}\min\left(1+L, b_{t,1}^h(z_t^h)\right)\\
 & \hspace{6em} = O\left(\sqrt{\log\frac{\calN TH}{\alpha\delta}+\frac{T}{\alpha^2}\epsilon}\cdot\left(\sqrt{\log\frac{\calN\calN_bTH}{\alpha\delta}}\cdot H\sqrt{|\calT|\cdot d_\alpha}+\log\frac{\calN\calN_bTH}{\alpha\delta}\cdot H d_\alpha+|\calT|H\epsc\right)\right).
\end{align*}
\end{lemma}

\begin{proof}
We first note that by assumption of $b_{t,1}$, 
\begin{align*}
    & \sum_{t\in\calT}\sum_{h\in[H]}\min\left(1+L,b_{t,1}^h(z_t^h) \right)\\
    & \hspace{5em} \stackrel{(i)}{=} O\left(\sum_{t\in\calT}\sum_{h\in[H]}\min\left(1,D_{\calF^h}(z_t^h; z_{[t-1]}^h, \bsigma_{[t-1]}^h)\cdot\sqrt{\left(\beta_{t,1}^h\right)^2+\lambda} \right)+ |\calT|H\epsc\cdot\max_{t\in\calT,h\in[H]}\beta_{t,1}^h\right) \\
    & \hspace{5em} \stackrel{(ii)}{=} O\left(\sqrt{\log\frac{\calN TH}{\alpha\delta}+\frac{T}{\alpha^2}\epsilon}\cdot\left( \sum_{t\in\calT}\sum_{h\in[H]}\min\left(1,\bsigma_t^h \cdot \Par{\bsigma_t^h}^{-1} D_{\calF^h}(z_t^h; z_{[t-1]}^h, \bsigma_{[t-1]}^h)\right)+|\calT|H\epsc\right)\right).
\end{align*}
Here we use $(i)$ the assumption on $b_{t,1}^h$ and $(ii)$ the definition of $\beta_{t,1}^h$ as in~\Cref{eq:def-beta-opti}.

Now we divide the indices of $(t,h)\in\calT\times[H]$ in the following cases similarly to the previous proof:
\begin{align*}
\calI_1 &= \{(t,h)\in\calT\times[H]~|~\Par{\bsigma_t^h}^{-1}D_{\calF^h}(z_t^h; z_{[t-1]}^h, \bsigma_{[t-1]}^h)\ge 1\},\\
\calI_2 &= \{(t,h)\in\calT\times[H]~|~(t,h)\notin\calI_1,~\bsigma_t^h = \alpha\},\\
\calI_3 &= \left\{(t,h)\in\calT\times[H]~|~(t,h)\notin\calI_1,~\bsigma_t^h = 2\Par{\sqrt{\upsilon(\delta_{t,h})}+\iota(\delta_{t,h})}\cdot\sqrt{ D_{\calF^h}(z^{h}_t; z^{h}_{[t-1]},\bsigma^{h}_{[t-1]})}\right\},\\
\calI_4 &= \left\{(t,h)\in\calT\times[H]~|~(t,h)\notin\calI_1,~\bsigma_t^h = \sigma_t^h\right\},\\
\calI_5 &= \left\{(t,h)\in\calT\times[H]~|~(t,h)\notin\calI_1,~\bsigma_t^h = \sqrt{2}\iota(\delta_{t,h})\sqrt{f^h_{t,2}(z_t^h)-f^h_{t,-2}(z_t^h)}\right\}.
\end{align*}
We then consider the summation of terms respectively, for $\calI_1$ we have
\begin{equation}\label{eq:crude-1}
\begin{aligned}
\sum_{(t,h)\in\calI_1}\min\left(1,\bsigma_t^h \cdot \Par{\bsigma_t^h}^{-1}D_{\calF^h}(z_t^h; z_{[t-1]}^h, \bsigma_{[t-1]}^h)\right) & \le \sum_{(t,h)\in\calI_1}\Par{\bsigma_t^h}^{-2}D^2_{\calF^h}(z_t^h; z_{[t-1]}^h, \bsigma^h_{[t-1]})\\
& \le \sum_{h\in[H]}\dim_{\alpha,T}(\calF^h).
\end{aligned}
\end{equation}
For $\calI_2$ we use Cauchy-Schwarz inequality to get 
\begin{equation}\label{eq:crude-2}
\begin{aligned}
\sum_{(t,h)\in\calI_2}\min\left(1,\bsigma_t^h \cdot \Par{\bsigma_t^h}^{-1}D_{\calF^h}(z_t^h; z_{[t-1]}^h, \bsigma_{[t-1]}^h)\right) & \le \sqrt{\alpha^2TH}\cdot\sqrt{\sum_{(t,h)\in\calI_2}\Par{\bsigma_t^h}^{-2}D^2_{\calF^h}(z_t^h; z_{[t-1]}^h, \bsigma^h_{[t-1]})}\\
& \le \sqrt{\sum_{h\in[H]}\dim_{\alpha,T}(\calF^h)}.
\end{aligned}
\end{equation}
For $\calI_3$ we have 
\begin{equation}\label{eq:crude-3}
\begin{aligned}
         & \sum_{(t,h)\in\calI_3}\min\left(1,\bsigma_{t}^h\cdot \Par{\bsigma_t^h}^{-1}D_{\calF^h}(z_{t}^h;z_{[t-1]}^h,\bsigma_{[t-1]}^h)\right) \\
         & \hspace{6em} \stackrel{(i)}{\le} \sum_{(t,h)\in\calI_3}\Par{8\upsilon(\delta_{t,h})+\iota^2(\delta_{t,h})}\cdot\min\left(1,\Par{\bsigma_{t}^h}^{-2}D^2_{\calF^h}(z_{t}^h;z_{[t-1]}^h,\bsigma_{[t-1]}^h)\right) 
         \\
        & \hspace{6em} \le O\left(\Par{\sqrt{\log\frac{\calN TH}{\alpha\delta}}+\log\frac{\calN\calN_b TH}{\alpha\delta} }\cdot\sum_{h\in[H]}\dim_{\alpha,T}(\calF^h)\right).
\end{aligned}
\end{equation}
Here for inequality $(i)$ we use the choice of $\bsigma_t^h$ so that $\bsigma_t^h= 2\Par{\sqrt{\upsilon(\delta_{t,h})}+\iota(\delta_{t,h})}\sqrt{D_{\calF^h}(z_t^h; z_{[t-1]}^h,\bsigma_{[t-1]}^h)}$, dividing both sides by
    $\sqrt{\bsigma_t^h}$ and rearranging gives  $\bsigma_t^h\le 8\Par{\upsilon(\delta_{t,h})+\iota^2(\delta_{t,h})}\Par{\bsigma_t^h}^{-1}D_{\calF^h}(z_t^h; z_{[t-1]}^h,\bsigma_{[t-1]}^h)$, and also the property that $\Par{\bsigma_t^h}^{-1}D_{\calF^h}(z_t^h; z_{[t-1]}^h,\bsigma_{[t-1]}^h)\le 1$ when $(t,h)\in\calI_3$ due to definition of $\calI_3$.
  
For $\calI_4$ we use Cauchy-Schwarz inequality and upper bound $\bsigma_t^h=\sigma_t^h= O(1)$ to get 
\begin{align*}
\sum_{(t,h)\in\calI_4}\min\left(1,\bsigma_t^h\cdot \Par{\bsigma_t^h}^{-1} D_{\calF^h}(z_t^h; z_{[t-1]}^h,\bsigma^h_{[t-1]})\right) & \le \sqrt{\sum_{(t,h)\in\calI_4}\left(\bsigma_t^h\right)^2}\cdot\sqrt{\sum_{(t,h)\in\calI_4}\Par{\bsigma_t^h}^{-2}D^2_{\calF^h}(z_t^h; z_{[t-1]}^h, \bsigma^h_{[t-1]})}\\
& \le \sqrt{|\calT|H\cdot\left(\sum_{h\in[H]}\dim_{\alpha,T}(\calF^h)\right)}.
\end{align*} 

For $\calI_5$ we use Cauchy-Schwarz inequality to get
\begin{align*}
& \sum_{(t,h)\in\calI_3}\min\left(1,\bsigma_t^h\cdot \Par{\bsigma_t^h}^{-1} D_{\calF^h}(z_t^h; z_{[t-1]}^h,\bsigma^h_{[t-1]})\right)\\
& \hspace{5em} \le O\Par{ \sqrt{\sum_{(t,h)\in\calI_3}\iota^2(\delta_{t,h})}\cdot\sqrt{\sum_{(t,h)\in\calI_2}\Par{\bsigma_t^h}^{-2}D^2_{\calF^h}(z_t^h; z_{[t-1]}^h, \bsigma^h_{[t-1]})}}\\
& \hspace{5em} \le O\Par{\sqrt{\log\frac{\calN\calN_bTH}{\alpha\delta}}\cdot\sqrt{|\calT|H\cdot\left(\sum_{h\in[H]}\dim_{\alpha,T}(\calF^h)\right)}},
\end{align*} 
where we use $f_{t,2}^h-f_{t,-2}^h=O(1)$ for the first inequality and the definition of $\iota(\delta_{t,h})$ for the second inequality.

Summing all terms above together we have
\begin{align*}
\sum_{t\in\calT}\sum_{h\in[H]}\min(1,\bsigma_t^h \Par{\bsigma_t^h}^{-1}D_{\calF^h}(z_t^h; z_{[t-1]}^h, \bsigma^h_{[t-1]})) \le O\left(\log\frac{\calN\calN_b TH}{\alpha\delta}\cdot H\cdot d_\alpha+\sqrt{\log\frac{\calN\calN_b TH}{\alpha\delta}}\cdot H\sqrt{|\calT|d_\alpha}\right).
\end{align*}
\end{proof}

\begin{corollary}[Corollary from adapted version using LTV]\label{coro:sum-variance}
 Recall the filtration definition
 \[\calH_{t-1}^H=\sigma(x_1^{1},r_1^1, x_1^2,\cdots, r^{H}_1,x^{H+1}_1; x_2^{1},r_2^1, x_2^2,\cdots, r^{H}_2,x^{H+1}_2;\cdots,x_{t-1}^1,r_{t-1}^1,\cdots, r_{t-1}^{H},x_{t-1}^{H+1}).\] Also we use $\rE[\cdot|z_t^h]=\rE_{r^h, x^{h+1}}[\cdot|z_t^h]$ and $\rV[\cdot|z_t^h]=\rV_{r^h, x^{h+1}}[\cdot|z_t^h]$ where the expectation is only taken over $r^h$ and $x^{h+1}$ due to model transition for shorthand. When $L=O(1)$ we have 
\begin{align*}
& \rE \left[\sum_{h=1}^H \rV \left[r^h+f_{t,1}^{h+1}(x^{h+1})| z_t^{h}\right]~|~\calH_{t-1}^H\right]\\
& \hspace{1em} \leq O\Par{1+H^2\delta+H^2\cdot \E\left[\1_{\{t\in\calToo\}}~|~\calH_{t-1}^H\right]+H\cdot \E\left[\sum_{h\in[H]}\Par{f_{t,2}^h(z_t^h)-f_{t,-2}^h(z_t^h)}~|~\calH_{t-1}^H\right]}.
\end{align*}

Consequently, we have 
\begin{align*}
   & \rE\left[ \sum_{t\in[T]}\sum_{h\in[H]} \rV \left[r^h+f_{t,1}^{h+1}(x^{h+1})| z_t^{h}\right]\right]\\
   & \hspace{3em} \le  O\left(T+TH^2\delta+H^2\rE\left[|\calToo|\right]+H\cdot\E\sum_{t\in[T]}\sum_{h\in[H]}\Par{f_{t,2}^h(z_t^h)-f_{t,-2}^h(z_t^h)}\right).
\end{align*}
\end{corollary}

\begin{proof}
For the first inequality, applying~\Cref{prop:LTV-2} we get
\begin{equation*}\label{eq:coro-LTV-helper}
\begin{aligned}
& \rE \left[\sum_{h=1}^H \rV \biggl[r^h+f_{t,1}^{h+1}(x^{h+1})| z_t^{h}\biggr]~|~\calH_{t-1}^H\right]\\
& \hspace{1em} \le 2\rE\left[ \left(\sum_{h=1}^H r_t^h - f_{t,1}^1(x_t^1)\right)^2~|~\calH_{t-1}^H\right] + 2\rE\left[\left(\1_{\{t\in\calTo\}}\sum_{h=1}^H \left(f_{t,1}^{h}(z_t^h) - \rE [r^h+f_{t,1}^{h+1}(x^{h+1})|z_t^{h}]\right)\right)^2~|~\calH_{t-1}^H\right]\\
& \hspace{5em} +2\rE\left[\left(\1_{\{t\in\calToo\}}\sum_{h=1}^H \left(f_{t,1}^{h}(x_t^h) - \rE [r^h+f_{t,1}^{h+1}(x^{h+1})|z_t^{h}]\right)\right)^2|~\calH_{t-1}^H\right]\\
& \hspace{1em} \le O(1)+O\left(H\right)\cdot \E\left[\biggl|\1_{\{t\in\calTo\}}\sum_{h\in[H]}\left(f_{t,1}^{h}(z_t^h) - \rE [r^h+f_{t,1}^{h+1}(x^{h+1})|z_t^{h}]\right)\biggr|~|~\calH_{t-1}^H\right]\\
& \hspace{5em}+ O(H^2)\cdot\E\left[\1_{\{t\in\calToo\}}~|~\calH_{t-1}^H\right].
\end{aligned}
\end{equation*}
Here for the last inequality we use boundedness of $|f_{t,1}^h|\le 1$, so that $\sum_{h=1}^H \left(f_{t,1}^{h}(x_t^h) - \rE [r^h+f_{t,1}^{h+1}(x^{h+1})|z_t^h]\right)\le O(H)$.

Note  we can bound $f_{t,1}^{h}(z_t^h) - \rE [r^h+f_{t,1}^{h+1}(x^{h+1})|z_t^{h}]\le f_{t,2}^h(z_t^h)-f_{t,-2}^h(z_t^h)$ conditioning on $\calE_{\le T}$ due to~\Cref{lem:mono}, plugging this back into the above inequality  we have
\begin{align*}
& \rE \left[\sum_{h=1}^H \rV [r^h+f_{t,1}^{h+1}(x^{h+1})| z_t^{h}]~|~\calH_{t-1}^H\right]\\
& \le  O\Par{1+H^2\delta+H^2\cdot \E\left[\1_{\{t\in\calToo\}}~|~\calH_{t-1}^H\right]+H\cdot \E\left[\sum_{h\in[H]}\Par{f_{t,2}^h(z_t^h)-f_{t,-2}^h(z_t^h)}~|~\calH_{t-1}^H\right]}.
\end{align*}

The second inequality is an immediate consequence of this corollary together with definition of $\calToo$.
\end{proof}

Following the previous expression of regret in~\eqref{eq:regret-prelim} and using bound in~\Cref{fact-xi}, we have when $\delta\le 1/6$,
\begin{equation}\label{eq:regret-expected}
\begin{aligned}
\E R_T & = \E\left[\1(\calE_{\le T}) \E[R_T|\calE_{\le T}] + \1(\text{not}~\calE_{\le T})	\E[R_T|~\text{not}~\calE_{\le T}]\right] \\
& \le O( TH\delta) + (1-5\delta) \E\left[O\left(1+TH\epsilon\right)+2\cdot I+2\cdot III~|~\calE_{\le T}\right]\\\
& \le O\left(TH\delta +TH\epsilon+1\right)+2\E[I|\calE_{\le T}]+2\E[III|\calE_{\le T}].
\end{aligned}
\end{equation}

\begin{lemma}[Bounding size of $\calToo$]\label{lem:size-of-Too}
Suppose $\alpha\le 1$, we set 
\[u_t \ge C\cdot\Par{\frac{\sqrt{\log\frac{\calN TH}{\alpha\delta}+\frac{T}{\alpha^2}\epsilon}\cdot\Par{\log\frac{\calN\calN_bTH}{\alpha\delta}\cdot H^{5/2}\sqrt{d_\alpha}+\sqrt{t}H\epsc}}{\sqrt{t}}+H^2\epsilon+H\delta},
\]
for some large enough constant $C<\infty$ and $\epsilon\le 1$, then we have the following facts about $\calToo$ holds true: 
\begin{align*}
\E\left[|\calToo||\calE_{\le T}\right] \le O\left(\frac{T}{\log\frac{\calN\calN_bTH}{\alpha\delta}\cdot H^3}\right).
\end{align*}
\end{lemma}

\begin{proof}
We will condition on $\calE_{\le T}$ throughout the arguments. Now we prove by contradiction, recall the definition of $h_t$, since for each $t\in\calToo$ we have $f_{t,2}^{h_t}(x_t^{h_t})\ge f_{t,1}^{h_t}(x_t^{h_t}) +u_t$, we have 
\begin{align*}
& \sum_{t\in\calToo}\Par{f_{t,2}^{h_t}(x_t^{h_t})-f_{t,1}^{h_t}(x_t^{h_t})}
\\
& \hspace{1em} \ge \frac{C}{4}\Par{\sqrt{\log\frac{\calN TH}{\alpha\delta}+\frac{T}{\alpha^2}\epsilon}\cdot\Par{\sqrt{\log\frac{\calN\calN_bTH}{\alpha\delta}}\cdot \sqrt{\log\frac{\calN\calN_bTH}{\alpha\delta}}\cdot H^{5/2} \sqrt{d_\alpha}\frac{|\calToo|}{\sqrt{T}}+|\calToo|H\epsc}}\\
& \hspace{3em} +\frac{C}{4}\Par{|\calToo|H^2\eps+|\calToo|H\delta}.
\end{align*}

Note we also have conditioning on $\calE_{\le T}$, since $f_{t,1}^h \geq f_\star^h \geq V_t^h$,  it holds that
\begin{align*}
    & \sum_{t\in\calToo}\left(f_{t,2}^{h_t}(x_t^{h_t})-f_{t,1}^{h_t}(x_t^{h_t})\right) \le \sum_{t\in\calToo}\left(f_{t,2}^{h_t}(x_t^{h_t})-V_t^{h_t}\right)\\
    & \hspace{3em} \stackrel{(i)}{\le} 2\sum_{t\in\calToo, h\in[H]}\min\left(4,b_{t,1}^{h}(z_{t}^h)\right)+2\sum_{t\in\calToo, h\in[H]}\min\left(4,b_{t,2}^{h}(z_{t}^h)\right)+\sum_{t\in\calToo}\sum_{h_t\le h\le H}\left(\xi_t^{h}-\xi_{t,2}^h\right)+O(|\calToo|H^2\epsilon)\\
    & \hspace{3em} \stackrel{(ii)}{\le} O\left(\sqrt{\log\frac{\calN TH}{\alpha\delta}+\frac{T}{\alpha^2}\epsilon}\cdot\Par{\sqrt{\log\frac{\calN\calN_bTH}{\alpha\delta}}H\sqrt{|\calToo|\cdot d_\alpha}+\log\frac{\calN\calN_bTH}{\alpha\delta}H\cdot d_\alpha+|\calToo|H\epsc}\right)\\
    & \hspace{5em}~+\sum_{t\in\calToo}\sum_{h_t\le h\le H}\left(\xi_t^{h}-\xi_{t,2}^h\right)+O(|\calToo|H^2\epsilon),
\end{align*}
where for $(i)$ we use~\Cref{lem:RL-f-tilde-oo}, and for  $(ii)$ we use~\Cref{lem:bound-E-III} and~\Cref{lem:bound-E-III-o} with $\calT = \calToo$.
Thus, conditioning on $\calE_{\le T}$, taking expectation and note 
\begin{align*}
    \E\left[\sum_{t\in\calToo}\sum_{h_t\le h\le H}(\xi_t^{h}-\xi_{t,2}^h)|\calE_{\le T}\right] & \le O\biggl(\E[|\calToo||\calE_{\le T}]H\delta+\E\biggl[\sum_{t\in\calToo}\sum_{h_t\le h\le H}(\xi_t^{h}-\xi_{t,2}^h)\biggr]\biggr) = O\biggl(\E[|\calToo||\calE_{\le T}]H\delta\biggr)
\end{align*} 
due to~\Cref{fact-xi}. Thus, in order for the two inequalities hold true simultaneously it must hold that $\E[|\calToo||\calE_{\le T}] \le O\left(T/(H^3\cdot\log(\calN\calN_bTH/\alpha\delta))\right)$.
\end{proof}

Building on this bound of $|\calToo|$, we show the next corollary on a tighter bound for the summation terms in $III$.
\begin{corollary}[Fine-grained bound on $III$]\label{lem:bound-fine-grained-b2}
Given $b_{t,2}(\cdot)\le C\cdot \biggl(D_{\calF^h}(\cdot, z_{[t-1]}^h, \1_{[t-1]}^h)\sqrt{\left(\beta_{t,2}^h\right)^2+\lambda}+\epsc \cdot \beta_{t,2}^h\biggr)$ and using the particular choice of $u_t$ as in~\Cref{lem:size-of-Too}, when $\lambda = \Theta(1)$, $\alpha\le 1$, we have the following inequality holds true:
\begin{align*}
    \E[III~|~\calE_{\le T}] & \defeq \E\left[\sum_{t\in\calToo}\sum_{h\in[H]}\min\left(1+L, b_{t,2}^h(z_t^h)\right)~|~\calE_{\le T}\right]\\
    & \hspace{3em} = O\left(\sqrt{\log\frac{\calN TH}{\delta}+T\epsilon}\cdot \sqrt{T\cdot d_\alpha}+\sqrt{\log\frac{\calN \calN_bTH}{\delta}+T\epsilon}\cdot\left(H\cdot d_\alpha+T\epsc\right)\right).
\end{align*}
\end{corollary}
\begin{proof}
This is an immediate corollary by combining~\Cref{lem:bound-E-III} and~\Cref{lem:size-of-Too}.
\end{proof}

Next, we proceed to bound $\E[I|\calE_{\le T}]$ properly. To do so, we will provide an additional helper lemma before bounding $I$. 

\begin{lemma}\label{lem:difference-f-bar}
When $\lambda = \Theta(1)$,  $\alpha\le 1$, $\epsilon\le 1$, $\delta\le 1/10$, we have 
\begin{align*}
& \E\left[\sum_{t\in[T]}\sum_{h\in[H]}\left[f_{t,2}^h(z_t^h)-f_{t,-2}^h(z_t^h)\right]~|~\calE_{\le T}\right]\\
& \hspace{2em}\le O\left(H\cdot \E[I|\calE_{\le T}]+H\cdot\E[II|\calE_{\le T}]+ H^2\cdot\E\left[|\calToo||\calE_{\le T}\right]+TH^2\epsilon+TH^2\delta\right)+H\cdot\sum_{t\in\calTo} u_t.
\end{align*}
\end{lemma}

\begin{proof}
For $t\in\calToo$, it holds that 
\[
\sum_{t\in\calToo}\sum_{h\in[H]}\left[f_{t,2}^h(z_t^h)-f_{t,-2}^h(z_t^h)\right]= O(|\calToo|H),
\]
and consequently $\E\left[\sum_{t\in\calToo}\sum_{h\in[H]}\left[f_{t,2}^h(z_t^h)-f_{t,-2}^h(z_t^h)\right]~|~\calE_{\le T}\right]= O(\E[|\calToo|~|~\calE_{\le T}]\cdot H)$

Otherwise, for iterations $t\in\calTo$, we know that it always holds true that $f_{t,2}^h(x_t^h)\le f_{t,1}^h(x_t^h)+u_t$, which implies
\begin{align*}
    & f_{t,2}^h(z_t^h)-f_{t,-2}^h(z_t^h)  \stackrel{(i)}{\le} f_{t,1}^{h}(z_t^{h})-f_{t,-2}^{h}(z_t^{h})+u_t \\
    &  \stackrel{(ii)}{\le}  u_t + O\left(\sum_{h\le h'\le H}b_{t,1}^{h'}(z_t^{h'})+\sum_{h\le h'\le H}b_{t,2}^{h'}(z_t^{h'})+\sum_{h\le h'\le H}\left(-\xi_{t,1}^{h'}+\xi_{t,-2}^{h'}\right)+H\epsilon\right),
\end{align*}
where we use $(i)$ the fact that $f_{t,2}^h(z_t^h)\le f_{t,1}^h(z_t^h)$ by \Cref{lem:mono} and $(ii)$~the upper bound on $f_{t,1}^h$ as in \Cref{lem:RL-f-tilde-oo} when $h_t=H+1$ and the lower bound on $f_{t,-1}^h$ as in~\Cref{lem:RL-f-tilde-op} conditioning on $\calE_{\le T}$. Also, we note $f_{t,2}^h(z_t^h)-f_{t,-2}^h(z_t^h)\le O(1)$. 

Now conditioning on $\calE_{\le T}$, we have 
\begin{align*}
& \sum_{t\in\calTo}\sum_{h\in[H]}\left[f_{t,2}^h(z_t^h)-f_{t,-2}^h(z_t^h)\right] \le O\left(\sum_{t\in\calTo}\sum_{ h\in[H]}\min\left(1,\sum_{h\le h'\le H}b_{t,1}^{h'}(z_t^{h'})\right)+\sum_{t\in\calTo}\sum_{h\in[H]}\min\left(1,\sum_{h\le h'\le H}b_{t,2}^{h'}(z_t^{h'})\right)\right)\\
 & \hspace{15em} +\sum_{t\in\calTo}H\cdot u_t+O\left(\sum_{t\in\calTo}\sum_{h\in[H]}\sum_{h\le h'\le H}\xi^{h'}_{t,-2}-\sum_{t\in\calTo}\sum_{h\in[H]}\sum_{h\le h'\le H}\xi^{h'}_{t,1}+TH^2\epsilon\right)\\
 & \hspace{3em}\le O\left(\sum_{t\in\calTo}\sum_{ h\in[H]}\min\left(1,\sum_{h\le h'\le H}b_{t,1}^{h'}(z_t^{h'})\right)+\sum_{t\in\calTo}\sum_{h\in[H]}\min\left(1,\sum_{h\le h'\le H}b_{t,2}^{h'}(z_t^{h'})\right)\right)\\
 & \hspace{12em} +\sum_{t\in\calTo}H\cdot u_t+O\left(\sum_{t\in[T]}\sum_{h\in[H]}\sum_{h\le h'\le H}\xi^{h'}_{t,-2}-\sum_{t\in[T]}\sum_{h\in[H]}\sum_{h\le h'\le H}\xi^{h'}_{t,1}+|\calToo|H^2+TH^2\epsilon\right).
\end{align*}
This implies that 
\begin{align*}
& \E\left[\sum_{t\in\calTo}\sum_{h\in[H]}\left[f_{t,2}^h(z_t^h)-f_{t,-2}^h(z_t^h)\right]~|~\calE_{\le T}\right]\\
& \hspace{3em} \stackrel{(i)}{\le} O\left(\E\left[\sum_{t\in\calTo}\sum_{ h\in[H]}\min\left(1,\sum_{h\le h'\le H}b_{t,1}^{h'}(z_t^{h'})\right)+\sum_{t\in\calTo}\sum_{h\in[H]}\min\left(1,\sum_{h\le h'\le H}b_{t,2}^{h'}(z_t^{h'})\right)~\big|~\calE_{\le T}\right]\right)\\
&  \hspace{8em}+\sum_{t\in\calTo}H\cdot u_t+O(H^2\E[|\calToo||\calE_{\le T}]+TH^2\epsilon+TH^2\delta)\\
 & \hspace{3em} \stackrel{(ii)}{\le} O\left(H\cdot\E[I~|~\calE_{\le T}]+H\cdot\E[II~|~\calE_{\le T}]+H\sum_{t\in\calTo}u_t+H^2\E[|\calToo||\calE_{\le T}]+TH^2\epsilon+TH^2\delta\right).
  \end{align*}
Here we use $(i)$ since $\calE_{\le T}$ happens with probability $1-5\delta\ge 1/2$, and $|\xi_{t,j}^h|\le O(1)$ for any $t,h$ and any $j=-2,1$, so that $\E\bigl[\sum_{t\in[T]}\sum_{h\in[H]}\sum_{h\le h'\le H}\xi^{h'}_{t,-2}~|~\calE_{\le T}\bigr]=O(TH^2\delta)$ and $\E\bigl[\sum_{t\in[T]}\sum_{h\in[H]}\sum_{h\le h'\le H}\xi^{h'}_{t,1}~|~\calE_{\le T}\bigr]=O(TH^2\delta)$ using~\Cref{fact-xi}. For $(ii)$ we simply use the definition of $I$ and $II$ and the fact that all bonus terms are non-negative.

Thus summing the two cases gives the claimed bound.
\end{proof}

\begin{lemma}[Fine-grained bound on $I$]\label{lem:regret-bound-I}
Recall the definition of $b_{t,1}$ and $b_{t,2}$ as in~\Cref{lem:bound-E-III-o} and~\Cref{lem:bound-E-III}. When $\lambda = 1$, $\alpha= 1/\sqrt{TH}$, $\epsilon\le 1$ and $\delta\le 1/10$, conditioning on the event $\calE_{\le T}$, we have the following inequality holds true:
\begin{align*}
    & \E[I|\calE_{\le T}] \defeq \E\left[\sum_{t\in[T]}\sum_{h\in[H]}\min\left(1+L, b_{t,1}^h(z_t^h)\right)|\calE_{\le T}\right]\\
    & \hspace{1em} = O\left(\sqrt{\log\frac{\calN TH}{\delta}+T^2H\epsilon}\cdot \sqrt{T}\cdot\sqrt{H d_\alpha}\right)\\
    & \hspace{3em} +O\left(\sqrt{\log\frac{\calN TH}{\delta}+T^2H\epsilon}\sqrt{\log\frac{\calN\calN_bTH}{\delta}}\sqrt{TH^3(\epsilon+\delta)+H^3\E[|\calToo||\calE_{\le T}]+H^2\sum_{t\in\calTo}u_t}\cdot\sqrt{Hd_\alpha}\right)\\
         & \hspace{3em}+O\left( \Par{\log\frac{\calN TH}{\delta}+T^2H\epsilon}\log^{1.5} \frac{\calN\calN_bTH}{\delta}\cdot H^{7/2}d_\alpha+\sqrt{\log\frac{\calN TH}{\delta}+T^2H\epsilon}\cdot TH\epsc\right).
\end{align*}
\end{lemma}

\begin{proof}
We first note that by assumption and definition, 
\begin{equation}\label{eq:lem-E-I-initial}
\begin{aligned}
   & \sum_{t\in[T]}\sum_{h\in[H]} \min\left(1+L,b_{t,1}^h(z_t^h) \right)
    = O\left(\sum_{t\in[T]}\sum_{h\in[H]}\min\left(1,D_{\calF^h}(z_t^h; z_{[t-1]}^h, \bsigma_{[t-1]}^h)\cdot\sqrt{\left(\beta_{t,1}^h\right)^2+\lambda} \right)+TH\epsc\cdot\max_{t,h}\beta_{t,1}^h\right) \\
    & \hspace{10em} = O\left(\sqrt{\log\frac{\calN TH}{\delta}+T^2H\epsilon}\cdot \Par{\sum_{t\in[T]}\sum_{h\in[H]}\min\left(1,D_{\calF^h}(z_t^h; z_{[t-1]}^h, \bsigma_{[t-1]}^h)\right)+TH\epsc}\right).
\end{aligned}
\end{equation}

Treating $L=O(1)$ as defined (see~\Cref{ass:eps-realizability-RL}), we now bound the summation terms \[\sum_{t\in[T]}\sum_{h\in[H]}\min\left(1,D_{\calF^h}(z_t^h; z_{[t-1]}^h, \bsigma^h_{[t-1]})\right) = \sum_{t\in[T]}\sum_{h\in[H]}\min\left(1,\bsigma_t^h \cdot \Par{\bsigma_t^h}^{-1} D_{\calF^h}(z_t^h; z_{[t-1]}^h, \bsigma^h_{[t-1]})\right)\] by dividing into cases.

We consider separating the index set $\{(t,h):t\in[T], h\in[H]\}$ as follows, same as the cases we consider in~\Cref{lem:bound-E-III-o}.
\begin{align*}
\calI_1 &= \{(t,h)\in\calT\times[H]~|~\Par{\bsigma_t^h}^{-1}D_{\calF^h}(z_t^h; z_{[t-1]}^h, \bsigma_{[t-1]}^h)\ge 1\},\\
\calI_2 &= \{(t,h)\in\calT\times[H]~|~(t,h)\notin\calI_1,~\bsigma_t^h = \alpha\},\\
\calI_3 &= \biggl\{(t,h)\in\calT\times[H]~|~(t,h)\notin\calI_1,~\bsigma_t^h = 2\bigl(\sqrt{\upsilon(\delta_{t,h})}+\iota(\delta_{t,h})\bigr)\cdot\sqrt{ D_{\calF^h}(z^{h}_t; z^{h}_{[t-1]},\bsigma^{h}_{[t-1]})}\biggr\},\\
\calI_4 &= \left\{(t,h)\in\calT\times[H]~|~(t,h)\notin\calI_1,~\bsigma_t^h = \sigma_t^h\right\},\\
\calI_5 &= \left\{(t,h)\in\calT\times[H]~|~(t,h)\notin\calI_1,~\bsigma_t^h = \sqrt{2}\iota(\delta_{t,h})\sqrt{f^h_{t,2}(z_t^h)-f^h_{t,-2}(z_t^h)}\right\}.
\end{align*}
    
For the terms restricting on $\calI_1$, $\calI_2$, $\calI_3$ we recall the bounds in~\Cref{eq:crude-1,eq:crude-2,eq:crude-3} respectively proven in~\Cref{lem:bound-E-III-o} such that
\begin{align*}
    & \sum_{(t,h)\in\calI_1\cup\calI_2\cup\calI_3}\min\left(1,\bsigma_{t}^h\cdot \Par{\bsigma_{t}^h}^{-1}D_{\calF^h}(z_{t}^h;z_{[t-1]}^h,\bsigma_{[t-1]}^h)\right) \le O\left(\log\frac{\calN\calN_b TH}{\delta} \cdot H d_\alpha\right).
    \end{align*}
     
    For terms restricting on $\calI_4$ and $\calI_5$ we do a tighter analysis different from~\Cref{lem:bound-E-III-o}. For summations terms in $\calI_5$, we have
    \begin{equation}\label{eq:regret-I-calI-5}
    \begin{aligned}
         & \sum_{(t,h)\in\calI_5}\min\left(1,\bsigma_{t}^h\cdot\Par{\bsigma_{t}^h}^{-1}D_{\calF^h}(z_{t}^h;z_{[t-1]}^h,\bsigma_{[t-1]}^h)\right)  \le\sum_{(t,h)\in\calI_5}\bsigma_{t}^h\cdot\Par{\bsigma_{t}^h}^{-1}D_{\calF^h}(z_{t}^h;z_{[t-1]}^h,\bsigma_{[t-1]}^h) \\
         & \hspace{5em} = \sum_{(t,h)\in\calI_5}\sqrt{2}\iota(\delta_{t,h})\cdot\sqrt{f_{t,2}^h(z_t^h)-f_{t,2}^h(z_t^h)}\cdot \Par{\bsigma_{t}^h}^{-1} D_{\calF^h}(z_{t}^h;z_{[t-1]}^h,\bsigma_{[t-1]}^h)
         \\
        & \hspace{5em} \le O\left(\sqrt{\log\frac{\calN\calN_b TH}{\delta}}\sqrt{\sum_{t,h}\Par{f_{t,2}^h(z_t^h)-f_{t,-2}^h(z_t^h)}}\cdot\sqrt{H\cdot d_\alpha}\right).
    \end{aligned}
    \end{equation}
    Here for the last inequality we use Cauchy-Schwarz inequality together with the definition of $\iota(\delta)$ as in~\eqref{eq:def-iota-opti}.
    
    Restricting on $\calI_4$, by Cauchy-Schwarz inequality and Jensen's inequality, we have 
    \begin{align*}
        & \E\left[\sum_{(t,h)\in\calI_4}\min\left(1,\bsigma_{t}^h\cdot \Par{\bsigma_t^h}^{-1}D_{\calF^h}(z_{t}^h;z_{[t-1]}^h,\bsigma_{[t-1]}^h)\right)|\calE_{\le T}\right] \\
        & \hspace{1em}\le \E\left[\sum_{(t,h)\in\calI_4}\bsigma_{t}^h\cdot \Par{\bsigma_t^h}^{-1}D_{\calF^h}(z_{t}^h;z_{[t-1]}^h,\bsigma_{[t-1]}^h)|\calE_{\le T}\right]\\
         & \hspace{1em} \stackrel{(o)}{\le} \sqrt{\E\left[\sum_{t,h\in\calI_4}\left(\sigma_t^h\right)^2|\calE_{\le T}\right]}\cdot \sqrt{\E\left[\sum_{t,h\in\calI_4}\Par{\bsigma_{t}^h}^{-2}D_{\calF^h}^2(z_{t}^h;z_{[t-1]}^h,\bsigma_{[t-1]}^h)|\calE_{\le T}\right]}\\
         & \hspace{1em} \stackrel{(i)}{\le} \sqrt{\E\left[\sum_{t,h\in\calI_4}\left(\sigma_t^h\right)^2|\calE_{\le T}\right]}\cdot \sqrt{\E\left[\sum_{t,h\in\calI_4}\min\left(1,\Par{\bsigma_{t}^h}^{-2}D_{\calF^h}^2(z_{t}^h;z_{[t-1]}^h,\bsigma_{[t-1]}^h)\right)|\calE_{\le T}\right]}\\
        & \hspace{1em}  \stackrel{(ii)}{\le} O\left(\sqrt{\E\left[\sum_{t,h}\rV_{r^h, x^{h+1}}\left[r^h+f^{h+1}_{t,1}(x^{h+1})|z_{t}^{h}\right]+\sum_{t,h}\left(f_{t,2}^h(z_{t}^h)-f_{t,-2}^h(z_{t}^h)\right)|\calE_{\le T}\right]}\cdot\sqrt{H\cdot d_\alpha}\right)\\
        & \hspace{3em} +O\left( \sqrt{\E\left[\sum_{t,h}\min\left(1,D_{\calF^h}(z_t^h; z_{[t-1]}^h, \1_{[t-1]}^h)\right)\sqrt{\log\frac{\calN\calN_b TH}{\delta}+T\epsilon}+TH\epsilon|\calE_{\le T}\right]}\cdot\sqrt{H\cdot d_\alpha}\right)\\
        & \hspace{1em} \stackrel{(iii)}{\le} O\left(\sqrt{2\E\sum_{t,h}\rV_{r^h, x^{h+1}}\left[r^h+f^{h+1}_{t,1}(x^{h+1})|z_{t}^{h}\right]+\E\left[\sum_{t,h}\left(f_{t,2}^h(z_{t}^h)-f_{t,-2}^h(z_{t}^h)\right)|\calE_{\le T}\right]}\cdot\sqrt{H\cdot d_\alpha}\right)\\
        & \hspace{3em} +O\left( \sqrt{\E\left[\sum_{t,h}\min\left(1,D_{\calF^h}(z_t^h; z_{[t-1]}^h, \1_{[t-1]}^h)\right)\sqrt{\log\frac{\calN\calN_b TH}{\delta}+T\epsilon}+TH\epsilon|\calE_{\le T}\right]}\cdot\sqrt{H\cdot d_\alpha}\right).
        \end{align*}
        Here we use $(o)$ the condition that $\Par{\bsigma_{t}^h}^{-1}D_{\calF^h}(z_{t}^h;z_{[t-1]}^h,\bsigma_{[t-1]}^h)\le 1$ by definition of $\calI_4$, $(i)$ definition of Eluder dimension, $(ii)$~\Cref{lem:RL-variance-bounds-upper} and $(iii)$ the fact that event $\calE_{\le T}$ happens with at least $1-5\delta$ probability so that 
 \[\E\left[\sum_{t,h}\rV_{r^h, x^{h+1}}\left[r^h+f^{h+1}_{t,1}(x^{h+1})|z_{t}^{h}\right]|\calE_{\le T}\right]\le \frac{1}{1-5\delta}\E\sum_{t,h}\rV_{r^h, x^{h+1}}\left[r^h+f^{h+1}_{t,1}(x^{h+1})|z_{t}^{h}\right].\]

Further, we have 
        \begin{align*}
         & \E\left[\sum_{(t,h)\in\calI_4}\min\left(1,\bsigma_{t}^h\cdot \Par{\bsigma_t^h}^{-1}D_{\calF^h}(z_{t}^h;z_{[t-1]}^h,\bsigma_{[t-1]}^h)\right)|\calE_{\le T}\right]\\
         & \hspace{1em}  \stackrel{(i)}{\le} O\left(\sqrt{T+H^2\cdot\E|\calToo|+TH^2(\epsilon+\delta)+H\cdot\E\left[\sum_{t,h}\left(f_{t,2}^h(z_{t}^h)-f_{t,-2}^h(z_{t}^h)\right)\right]}\cdot\sqrt{H\cdot d_\alpha}\right)\\
         & \hspace{3em} +O\left(\sqrt{\E\left[\sum_{t,h}\left(f_{t,2}^h(z_{t}^h)-f_{t,-2}^h(z_{t}^h)\right)|\calE_{\le T}\right]}\cdot\sqrt{H\cdot d_\alpha}\right)\\
         & \hspace{3em} +O\left( \sqrt{\Par{H\cdot d_\alpha+H\sqrt{T\cdot d_\alpha}}\sqrt{\log\frac{\calN\calN_b TH}{\delta}+T\epsilon}}\cdot\sqrt{H\cdot d_\alpha}\right)\\
         & \hspace{1em} \stackrel{(ii)}{\le} O\left(\sqrt{T+H^2\E[|\calToo|~|~\calE_{\le T}]+TH^2(\epsilon+\delta)}\cdot\sqrt{H\cdot d_\alpha}\right)\\
         & \hspace{3em} +O\left(\sqrt{H\cdot\E\left[\sum_{t,h}\left(f_{t,2}^h(z_{t}^h)-f_{t,-2}^h(z_{t}^h)\right)|\calE_{\le T}\right]}\cdot\sqrt{H\cdot d_\alpha}\right)\\
         & \hspace{3em} +O\left( H\sqrt{d_\alpha\cdot\Par{\log\frac{\calN\calN_b TH}{\delta}+T\epsilon}}\cdot\sqrt{H\cdot d_\alpha}\right).
    \end{align*}
     Here for $(i)$ we have plugged in bounds in~\Cref{coro:sum-variance} and~\Cref{eq:bound-b-2-sum}. For $(ii)$
in the first line we note $\calE_{\le T}$ happens with probability $1-5\delta$ so that $\E[\sum_{t,h}(f_{t,2}^h-f_{t,-2}^h)]\le O(\delta TH+\E[\sum_{t,h}(f_{t,2}^h-f_{t,-2}^h)|\calE_{\le T}])$ and $\E|\calToo|\le O(\delta T+\E[|\calToo||\calE_{\le T}])$ since with probability 1 we have $\sum_{t,h}(f_{t,2}^h-f_{t,-2}^h)\le O(TH)$ and $|\calToo|\le T$. In the third line of $(ii)$ we also use AM-GM inequality such that \begin{equation}\label{eq:the-AM-GM-pre}
    H \sqrt{T\cdot d_\alpha\Par{\log\frac{\calN\calN_bTH}{\delta}+T\epsilon}}\le T+H^2d_\alpha\cdot\Par{\log\frac{\calN\calN_bTH}{\delta}+T\epsilon}.
\end{equation}

Summing all terms together and taking conditional expectation we have
\begin{align*}
        & \E\left[\sum_{t\in[T]}\sum_{h\in[H]}\min\left(1,D_{\calF^h}(z_t^h; z_{[t-1]}^h, \bsigma_{[t-1]}^h)\right)|\calE_{\le T}\right]\\
        & \hspace{1em} =  O\left(\sqrt{T+H^2\E[|\calToo|~|~\calE_{\le T}]+TH^2(\epsilon+\delta)}\cdot\sqrt{H d_\alpha}\right)\\
         & \hspace{3em} + O\Par{\sqrt{\log\frac{\calN\calN_bTH}{\delta}}\cdot\sqrt{H d_\alpha}\cdot\sqrt{H\cdot\E\left[\sum_{t,h}\Par{f_{t,2}^h(z_t^h)-f_{t,-2}^h(z_t^h)}|\calE_{\le T}\right]}}\\
         & \hspace{3em} + O\Par{\Par{\log\frac{\calN\calN_b TH}{\delta}+T\epsilon}\cdot H^{1.5}\cdot d_\alpha}.
\end{align*}
Now plugging in the bounds proven in \Cref{lem:difference-f-bar} we have
\begin{equation}\label{eq:lem-E-I-step}
\begin{aligned}
& \E\left[\sum_{t\in[T]}\sum_{h\in[H]}\min\left(1,D_{\calF^h}(z_t^h; z_{[t-1]}^h, \bsigma_{[t-1]}^h)\right)|\calE_{\le T}\right]\\
        & \hspace{1em} =  O\left(\sqrt{T}\cdot\sqrt{H d_\alpha}+\sqrt{\log\frac{\calN\calN_b TH}{\delta}}\sqrt{H^3\cdot\E[|\calToo|~|~\calE_{\le T}]+TH^3(\epsilon+\delta)+H^2\sum_{t\in[T]}u_t}\cdot\sqrt{H\cdot d_\alpha}\right)\\
        & \hspace{3em} + O\Par{\Par{\log\frac{\calN\calN_b TH}{\delta}+T\epsilon}\cdot H^{1.5}\cdot d_\alpha}\\
         & \hspace{3em} + O\Par{\sqrt{\log\frac{\calN\calN_bTH}{\delta}}\cdot\sqrt{H d_\alpha}\cdot\sqrt{H^2\cdot\E\left[I|\calE_{\le T}\right]+H^2\cdot\E\left[II|\calE_{\le T}\right]}}.
\end{aligned}
\end{equation}
We can further bound the last term in the RHS of above inequality as
\begin{equation}\label{eq:lem-E-I-helper}
\begin{aligned}
	& H^2\cdot\E\left[I|\calE_{\le T}\right]+H^2\cdot\E\left[II|\calE_{\le T}\right] \\
	& \hspace{1em} \le  O\Par{\sqrt{\log\frac{\calN TH}{\delta}+T^2H\epsilon}\cdot\sqrt{\log\frac{\calN\calN_bTH}{\delta}}\cdot H^3\sqrt{Td_\alpha}}\\
	& \hspace{3em} +O\Par{\sqrt{\log\frac{\calN TH}{\delta}+T^2H\epsilon}\cdot \log\frac{\calN\calN_b TH}{\alpha\delta}\cdot H^2d_\alpha+\sqrt{\log\frac{\calN\calN_b TH}{\delta}+T^2H\eps}\cdot TH^2\epsc}\\
   &\hspace{1em} \le O\Par{\frac{T}{\log\frac{\calN\calN_bTH}{\delta}}+\Par{\log\frac{\calN TH}{\delta}+T^2H\epsilon}\Par{\log\frac{\calN\calN_bTH}{\delta}}^2H^6d_\alpha +\sqrt{\log\frac{\calN\calN_b TH}{\delta}+T^2H\epsilon}\cdot TH^2\epsc}.
\end{aligned}
\end{equation}
Here for the first inequality we plug in crude bounds in~\Cref{lem:bound-E-III} and~\Cref{lem:bound-E-III-o} given the definition of bonus terms, and absorb low-order terms. For the second inequality we use the  AM-GM inequality such that
\[
 \sqrt{\log\frac{\calN TH}{\delta}+T^2H\epsilon}\cdot\sqrt{\log\frac{\calN\calN_bTH}{\delta}}\cdot H^3\sqrt{Td_\alpha} \le \frac{T}{\log\frac{\calN\calN_bTH}{\delta}}+\Par{\log\frac{\calN TH}{\delta}+T^2H\epsilon}\Par{\log\frac{\calN\calN_bTH}{\delta}}^2H^6d_\alpha,
\]
and absorb other low-order terms.

Plugging~\eqref{eq:lem-E-I-helper} back to the original bound in~\eqref{eq:lem-E-I-step}, rearranging terms and absorbing low-order terms we get
\begin{align*}
& \E\left[\sum_{t\in[T]}\sum_{h\in[H]}\min\left(1,D_{\calF^h}(z_t^h; z_{[t-1]}^h, \bsigma_{[t-1]}^h)\right)|\calE_{\le T}\right]\\
         & \hspace{1em} \le O\left(\sqrt{T}\cdot\sqrt{H d_\alpha}+\sqrt{\log\frac{\calN\calN_b TH}{\delta}}\sqrt{H^3\cdot\E[|\calToo|~|~\calE_{\le T}]+TH^3(\epsilon+\delta)+H^2\sum_{t\in[T]}u_t}\cdot\sqrt{H\cdot d_\alpha}\right)\\
         & \hspace{3em} + O\left(\sqrt{ \Par{\log\frac{\calN TH}{\delta}+T^2H\epsilon}\log^3\frac{\calN\calN_bTH}{\delta}}\cdot H^{7/2}d_\alpha+TH\epsc\right).
    \end{align*}
Here the low-order $TH\epsc$ term comes from applying AM-GM inequality on the $\mathrm{poly(\epsc)}$ term and absorbing other low-order terms by $O\left(\sqrt{ \Par{\log\frac{\calN TH}{\delta}+T^2H\epsilon}\log^3\frac{\calN\calN_bTH}{\delta}}\cdot H^{7/2}d_\alpha\right)$.

Plugging the bound back to~\eqref{eq:lem-E-I-initial} and rearranging terms, we have the claimed bounds.
\end{proof}

\theoremregret*

\begin{proof}
    Following~\Cref{eq:regret-expected}, we have 
\begin{align*}
    \E R_T & = O\left(1+\delta HT+\epsilon HT \right) +2\E\left[I|\calE_{\le T}\right]+ 2\E\left[III|\calE_{\le T}\right].
\end{align*}

Now plugging in the guarantees in~\Cref{lem:regret-bound-I} for bounding $\E[I|\calE_{\le T}]$, and~\Cref{lem:bound-fine-grained-b2} for bounding $\E\left[III|\calE_{\le T}\right]$, we have
\begin{align*}
    \E R_T & = O\Par{1+TH(\epsilon+\delta)+\sqrt{\log\frac{\calN TH}{\delta}+T^2H\epsilon}\cdot \sqrt{T}\cdot\sqrt{Hd_\alpha}}\\
    & \hspace{3em} +O\left(\sqrt{\log\frac{\calN TH}{\delta}+T^2H\epsilon}\sqrt{\log\frac{\calN\calN_bTH}{\delta}}\sqrt{1+TH^3(\epsilon+\delta)+H^3\E[|\calToo||\calE_{\le T}]+H^2\sum_{t\in\calTo}u_t}\cdot\sqrt{Hd_\alpha}\right)\\
         & \hspace{3em}+O\left( \Par{\log\frac{\calN TH}{\delta}+T^2H\epsilon}\log^{1.5} \frac{\calN\calN_bTH}{\delta}\cdot H^{7/2}d_\alpha+\sqrt{\log\frac{\calN\calN_b TH}{\delta}+T^2H\epsilon}\cdot TH\epsc\right).
\end{align*}
Now further plugging in the choice of $u_t$ and $\E[|\calToo||\calE_{\le T}]$ due to~\Cref{lem:size-of-Too}, we have 
\begin{align*}
    & \sqrt{\log\frac{\calN\calN_bTH}{\delta}}\sqrt{H^2\E[|\calToo||\calE_{\le T}]+H^2\sum_{t\in\calTo}u_t}\\
    & \hspace{2em} =O\Par{\sqrt{\log\frac{\calN\calN_bTH}{\delta}}\sqrt{\frac{T}{\log\frac{\calN\calN_bTH}{\delta}}+\sqrt{\log\frac{\calN TH}{\delta}+T^2H\epsilon}\cdot TH^3\epsc+TH^4\eps+TH^3\delta}}\\
    & \hspace{5em} +O\Par{\sqrt{\log\frac{\calN\calN_bTH}{\delta}}\cdot\sqrt{\sqrt{\log\frac{\calN TH}{\delta}+T^2H\epsilon}\cdot\log\frac{\calN\calN_bTH}{\delta}H^{4.5}\sqrt{d_\alpha}\sqrt{T}}}\\
    & \hspace{2em} \le O\Par{\sqrt{\log\frac{\calN\calN_bTH}{\delta}}\sqrt{\frac{T}{\log\frac{\calN\calN_bTH}{\delta}}+\sqrt{\log\frac{\calN TH}{\delta}+T^2H\epsilon}\cdot TH^3\epsc+TH^{4}\eps+TH^3\delta}}\\
    & \hspace{5em} +O\Par{\sqrt{\log\frac{\calN\calN_bTH}{\delta}}\cdot\sqrt{\frac{T}{\log\frac{\calN\calN_b TH}{\delta}}+\Par{\log\frac{\calN TH}{\delta}+T^2H\epsilon}\cdot\log^3\frac{\calN\calN_bTH}{\delta}H^{9}d_\alpha}}\\
    & \hspace{2em} = O(\sqrt{T})+O\Par{\sqrt{\log\frac{\calN\calN_bTH}{\delta}}\cdot\sqrt{\Par{\log\frac{\calN TH}{\delta}+T^2H\epsilon}\cdot\Par{\log^3\frac{\calN\calN_bTH}{\delta}H^{9}d_\alpha+TH^3(\epsc+\delta)}}}
\end{align*}
which by multiplying both sides with $\sqrt{\log\frac{\calN TH}{\delta}+T^2H\epsilon}\cdot\sqrt{Hd_\alpha}$, plugging back, using AM-GM inequality to simplify the $\mathrm{poly}(\epsb,\delta)$ terms gives
\begin{align*}
    \E R_T & = O\Par{1+TH(\epsilon+\delta)+\sqrt{\log\frac{\calN TH}{\delta}+T^2H\epsilon}\cdot \sqrt{T}\cdot\sqrt{H d_\alpha}}\\
    & \hspace{3em} +O\left(\Par{\log\frac{\calN TH}{\delta}+T^2H\epsilon}\cdot\Par{\log^2\frac{\calN\calN_bTH}{\delta}H^{5}d_\alpha+T^2\epsc^2+T\delta}\right).
\end{align*}
Adjusting the constant $\delta\leftarrow \tfrac{\delta}{5}$, using the range of $\delta$ and omitting low-order terms of $\mathrm{poly}(\eps)$ conclude the final bound for regret.
\end{proof}

\section{Regret with High Probability.}\label{app:regret-hp}

In this section, we provide the full analysis of the high-probability guarantee stated in~\Cref{thm:regret-genera.} of~\Cref{alg:fitted-Q-simpler}. We first provide a complete statement of the guarantee as follows.

\begin{restatable}[Regret bound with high probability]{theorem}{thmhp}\label{thm:regret-hp}
Suppose function class $\{\calF^h\}_{h\in[H]}$ satisfy~\Cref{ass:eps-realizability-RL} with $\epsilon\in[0,1]$ and~\Cref{def:general-eluder-RL} with $\lambda=1$, given consistent bonus oracle $\oracle$ satisfying~\Cref{def:bonus-conditions}, \Cref{alg:fitted-Q-simpler} with $\alpha= \sqrt{1/TH}$, $\delta\le 1/(H^2+11)$, $\epsilon\le 1$ and 
\[u_t=  C\cdot\Par{\frac{\sqrt{\log\frac{\calN TH}{\alpha\delta}+\frac{T}{\alpha^2}\epsilon}\cdot\Par{\log\frac{\calN\calN_bTH}{\alpha\delta}\cdot H^{5/2}\sqrt{d_\alpha}+\sqrt{t}H^2\epsc}}{\sqrt{t}}+H^2\epsilon}
\]
for sufficiently large constant $C<\infty$, 
 with high probability $1-\delta$  event $\calE = \calE_{\le T}\cap \calE_{\xi_1}\cap\calE_{\xi_{-2}}\cap\calE_{\xi_2}\cap\calE_{\xid}\cap\calE_{\rV}$ happens. Further, when conditioning on $\calE$ the algorithm achieves a total regret of
\begin{align*}
R_T &  = O\Par{\sqrt{\log\frac{\calN TH}{\delta}+T^2H\epsilon}\cdot \sqrt{T}\cdot\sqrt{H d_\alpha}+\Par{\log\frac{\calN TH}{\delta}+T^2H\epsilon}\cdot\Par{\log^2\frac{\calN\calN_bTH}{\delta}H^{5}d_\alpha+T^2\epsc^2}}.
\end{align*}
\end{restatable}

The notation of this section is the same as in~\Cref{app:proofs}; see~\Cref{tbl:notation} and~\Cref{tbl:parameters} for the formal definitions. It also builds on~\Cref{app:notation,app:CI,app:variance,app:approx-error}. The section can be viewed as an alternative of the analysis in~\Cref{app:regret} for the high-probability setting. It is organized as follows: In~\Cref{app:concentrate-xi} we prove some concentration properties of the martingale sequences used in bounding the final regret. In~\Cref{app:size} we bound the size of $|\calToo|$, showing that with high probability the agent doesn't use $f_{t,2}$ too often in the exploration. In~\Cref{app:variance-hp}, we bound the summation of variances incurred in the total exploration in high probability, using the concentration property together with expectation bound shown in~\Cref{coro:sum-variance}. In~\Cref{app:diff-of-f}, we give the high-probability bound on the summation of differences $\sum_{t,h}(f_{t,2}^h-f_{t,-2}^h)$ which are used in the definition of $\bsigma_t^h$. Finally, we combine all parts together and bound the summation of bonus terms in~\Cref{app:regret-hp-subsec}  to prove the final high-probability regret bound.

\subsection{Concentration of Random Variables}\label{app:concentrate-xi}
To turn the in-expectation bound into a high probability argument, we first provide a few concentration results on the random variables $\xi$s, building on their MDS property as stated in~\Cref{fact-xi}.

\begin{lemma}\label{lem:freedman-xi}
   Recall the simplified notation of $\rV[\cdot|z_t^h] = \rV_{r^h, x^{h+1}}[\cdot|z_t^h]$. For given $\delta\in(0,1)$, we have:
    \begin{itemize}
        \item For $\{\xi_{t}^h\}_{t,h}$, we let 
        $\calE_{\xi}$ be the event such that \begin{align}
        \left|\sum_{t\in[T],h\in[H]}\xi_t^h\right|& \le \sqrt{\sum_{t\in[T],h\in[H]}\rV\left[r^h+V_t^{h+1}|z_t^{[h]}, f_{t,1}^{[H]}, f_{t,2}^{[H]}\right]\log\frac{2}{\delta}}+2\log\frac{2}{\delta}\nonumber\\
        & \le \sqrt{4TH\log\frac{2}{\delta}}+2\log\frac{2}{\delta},\label{eq:freedman-xi}
        \end{align}
        we thus have event $\calE_{\xi}$ happens with probability at least $1-\delta$.
        \item For $\{\xi_{t,1}^h\}_{t,h}$, we let $\calE_{\xi_1}$ be the event such that \begin{align}
        \left|\sum_{t\in[T],h' \ge h}\xi_{t,1}^{h'}\right| & \le \sqrt{\sum_{t\in[T],h'\ge h}\rV[r^{h'}+f_{t,1}^{h'+1}(x^{h'+1})|z_t^{h'}]\log\frac{2H}{\delta}}+2\log\frac{2H}{\delta}\nonumber\\
        & \le \sqrt{4TH\log\frac{2H}{\delta}}+2\log\frac{2H}{\delta}~\text{for all}~h\in[H],\label{eq:freedman-xi-1}
        \end{align}
        we thus have event $\calE_{\xi_1}$ happens with probability at least $1-\delta$.
        \item For $\{\xi_{t,2}^h\}_{t,h}$, we let $\calE_{\xi_2}$ be the event such that \begin{align}
        \left|\sum_{t\in[T],h\in[H]}\xi_{t,2}^{h}\right| & \le \sqrt{\sum_{t\in[T],h'\ge h}\rV[r^{h'}+f_{t,2}^{h'+1}(x^{h'+1})|z_t^{h'}]\log\frac{2}{\delta}}+2\log\frac{2}{\delta}\nonumber\\
        & \le \sqrt{4TH\log\frac{2}{\delta}}+2\log\frac{2}{\delta},\label{eq:freedman-xi-2-pos}
        \end{align}
        we thus have event $\calE_{\xi_2}$ happens with probability at least $1-\delta$.
        \item For $\{\xi_{t,-2}\}_{t,h}$, we let $\calE_{\xi_{-2}}$ be the event such that \begin{align}
        \left|\sum_{t\in[T],h'\ge h}\xi_{t,-2}^{h'}\right| & \le \sqrt{\sum_{t\in[T],h'\ge h}\rV[r^{h'}+f_{t,-2}^{h'+1}(x^{h'+1})|z_t^{h'}]\log\frac{2H}{\delta}}+2\log\frac{2H}{\delta}\nonumber\\
       & \le \sqrt{4TH\log\frac{2H}{\delta}}+2\log\frac{2H}{\delta} ~\text{for all}~h\in[H],\label{eq:freedman-xi-2}
        \end{align}
        we thus have event $\calE_{\xi_{-2}}$ happens with probability at least $1-\delta$.
    \end{itemize}
\end{lemma}
The proof of this lemma is an immdiate application of~\Cref{lem:Freedman}.

\subsection{Size of $|\calToo|$}\label{app:size}

We consider the following lemma due to martingale concentration, which will be useful to give a with high probability argument for bounding the size of $\calToo$.

\begin{lemma}[Concentration with indicators]\label{lem:concentration-indicator}
Let $\D_{t}^{h} = \left(\xi_{t}^h-\xi_{t,2}^h\right)\1_{\{h\ge h_t\}}$ for any $t\in[T]$, $h\in[H]$. We have $\D_{t}^{h}$ is a martingale difference sequence and with probability $1-\delta$, 
\begin{align*}
\sum_{t\in[T],h\in[H]}\left(\xi_{t}^h-\xi_{t,2}^h\right)\1_{\{h\ge h_t\}}\le O\left( \sqrt{|\calToo|H\log\frac{TH}{\delta}}+ \log\frac{TH}{\delta}\right),\\
\text{and also}~\sum_{t\in[T],h\in[H]}\left(\xi_{t,1}^h-\xi_{t,2}^h\right)\1_{\{h\ge h_t\}}\le O\left( \sqrt{|\calToo|H\log\frac{TH}{\delta}}+ \log\frac{TH}{\delta}\right).
\end{align*}
We call this event $\calE_{\xid}$.
\end{lemma}
\begin{proof}
We first prove the first inequality. Recall the definition that $\xi_{t}^h = r^{h}_t+V_t^{h+1}-\E\left[r^h+V_t^{h+1}|z_t^{[h]}, f_{t,1}^{[H]},f_{t,2}^{[H]}\right]$ and $\xi_{t,2}^h = r^{h}_t+f_{t,2}^{h+1}(x^{h+1}_{t})-\E_{r^h, x^{h+1}}\left[r^h+f_{t,2}^{h+1}(x^{h+1})|z_t^h\right]$. We also recall the filtration defined earlier as \[\calH_t^{h}=\sigma(x_1^{1},r_1^1, x_1^2,\cdots, r^{H}_1,x^{H+1}_1; x_2^{1},r_2^1, x_2^2,\cdots, r^{H}_2,x^{H+1}_2;\cdots,x_t^1,r_t^1,\cdots, r_t^{h},x_t^{h+1})~~\text{for any}~t\in[T], h=0,1,\dots, H.\]

Thus following~\Cref{fact-xi} we have
\[
\E[\D_{t}^{h}|\calH_{t}^{h-1}] = \1_{\{ h\ge h_t\}}\cdot\E[\xi_t^h-\xi_{t,2}^h|\calH_{t}^{h-1}] = 0,
\]
which by definition shows that $\D_{t}^{h}$ as defined is a martingale difference sequence.

Further, applying~\Cref{lem:Freedman-variant} to $\{\D_{t}^{h}\}_{t\in[T]}^{h\in[H]}$, we have with probability at least $1-\delta/2$, it holds that
\[
\sum_{t\in[T], h\in[H]} \left(\xi_{t}^h-\xi_{t,2}^h\right)\1_{\{h\ge h_t\}}\le O\left( \sqrt{|\calToo|H\log\frac{TH}{\delta}}+ \log\frac{TH}{\delta}\right),
\]
where we use the variance of $\xi_t^h$ and $\xi_{t,2}^h$ are constants for each $t,h$.

Similarly we can show with probability at least $1-\delta/2$ the second inequality holds true too.
\end{proof}

\begin{lemma}[Bounding size of $\calToo$]\label{lem:size-of-Too-hp}
Suppose $\alpha\le 1$, we set 
\[u_t \ge C\cdot\Par{\frac{\sqrt{\log\frac{\calN TH}{\alpha\delta}+\frac{T}{\alpha^2}\epsilon}\cdot\Par{\log\frac{\calN\calN_bTH}{\alpha\delta}\cdot H^{5/2}\sqrt{d_\alpha}+\sqrt{t}H\epsc}}{\sqrt{t}}+H^2\epsilon},
\]
for some large enough constant $C<\infty$ and $\epsilon\le 1$, then we have the following facts about $\calToo$ holds true: 
\begin{align*}
|\calToo| \le O\left(\frac{T}{\log\frac{\calN\calN_bTH}{\alpha\delta}\cdot H^3}\right)~\text{when}~\calE_{\le {T}}\cap \calE_{\xid}~\text{happens}.
\end{align*}
\end{lemma}

\begin{proof}
Similar to the proof of~\Cref{lem:size-of-Too}, when the events $\calE_{\le T}$ and also $\calE_{\xid}$ by~\Cref{lem:concentration-indicator} happen, we have 
\begin{align*}
& \sum_{t\in\calToo}\Par{f_{t,2}^{h_t}(x_t^{h_t})-f_{t,1}^{h_t}(x_t^{h_t})}
\\
& \hspace{3em} \ge \frac{C}{4}\Par{\sqrt{\log\frac{\calN TH}{\alpha\delta}+\frac{T}{\alpha^2}\epsilon}\cdot\Par{\log\frac{\calN\calN_bTH}{\alpha\delta}\cdot H^{5/2} \sqrt{d_\alpha}\frac{|\calToo|}{\sqrt{T}}+|\calToo|H\epsc}+|\calToo|H^2\eps+|\calToo|H\delta}
\end{align*}
and also 
\begin{align*}
    & \sum_{t\in\calToo}\left(f_{t,2}^{h_t}(x_t^{h_t})-f_{t,1}^{h_t}(x_t^{h_t})\right) \le \sum_{t\in\calToo}\left(f_{t,2}^{h_t}(x_t^{h_t})-V_t^{h_t}\right)\\
    & \hspace{3em}\le O\left(\sqrt{\log\frac{\calN TH}{\alpha\delta}+\frac{T}{\alpha^2}\epsilon}\cdot\Par{\sqrt{\log\frac{\calN\calN_bTH}{\alpha\delta}}H\sqrt{|\calToo|\cdot d_\alpha}+\log\frac{\calN\calN_bTH}{\alpha\delta}H\cdot d_\alpha+|\calToo|H\epsc}\right)\\
    & \hspace{5em}~+\sum_{t\in\calToo}\sum_{h_t\le h\le H}\left(\xi_t^{h}-\xi_{t,2}^h\right)+O(|\calToo|H^2\epsilon)\\
    & \hspace{3em}\le O\left(\sqrt{\log\frac{\calN TH}{\alpha\delta}+\frac{T}{\alpha^2}\epsilon}\cdot\Par{\sqrt{\log\frac{\calN\calN_bTH}{\alpha\delta}}H\sqrt{|\calToo|\cdot d_\alpha}+\log\frac{\calN\calN_bTH}{\alpha\delta}H\cdot d_\alpha+|\calToo|H\epsc}\right)\\
    & \hspace{5em}~+O\left(\sqrt{|\calToo| H\log\frac{TH}{\delta}}+\log\frac{TH}{\delta}+|\calToo|H^2\epsilon\right).
\end{align*}
Here for the last inequality we use the bound in~\Cref{lem:concentration-indicator} conditioning on $\calE_{\xid}$. Thus for sufficiently large constant $C<\infty$, the above two inequalities hold true only when $|\calToo|\le O\left(T/(H^3\cdot\log\frac{\calN\calN_bTH}{\alpha\delta})\right)$.
\end{proof}

Building on this bound of $|\calToo|$, we show the next corollary on a tighter bound for the summation terms in $III$.
\begin{corollary}[Fine-grained bound on $III$]\label{lem:bound-fine-grained-b2-hp}
Given $b_{t,2}(\cdot)\le C\cdot \biggl(D_{\calF^h}(\cdot, z_{[t-1]}^h, \1_{[t-1]}^h)\sqrt{\left(\beta_{t,2}^h\right)^2+\lambda}+\epsc \cdot \beta_{t,2}^h\biggr)$ and using the particular choice of $u_t$ as in~\Cref{lem:size-of-Too-hp}, when $\lambda = \Theta(1)$, $\alpha\le 1$, we have when $\calE_{\le T}$ and $\calE_{\xid}$ happen, 
\begin{align*}
    III & \defeq \sum_{t\in\calToo}\sum_{h\in[H]}\min\left(1+L, b_{t,2}^h(z_t^h)\right)\\
    & \hspace{3em} = O\left(\sqrt{\log\frac{\calN TH}{\delta}+T\epsilon}\cdot \sqrt{T\cdot d_\alpha}+\sqrt{\log\frac{\calN \calN_bTH}{\delta}+T\epsilon}\cdot\left(H\cdot d_\alpha+T\epsc\right)\right).
\end{align*}
\end{corollary}
\begin{proof}
This is an immediate corollary by combining~\Cref{lem:bound-E-III} and~\Cref{lem:size-of-Too-hp}.
\end{proof}

\subsection{Sum of Variances}\label{app:variance-hp}

Further, we will provide a lemma showing that the concentration of the summation variance terms in the exploration happens with high probability, which we denote as event $\calE_\rV$. The result builds on law of total variance due to~\Cref{prop:LTV-2} and the in-expectation bounds due to~\Cref{coro:sum-variance}. Recall the definition of $\calH_{t-1}^H = \sigma(x_1^{1},r_1^1, x_1^2,\cdots, r^{H}_1,x^{H+1}_1; x_2^{1},r_2^1, x_2^2,\cdots, r^{H}_2,x^{H+1}_2;\cdots,x_{t-1}^1,r_{t-1}^1,\cdots, r_{t-1}^{H},x_{t-1}^{H+1})$ and exploration rule~\eqref{eq:greedy-policy}. 

\begin{corollary}[Corollary from adapted version using LTV, high probability]\label{coro:sum-variance-hp}
Recall the simplified expression of $\rV[\cdot|z_t^h] = \rV_{r^h, x^{h+1}}[\cdot|z_t^h]$. When $L=O(1)$ we have with probability at least $1-\delta$, 
\begin{align*}
& \sum_{t\in[T]}\sum_{h\in[H]} \rV \left[r^h+f_{t,1}^{h+1}(x^{h+1})| z_t^h\right]\\
& \hspace{2em} \le O\left(H^4\log^2\frac{TH}{\delta}+T+TH^2\delta+H^2|\calToo|+H\cdot \sum_{t\in[T]}\sum_{h\in[H]}\Par{f_{t,2}^h(z_t^h)-f_{t,-2}^h(z_t^h)}\right).
\end{align*}
We denote such event as $\calE_{\rV}$ hereinafter.
\end{corollary}

\begin{proof}
For the high probability bounds, we consider applying Freedman's inequality in~\Cref{coro:Freedman-variant} to \begin{align*}
    \D_t & = \left(\sum_{h\in[H]}\rV \left[r^h+f_{t,1}^{h+1}(x^{h+1})| x_t^h,a_t^h\right]- \rE\left[\sum_{h\in[H]}\rV\left[r^h+f_{t,1}^{h+1}(x^{h+1})| z_t^h\right]~|~\calH_{t-1}^H\right]\right),\\
    \text{with}~~M & = O(H),~~V^2 =  O(TH^2),\\
    \text{and}~\sum_{t\in[T]}\E\left[\D_t^2~|~\calH_{t-1}^H\right] & = ~H\cdot\sum_{t\in[T]}\rE\left[\sum_{h\in[H]}\rV \left[r^h+f_{t,1}^{h+1}(x^{h+1})| z_t^h\right]~|~\calH_{t-1}^H\right].
\end{align*}

Thus~\Cref{coro:Freedman-variant} gives that with probability at least $1-\delta$, it holds that 
\begin{align*}
& \sum_{t\in[T],h\in[H]}\rV\left[r^h+f_{t,1}^{h+1}(x^{h+1})|z_t^h\right] \stackrel{(i)}{\le} O\left(H\cdot \log(TH/\delta)+\sum_{t\in[T]}\rE\left[\sum_{h\in[H]}\rV \left[r^h+f_{t,1}^{h+1}(x^{h+1})| z_t^h\right]~|~\calH^H_{t-1}\right]\right)\\
& \hspace{1em} \stackrel{(ii)}{\le} O\left(H\log\frac{TH}{\delta}+T+TH^2\delta+H^2\cdot \sum_{t\in[T]}\E[\1_{\{t\in\calToo\}}~|~\calH_{t-1}^H]+H\cdot \sum_{t\in[T]}\E\biggl[\sum_{h\in[H]}\bigl(f_{t,2}^h(z_t^h)-f_{t,-2}^h(z_t^h)\bigr)~|~\calH_{t-1}^H\biggr]\right)\\
& \hspace{1em}\stackrel{(iii)}{\le} O\left((H+H^2\sqrt{T})\log\frac{TH}{\delta}+T+TH^2\delta+H^2\cdot|\calToo|+H\sum_{t\in[T]}\sum_{h\in[H]}\Par{f_{t,2}^h(z_t^h)-f_{t,-2}^h(z_t^h)}\right),
\end{align*}
where for $(i)$ we also use the AM-GM inequality, for $(ii)$ we use~\Cref{coro:sum-variance}, and for $(iii)$ we use the Azuma-Hoeffding concentration of martingale.  The claim follows by again applying the AM-GM inequality. 
\end{proof}

\subsection{Difference of Overly Optimistic Sequence}\label{app:diff-of-f}

\begin{lemma}\label{lem:difference-f-bar-hp}
When good events $\calE_{\le T}$, $\calE_{\xi_1}$, $\calE_{\xi_{-2}}$ and $\calE_{\xi_2}$ happen, and when $\lambda = \Theta(1)$, $\alpha\le 1$, $\epsilon\le 1$  we have 
\begin{align*}
& \sum_{t\in[T]}\sum_{h\in[H]}\left[f_{t,2}^h(z_t^h)-f_{t,-2}^h(z_t^h)\right]\\
& \hspace{3em} \le O\left(H\cdot [I]+H\cdot [II]+H\sqrt{HT\log(H/\delta)}+H\log(H/\delta)+|\calToo|H^2+TH^2\epsilon\right)+H\cdot\sum_{t\in\calTo} u_t.
\end{align*}
\end{lemma}

\begin{proof}
Similar to~the proof of~\Cref{lem:difference-f-bar}, for $t\in\calToo$ it holds that 
\[
\sum_{t\in\calToo}\sum_{h\in[H]}\left[f_{t,2}^h(z_t^h)-f_{t,-2}^h(z_t^h)\right]= O(|\calToo|H).
\]
Otherwise, for iterations $t\in\calTo$, we also have 
\begin{align*}
& \sum_{t\in\calTo}\sum_{h\in[H]}\left[f_{t,2}^h(z_t^h)-f_{t,-2}^h(z_t^h)\right] \le O\left(\sum_{t\in\calTo}\sum_{ h\in[H]}\min\left(1,\sum_{h\le h'\le H}b_{t,1}^{h'}(z_t^{h'})\right)+\sum_{t\in\calTo}\sum_{h\in[H]}\min\left(1,\sum_{h\le h'\le H}b_{t,2}^{h'}(z_t^{h'})\right)\right)\\
 & \hspace{15em} +\sum_{t\in\calTo}H\cdot u_t+O\left(\sum_{t\in\calTo}\sum_{h\in[H]}\sum_{h\le h'\le H}\xi^{h'}_{t,-2}-\sum_{t\in\calTo}\sum_{h\in[H]}\sum_{h\le h'\le H}\xi^{h'}_{t,1}+TH^2\epsilon\right),
\end{align*}
given event $\calE_{\le T}$ happens.

Summing over all iterations $t\in\calTo$, this implies 
\begin{align*}
& \sum_{t\in\calTo}\sum_{h\in[H]}\left[f_{t,2}^h(z_t^h)-f_{t,-2}^h(z_t^h)\right] \\
& \hspace{3em} \stackrel{(i)}{\le} O\left(H\cdot I+H\cdot II+TH^2\epsilon\right)+\sum_{t\in\calTo}H\cdot u_t\\
 & \hspace{6em} +O\left(\left|\sum_{t\in[T]}\sum_{h\in[H]}\sum_{h\le h'\le H}\xi^{h'}_{t,-2}\right|+\left|\sum_{t\in[T]}\sum_{h\in[H]}\sum_{h\le h'\le H}\xi^{h'}_{t,1}\right|+\left|\sum_{t\in[T]}\sum_{h\in[H]}\xi^{h}_{t,2}\right|\right)\\
 &\hspace{6em} +O\left(\left|\sum_{t\in\calToo}\sum_{h\in[H]}\sum_{h\le h'\le H}\xi^{h'}_{t,-2}\right|+\left|\sum_{t\in\calToo}\sum_{h\in[H]}\sum_{h\le h'\le H}\xi^{h'}_{t,1}\right|+\left|\sum_{t\in\calToo}\sum_{h\in[H]}\xi^{h}_{t,2}\right|\right)\\
 & \hspace{3em} \le O\left(H\cdot I+H\cdot II+TH^2\epsilon+|\calToo|H^2\right)+\sum_{t\in\calTo}H\cdot u_t\\
 & \hspace{6em} +O\left(\left|\sum_{t\in[T]}\sum_{h\in[H]}\sum_{h\le h'\le H}\xi^{h'}_{t,-2}\right|+\left|\sum_{t\in[T]}\sum_{h\in[H]}\sum_{h\le h'\le H}\xi^{h'}_{t,1}\right|+\left|\sum_{t\in[T]}\sum_{h\in[H]}\xi^{h}_{t,2}\right|\right).
\end{align*}

Now since also events $\calE_{\xi_1}$, $\calE_{\xi_2}$ and $\calE_{\xi_{-2}}$ happen, plugging in bounds in~\eqref{eq:freedman-xi-1},~\eqref{eq:freedman-xi-2-pos} and~\eqref{eq:freedman-xi-2} we have 
\begin{align*}
& \sum_{t\in\calTo}\sum_{h\in[H]}\left[f_{t,2}^h(z_t^h)-f_{t,-2}^h(z_t^h)\right] \\
 & \hspace{3em} \le O\left(H\cdot [I]+H\cdot [II]+H\sqrt{HT\log(H/\delta)}+H\log(H/\delta)+|\calToo|H^2+TH^2\epsilon\right)+H\cdot\sum_{t\in\calTo} u_t.
\end{align*}

Summing the two cases gives the claimed bound.
\end{proof}

\subsection{Bounding the Regret in High Probability}\label{app:regret-hp-subsec}

Now we will continue working on bounding $I = \sum_{t\in[T]}\sum_{h\in[H]}\min\left(1+L, b_{t,1}^h(z_t^h)\right)$.

\begin{lemma}[Fine-grained bound on $I$]\label{lem:regret-bound-I-hp}
Recall the definition of $b_{t,1}$ and $b_{t,2}$ as in~\Cref{lem:bound-E-III-o} and~\Cref{lem:bound-fine-grained-b2-hp}. When $\lambda = 1$, $\alpha= 1/\sqrt{TH}$, $\epsilon\le1$, when event $\calE = \calE_{\le T}\cap \calE_{\xi_1}\cap\calE_{\xi_{-2}}\cap\calE_{\xi_2}\cap\calE_{\xid}\cap\calE_{\rV}$ happens, we have the following inequality holds true:
\begin{align*}
    & I \defeq \sum_{t\in[T]}\sum_{h\in[H]}\min\left(1+L, b_{t,1}^h(z_t^h)\right)\\
    & \hspace{1em} =  O\left(\sqrt{\log\frac{\calN TH}{\delta}+T^2H\epsilon}\sqrt{T+TH^2\delta}\cdot\sqrt{H\cdot d_\alpha}\right)\\
    & \hspace{3em} +O\left(\sqrt{\log\frac{\calN TH}{\delta}+T^2H\epsilon}\sqrt{\log\frac{\calN\calN_b TH}{\delta}} \sqrt{H^3|\calToo|+TH^3\epsilon+H^2\sum_{t\in[T]}u_t}\cdot\sqrt{H\cdot d_\alpha}\right)\\
        & \hspace{3em}+O\left(\Par{\log\frac{\calN TH}{\delta}+T^2H\epsilon}\log^{1.5} \frac{\calN\calN_bTH}{\delta}\cdot H^{7/2}d_\alpha+\sqrt{\log\frac{\calN TH}{\delta}+T^2H\epsilon}\cdot TH\epsc\right).
\end{align*}
\end{lemma}
\begin{proof}
Again we note that by assumption and definition, 
\begin{align*}
   & \sum_{t\in[T]}\sum_{h\in[H]} \min\left(1+L,b_{t,1}^h(z_t^h) \right)
    = O\left(\sum_{t\in[T]}\sum_{h\in[H]}\min\left(1,D_{\calF^h}(z_t^h; z_{[t-1]}^h, \bsigma^h_{[t-1]})\cdot\sqrt{\left(\beta_{t,1}^h\right)^2+\lambda} \right)+TH\epsc\cdot\max_{t,h}\beta_{t,1}^h\right) \\
    & \hspace{10em} = O\left(\sqrt{\log\frac{\calN TH}{\delta}+T^2H\epsilon}\cdot \Par{\sum_{t\in[T]}\sum_{h\in[H]}\min\left(1,D_{\calF^h}(z_t^h; z_{[t-1]}^h, \bsigma_{[t-1]}^h)\right)+TH\epsc}\right).
\end{align*}

Treating $L=O(1)$ as defined (see~\Cref{ass:eps-realizability-RL}), we now bound the summation terms \[\sum_{t\in[T]}\sum_{h\in[H]}\min\left(1,D_{\calF^h}(z_t^h; z_{[t-1]}^h, \bsigma^h_{[t-1]})\right) = \sum_{t\in[T]}\sum_{h\in[H]}\min\left(1,\bsigma_t^h \cdot \Par{\bsigma_t^h}^{-1} D_{\calF^h}(z_t^h; z_{[t-1]}^h, \bsigma^h_{[t-1]})\right)\] by dividing into cases, same as we do in~\Cref{lem:bound-E-III-o} and~\Cref{lem:regret-bound-I}.
\begin{align*}
\calI_1 &= \{(t,h)\in\calT\times[H]~|~\Par{\bsigma_t^h}^{-1}D_{\calF^h}(z_t^h; z_{[t-1]}^h, \bsigma_{[t-1]}^h)\ge 1\},\\
\calI_2 &= \{(t,h)\in\calT\times[H]~|~(t,h)\notin\calI_1,~\bsigma_t^h = \alpha\},\\
\calI_3 &= \biggl\{(t,h)\in\calT\times[H]~|~(t,h)\notin\calI_1,~\bsigma_t^h = 2\bigl(\sqrt{\upsilon(\delta_{t,h})}+\iota(\delta_{t,h})\bigr)\cdot\sqrt{ D_{\calF^h}(z^{h}_t; z^{h}_{[t-1]},\bsigma^{h}_{[t-1]})}\biggr\},\\
\calI_4 &= \left\{(t,h)\in\calT\times[H]~|~(t,h)\notin\calI_1,~\bsigma_t^h = \sigma_t^h\right\},\\
\calI_5 &= \left\{(t,h)\in\calT\times[H]~|~(t,h)\notin\calI_1,~\bsigma_t^h = \sqrt{2}\iota(\delta_{t,h})\sqrt{f^h_{t,2}(z_t^h)-f^h_{t,-2}(z_t^h)}\right\}.
\end{align*}
    
Now same as in~\Cref{lem:bound-E-III-o} and~\Cref{lem:regret-bound-I} we could bound the first three terms as
    \begin{align*}
    \sum_{(t,h)\in\calI_1\cup\calI_2\cup\calI_3}\min\left(1,\bsigma_{t}^h\cdot \Par{\bsigma_{t}^h}^{-1}D_{\calF^h}(z_{t}^h;z_{[t-1]}^h,\bsigma_{[t-1]}^h)\right) & \le O\Par{\log\frac{\calN\calN_bTH}{\delta}\cdot Hd_\alpha}.
    \end{align*}
    
For terms restricting on $\calI_4$ and $\calI_5$ we do a similar tighter analysis in correspondence to~\Cref{lem:regret-bound-I}. For summation terms in $\calI_5$, recall~\Cref{eq:regret-I-calI-5} already shows that
    \begin{align*}
        & \sum_{(t,h)\in\calI_5}\min\left(1,\bsigma_{t}^h\cdot\Par{\bsigma_{t}^h}^{-1}D_{\calF^h}(z_{t}^h;z_{[t-1]}^h,\bsigma_{[t-1]}^h)\right)  \\
        & \hspace{5em} \le O\left(\sqrt{\log\frac{\calN\calN_b TH}{\delta}}\sqrt{\sum_{t,h}\Par{f_{t,2}^h(z_t^h)-f_{t,-2}^h(z_t^h)}}\cdot\sqrt{H\cdot d_\alpha}\right).
    \end{align*}
    
Restricting on $\calI_4$, when  event $\calE = \calE_{\le T}\cap \calE_{\xi_1}\cap\calE_{\xi_{-2}}\cap\calE_{\xi_2}\cap\calE_{\xid}\cap\calE_{\rV}$  happens, by using Cauchy-Schwarz inequality,~\Cref{lem:RL-variance-bounds-upper},~\Cref{coro:sum-variance-hp}, and AM-GM inequality, similar to the in-expectation proof we have 
    \begin{align*}
          &\sum_{(t,h)\in\calI_4}\min\left(1,\bsigma_{t}^h\cdot \Par{\bsigma_t^h}^{-1}D_{\calF^h}(z_{t}^h;z_{[t-1]}^h,\bsigma_{[t-1]}^h)\right) \\
        & \hspace{1em}  \le O\left(\sqrt{T+H^2|\calToo|+TH^2(\eps+\delta)}\cdot\sqrt{H\cdot d_\alpha}\right)\\
        & \hspace{3em} +O\left( \sqrt{H\cdot \sum_{t,h}\left(f_{t,2}^h(z_{t}^h)-f_{t,-2}^h(z_{t}^h)\right)+H^2d_\alpha\cdot\Par{\log^2\frac{\calN\calN_b TH}{\delta}+T\epsilon}}\cdot\sqrt{H\cdot d_\alpha}\right).
    \end{align*}

    Summing all cases together we have
    \begin{align*}
    & \sum_{t\in[T]}\sum_{h\in[H]}\min\left(1,\bsigma_{t}^h\cdot \Par{\bsigma_t^h}^{-1}D_{\calF^h}(z_{t}^h;z_{[t-1]}^h,\bsigma_{[t-1]}^h)\right)\\
   & \hspace{1em}  \le O\left(\sqrt{T+H^2|\calToo|+TH^2(\eps+\delta)}\cdot\sqrt{H\cdot d_\alpha}\right)\\
   & \hspace{3em} + O\left(\sqrt{\log\frac{\calN\calN_b TH}{\delta}} \sqrt{H\sum_{t,h}\left(f_{t,2}^h(z_{t}^h)-f_{t,-2}^h(z_{t}^h)\right)}\cdot\sqrt{H\cdot d_\alpha}\right)\\
        & \hspace{3em} +O\left(\Par{\log\frac{\calN\calN_b TH}{\delta}+T\epsilon}\cdot H^{1.5}\cdot d_\alpha\right).
    \end{align*}
    
    Now plugging in the bound on $\sum_{t,h}\left(f_{t,2}^h(z_{t}^h)-f_{t,-2}^h(z_{t}^h)\right)$ in~\Cref{lem:difference-f-bar-hp}, we have
    \begin{align*}
    & \sum_{t\in[T]}\sum_{h\in[H]}\min\left(1,\bsigma_{t}^h\cdot \Par{\bsigma_t^h}^{-1}D_{\calF^h}(z_{t}^h;z_{[t-1]}^h,\bsigma_{[t-1]}^h)\right)\\
   & \hspace{1em}  \le O\left(\sqrt{T+TH^2\delta}\cdot\sqrt{H\cdot d_\alpha}\right)+O\left(\sqrt{\log\frac{\calN\calN_b TH}{\delta}} \sqrt{H^3|\calToo|+TH^3\epsilon+H^2\sum_{t\in[T]}u_t}\cdot\sqrt{H\cdot d_\alpha}\right)\\
  & \hspace{2em} +O\left(\Par{\log\frac{\calN\calN_b TH}{\delta}+T\epsilon}\cdot H^{1.5}\cdot d_\alpha\right)\\
  & \hspace{2em} + O\left(\sqrt{\log\frac{\calN\calN_b TH}{\delta}} \sqrt{H^2\cdot [I]+H^2\cdot [II]+H^2\sqrt{HT\log\frac{H}{\delta}}}\cdot\sqrt{H\cdot d_\alpha}\right).
    \end{align*}
    
    To bound the last term in the RHS of inequality above, we plug in crude bounds in~\Cref{lem:bound-E-III} and~\Cref{lem:bound-E-III-o}, note the crude bound of $I$ dominates that of $II$ and absorbing the low-order terms we have
     \begin{align*}
   &H^2\cdot [I]+H^2\cdot [II]+H^2\sqrt{HT\log\frac{H}{\delta}}\\
   & \hspace{1em} \le  O\Par{\sqrt{\log\frac{\calN TH}{\delta}+T^2H\epsilon}\cdot\sqrt{\log\frac{\calN\calN_bTH}{\delta}}\cdot H^3\sqrt{Td_\alpha}+\sqrt{\log\frac{\calN TH}{\delta}+T^2H\epsilon}\cdot \Par{\log\frac{\calN\calN_b TH}{\alpha\delta}\cdot H^2d_\alpha+TH^2\epsc}}\\
   &\hspace{1em} \le O\Par{\frac{T}{\log\frac{\calN\calN_bTH}{\delta}}+\Par{\log\frac{\calN TH}{\delta}+T^2H\epsilon}\Par{\log\frac{\calN\calN_bTH}{\delta}}^2H^6d_\alpha +\sqrt{\log\frac{\calN TH}{\delta}+T^2H\epsilon}\cdot TH^2\epsc}.
    \end{align*}

Plugging this back, rearranging terms, absorbing lower terms and again use AM-GM inequality we have
    \begin{align*}
        & \sum_{t\in[T]}\sum_{h\in[H]}\min\left(1,D_{\calF^h}(z_t^h; z_{[t-1]}^h, \bsigma_{[t-1]}^h)\right)\\
        & \hspace{1em}\le  O\left(\sqrt{T+TH^2\delta}\cdot\sqrt{H\cdot d_\alpha}\right)+O\left(\sqrt{\log\frac{\calN\calN_b TH}{\delta}} \sqrt{H^3|\calToo|+TH^3\epsilon+H^2\sum_{t\in[T]}u_t}\cdot\sqrt{H\cdot d_\alpha}\right)\\
        & \hspace{3em}+O\left(\sqrt{ \Par{\log\frac{\calN TH}{\delta}+T^2H\epsilon}}\log^{1.5} \frac{\calN\calN_bTH}{\delta}\cdot H^{7/2}d_\alpha+TH\epsc\right).
    \end{align*}
\end{proof}

With these bounds we are ready to prove formally the regret bounds for~\Cref{alg:fitted-Q-simpler}.

\thmhp*

\begin{proof}
When event $\calE = \calE_{\le T}\cap\calE_{\xi}\cap\calE_{\xi_1}\cap\calE_{\xi_{-2}}\cap\calE_{\xi_2}\cap\calE_{\xid}\cap\calE_{\rV}$  happen (with probability $1-11\delta$), we recall the upper bound on regret as:
\begin{align*}
R_T 
    & \le O(1+HT\epsilon)+ 2\underbrace{\sum_{ t\in [T]}\sum_{h\in[H]} \min\left(1+L,b_{t,1}^h(z_{t}^h)\right)}_{I}+2\sum_{ t\in \calToo} \min\left(1+L,\sum_{h_t\le h\le H}b_{t,2}^h(z_{t}^h)\right)\\
    & \hspace{3em} +\left[\sum_{t\in[T],h\in[H]}\xi_t^h-\sum_{ t\in[T]}\sum_{h\in[H]}\xi_{t,1}+\sum_{ t\in\calToo}\left(\sum_{h_t\le h\le H}\xi_{t,1}^h-\sum_{h_t\le h\le H}\xi_{t,2}^h\right)\right]\\
    & \le O\left(1+HT\epsilon+I+III\right)+\left[\sum_{t\in[T],h\in[H]}\xi_t^h-\sum_{ t\in[T]}\sum_{h\in[H]}\xi_{t,1}\right]+\sum_{ t\in\calToo}\sum_{h_t\le h\le H}\left(\xi_{t,1}^h-\xi_{t,2}^h\right).
\end{align*}

Now plugging in guarantees of~\Cref{lem:regret-bound-I-hp} for bounding $I$,~\Cref{lem:bound-fine-grained-b2-hp} for bounding III,~\Cref{eq:freedman-xi} for bounding  $\sum_{t,h}\xi_t^h$,~\Cref{eq:freedman-xi-1} for bounding~$\sum_{t,h}\xi_{t,1}^h$, ~\Cref{lem:concentration-indicator} for bounding $\sum_{ t\in\calToo}\sum_{h_t\le h\le H}\left(\xi_{t,1}^h-\xi_{t,2}^h\right)$, we have 
\begin{align*}
    R_T & = O\left(TH\epsilon+\sqrt{\log\frac{\calN TH}{\delta}+T^2H\epsilon}\cdot\Par{\sqrt{T+TH^2\delta}\cdot\sqrt{Hd_\alpha}}\right.\\
     & \hspace{3em} +\sqrt{\log\frac{\calN TH}{\delta}+T^2H\epsilon}\sqrt{\log\frac{\calN\calN_b TH}{\delta}} \sqrt{H^3|\calToo|+TH^3\epsilon+H^2\sum_{t\in[T]}u_t}\cdot\sqrt{H\cdot d_\alpha}\\
    & \hspace{3em}\left.+ \Par{\log\frac{\calN TH}{\delta}+T^2H\epsilon}\cdot \log^{1.5} \frac{\calN\calN_bTH}{\delta}\cdot H^{7/2}\cdot d_\alpha+\sqrt{\log\frac{\calN\calN_bTH}{\delta}+T^2H\epsilon}\cdot TH\epsc\right).
\end{align*}

Now further plugging in the assumption of $\delta$, choice of $u_t$ and the bound of $\calToo$ under such choice as in~\Cref{lem:size-of-Too-hp}, we have
\begin{align*}
    R_T & =   O\Par{1+\sqrt{\log\frac{\calN TH}{\delta}+T^2H\epsilon}\cdot \sqrt{T}\cdot\sqrt{H d_\alpha}}\\
    & \hspace{3em} +O\left(\Par{\log\frac{\calN TH}{\delta}+T^2H\epsilon}\cdot\Par{\log^2\frac{\calN\calN_bTH}{\delta}H^{5}d_\alpha+T^2 \epsc^2}\right).
\end{align*}
Adjusting the constant $\delta\leftarrow \tfrac{\delta}{11}$ and absorbing the low-order terms of $\mathrm{poly}(\epsilon)$ conclude the final bound for regret.
\end{proof}

\end{document}